\pdfoutput=1
\documentclass[prodmode,acmcsur]{acmsmall} 

\usepackage{QSR-survey}
\newcommand{\calc}[1]{\mbox{\textsf{#1}}}
\newcommand{\system}[1]{\textsl{#1}}

\usepackage[ruled,vlined,linesnumbered]{algorithm2e}
\renewcommand{\algorithmcfname}{ALGORITHM}
\SetAlFnt{\small}
\SetAlCapFnt{\small}
\SetAlCapNameFnt{\small}
\SetAlCapHSkip{0pt}
\IncMargin{-\parindent}

\begin{document}

\listoftodos

\markboth{Frank Dylla et al.}{A Survey of Qualitative Spatial and Temporal Calculi}

\title{A Survey of Qualitative Spatial and Temporal Calculi --- Algebraic and Computational Properties}
\author{FRANK DYLLA
\affil{University of Bremen}
JAE HEE LEE
\affil{University of Bremen}
TILL MOSSAKOWSKI
\affil{University of Magdeburg}
THOMAS SCHNEIDER
\affil{University of Bremen}
ANDRÉ VAN DELDEN
\affil{University of Bremen}
JASPER VAN DE VEN
\affil{University of Bremen}
DIEDRICH WOLTER
\affil{University of Bamberg}}

\begin{abstract}
  Qualitative Spatial and Temporal Reasoning (QSTR) is concerned with symbolic knowledge representation,
  typically over infinite domains.
  The motivations for employing QSTR techniques range from exploiting computational properties that allow efficient
  reasoning to capture human cognitive concepts in a computational framework.
  The notion of a qualitative calculus is one of the most prominent QSTR formalisms.
  This article presents the first overview of all qualitative calculi developed to date and their computational properties,
  together with generalized definitions of the fundamental concepts and methods, which now encompass all existing calculi.
  Moreover, we provide a classification of calculi according to their algebraic properties. 
\end{abstract}




\acmformat{Frank Dylla, Jae Hee Lee, Till Mossakowski, Thomas Schneider, André van Delden, Jasper van de Ven and Diedrich Wolter. 2015. 
A Survey of Qualitative Spatial and Temporal Calculi --- Algebraic and Computational Properties
}


\begin{bottomstuff}
This work has been supported by the DFG-funded SFB/TR 8 “Spatial Cognition”, projects R3-[QShape] and R4-[LogoSpace]. 

Author names in alphabetic order.

Author's addresses: 
Frank Dylla, Thomas Schneider, André van Delden, {and} Jasper van de Ven, Faculty of Computer Science, University of Bremen, Germany;
Jae Hee Lee, Centre for Quantum Computation \& Intelligent Systems, University of Technology Sydney, Australia;
Till Mossakowski, Faculty of Computer Science, Otto-von-Guericke-University of Magdeburg, Germany;
Diedrich Wolter, Faculty of Information Systems and Applied Computer Sciences, University of Bamberg, Germany.
\end{bottomstuff}

\maketitle

\section{Introduction}
Knowledge about our world is densely interwoven with spatial and temporal facts.
Nearly every knowledge-based system  comprises means for representation of, and 
possibly reasoning about, spatial or temporal knowledge.
Among the different options available to a system designer, ranging from 
domain-level data structures to highly abstract logics,
qualitative approaches stand out for their ability to mediate between the domain level and the conceptual level.
Qualitative representations explicate relational knowledge between (spatial or temporal) domain entities, allowing individual statements to be evaluated by truth values.
The aim of qualitative representations is to focus on the aspects that are essential for a task at hand by abstracting away from other, unimportant aspects.
As a result, a wide range of representations has been applied, using various kinds of knowledge representation languages.
The most fundamental principles for representing knowledge qualitatively that are at the heart of virtually every representation language are captured by a construct called \emph{qualitative spatial (or temporal) calculus}.
In the past decades, a great variety of qualitative calculi have 
been developed, each tailored to 
specific aspects of spatial or temporal knowledge. 
They share common principles but differ in formal and computational properties.

This article presents an up-to-date comprehensive overview of \emph{qualitative spatial and temporal reasoning (QSTR)}.
We provide a general definition of QSTR (Section \ref{sec:what_is}),
give a uniform account of a calculus that is more integrative than existing ones (Section \ref{sec:qualitative_representations}),
identify and differentiate algebraic properties of calculi (Section \ref{sec:relation_algebras}),
and discuss 
their role within 
other knowledge representation paradigms (Section \ref{sec:integration+combination}) as well as alternative approaches (Section \ref{sec:alternatives}). 
Besides the survey character, the article provides a taxonomy of the most prominent reasoning problems,
a survey of {\em all} existing calculi proposed so far (to the best of our knowledge),
and the first comprehensive overview of their computational properties.

This article is accompanied by an electronic appendix that contains minor technical details such as mathematical proofs of some claims
and detailed experimental results.
%

\subsection*{Demarcation of Scope and Contribution}

This article addresses researchers and engineers working with
knowledge about space or time and wishing to employ reasoning on a symbolic level.
We supply a thorough overview of the wealth of qualitative spatial and temporal 
calculi available,
\begin{New}
  many of which have emerged from concrete application scenarios,
  for example, geographical information systems (GIS)
  \cite{DBLP:conf/ssd/Egenhofer91,Frank91};
  see also the overview given in \cite{Ligozat11}.
  Our survey focuses on the calculi themselves (Tables~\ref{tab:calculi_new_1}--\ref{tab:calculi_new_3})
  and their computational and algebraic properties, i.e.,
\end{New}
characteristics relevant for reasoning and symbolic manipulation
(Table~\ref{tab:calculi_reasoning}, Figure~\ref{fig:algebra_notions}).
To this end, we also categorize reasoning tasks involving qualitative representations
(Figure~\ref{fig:tax2}).

We exclusively consider qualitative formalisms for reasoning on the basis of finite sets of relations over an infinite spatial or temporal domain. 
As such, the mere use of symbolic labels is not surveyed.
We also disregard approaches augmenting qualitative formalisms with an additional interpretation such as fuzzy sets or probability theory.

This article significantly advances from previous publications with a survey character in several regards.
\citeN{Ligozat11} describes in the course of the book ``the main'' qualitative calculi%
, describes their relations, complexity issues and selected techniques. 
\begin{New}
Although an algebraic perspective is taken as well, we integrate this in a more general context.
Additionally to mentioning general axioms in context of relation algebras we present a thorough investigation of calculi regarding these axioms.
\end{New}
He also gives references to applications that employ QSTR techniques in a broad sense.
Our survey supplements precise definitions of the underlying formal aspects,
which will then be general enough to encompass \emph{all} existing calculi that we are aware of.
%
\citeN{chen_survey_2013} summarize the progress in QSTR by presenting selected key calculi for important spatial aspects.
They give a brief introduction to basic properties of calculi, but neither detail formal properties nor picture the entire variety of formalisms achieved so far as provided by this article.
%
Algebra-based methods for reasoning with qualitative constraint calculi have been covered by \citeN{Renz_Nebel_2007_Qualitative}.
Their description applies to calculi that satisfy rather strong properties, which we relax. 
We present revised definitions and an algebraic closure algorithm that generalizes to all existing calculi, and, to the best of our knowledge, we give the first comprehensive overview on computational properties.
\citeN{Cohn:2008vn} present an introduction to the field which extends the earlier article of \citeN{cohn-hazarika:01} by a more detailed discussion of logic theories for mereotopology and by presenting efficient reasoning algorithms.

\section{What is Qualitative Spatial and Temporal Reasoning}
\label{sec:what_is}

We characterize QSTR by considering the reasoning problems it is concerned with.
Generally speaking, reasoning is a process to generate new knowledge from existing one.
Knowledge primarily refers to facts given explicitly, possibly implicating implicit ones.
\emph{Sound reasoning} is involved with explicating the implicit, allowing it to be processed further.
Thus sound reasoning is crucial for many applications. 
In QSTR it is a key characteristic and the applied reasoning methods are largely shaped by the specifics of qualitative knowledge about spatial and temporal domains as provided within the \emph{qualitative domain representation}.

\subsection{A General Definition of QSTR}

Qualitative domain representations employ symbols to represent semantically meaningful properties of a \emph{perceived domain}, abstracting away any details not regarded relevant to the context at hand.
The perceived domain comprises the available raw information about objects.
By \emph{qualitative abstraction}, the perceived domain is mapped to the qualitative domain representation, called domain representation from now on.
Various aims motivate research on qualitative abstractions, most importantly the desire to develop formal models of common sense relying on coarse concepts \cite{Williams:1991wz,Bredeweg:2003vc} and to capture the catalog of concepts and inference patterns in human cognition \cite{Kuipers:1978:SpatialKnowledge,KnauffEtAl:2004:PsychValidityIAI}, which in combination enables intuitive approaches to designing intelligent systems \cite{Davis:1990:RepOfCSK} or human-centered computing \cite{Frank:1992:QSR}.
Within QSTR it is required that qualitative abstraction yields a {\em finite} set of elementary concepts.
The following definition aims to encompass all contexts in which QSTR is studied in the literature.

\begin{definition}
	Qualitative spatial and temporal representation and reasoning (QSTR) is the study of techniques for representing and manipulating spatial and temporal knowledge by means of relational languages that use a finite set of symbols.
	These symbols stand for classes of semantically meaningful properties of the represented domain (positions, directions, etc.).
\end{definition}

Spatial and temporal domains are typically infinite and exhibit complex structures.
Due to their richness and diversity, QSTR is confronted with unique theoretic and computational challenges.
Consequently, there is a high variety of domain representations, each focusing on specific aspects relevant to specific tasks.
In infinite domains, concepts that are meaningful to a wide range of settings are typically relative since there are no universally `important' values.
As a consequence, QSTR is involved with relations, using a relational language to formulate domain representations.
It turns out that binary relations can capture most relevant facets of space and time -- this class also received most attention by the research community.
Expressive power is purely based on these pre-defined relations, no conjuncts or quantifiers are considered.
Thus the associated reasoning methods can be regarded as variants of constraint-based reasoning.
Additionally, constraint-based reasoning techniques can be used to empower other methods,
for example to assess the similarity of represented entities or logic inference.

Finally, to map a domain representation to the perceived domain a \emph{realization} process is applied.
This process instantiates entities in the perceived domain that are based on entities provided in the domain representation.

Figure \ref{fig:tax1} depicts the overall view on knowledge representation and aligns with the well-known view on intelligent agents considered in AI, which connects the environment to the agent and its internal representation by means of perception (which  is an abstraction process as well) and, vice versa, by actions (see, e.g., \cite[Chapter 2]{russell-norvig}). 
\begin{figure}[t]
	\centering
	\includegraphics{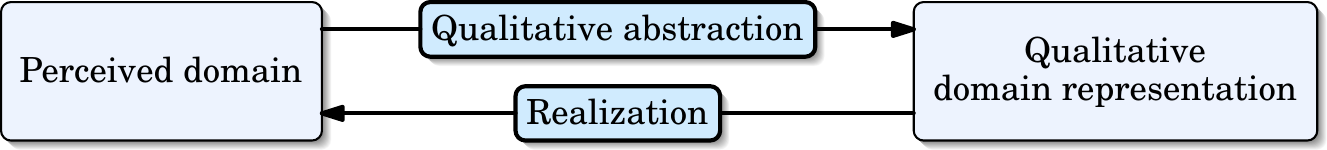}
	\caption{\label{fig:tax1} Relation between perceived domain and domain representation}
\end{figure}

\subsection{Taxonomy of Constraint-Based Reasoning Tasks}\label{ssec:what_is:taxonomy}
Figure \ref{fig:tax2} depicts an overview of constraint-based reasoning tasks in the context of QSTR.
We now briefly describe these tasks and highlight some associated literature.
The description is deliberately provided at an abstract level:
each task may come in different flavors, depending on specific (application) contexts.
Also, applicability of specific algorithms largely depends on the qualitative representation at hand.
The following taxonomy is loosely based on the overview by \citeN{Wallgruen:2013:QR:Taxonomy}.

In the following, we refer to the set of objects received from the perceived domain by applying qualitative abstraction as domain entities.
These are for example geometric entities such as points, lines, or polygons.
\begin{New}
In general domain entities can be of any type regarding spatial or temporal aspects.
\end{New}

We further use the notion of a \emph{qualitative constraint network (QCN)},
which is a special form of abstract representation.
Commonly, a QCN $Q$ is depicted as a directed labeled graph,
with nodes representing abstract domain entities, i.e., with no specific values from the domain assigned,
and edges being labeled with \emph{constraints}: symbols representing relationships  that have to hold between these entities, e.g., see Figure \ref{fig:desc_qual_full}. 
An assignment of concrete domain entities to the nodes in $Q$ is called a \emph{solution} of $Q$ if the assigned entities satisfy all constraints in $Q$.
Section \ref{sec:reasoning} contains precise definitions.

\paragraph{Constraint network generation}
This task determines relational statements that describe given domain entities regarding specific aspects, using a predetermined qualitative language fulfilling certain properties, i.e., in our case provided by a qualitative spatial calculus.
For instance, Figure \ref{fig:desc_qual_full} could be the constraint network derived from the scene shown in Figure \ref{fig:desc_quant}.
%
Techniques are given, e.g., by 
\citeN{cohn-GeoI:97}, 
\citeN{Worboys.2004},
\citeN{Forbus2004}, and \citeN{Dylla_Wallgruen_07_Qualitative}.

\begin{figure}[t]
	\centering
	\includegraphics{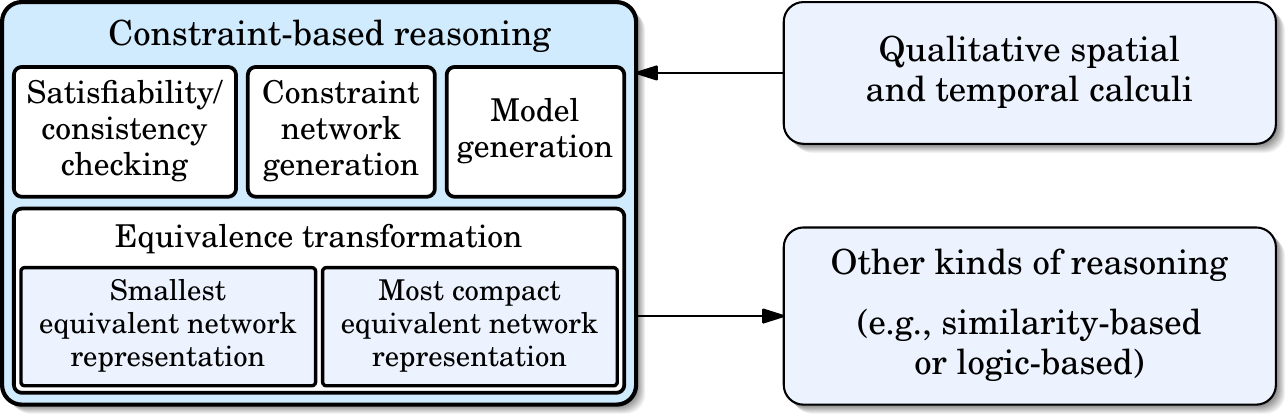}
	\caption{\label{fig:tax2}Classification of fundamental reasoning tasks and representation formalisms}
\end{figure}

\begin{figure}[b]
	\addtolength{\abovecaptionskip}{-1.5ex}
	\centering
	\begin{subfigure}[b]{.32\linewidth}
		\centering
		\includegraphics{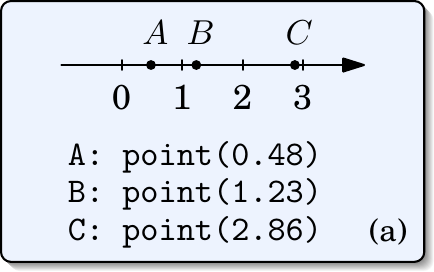}%
		\phantomsubcaption\label{fig:desc_quant}%
	\end{subfigure}%
	\hspace*{\fill}%
	\begin{subfigure}[b]{.32\linewidth}
		\centering
		\includegraphics{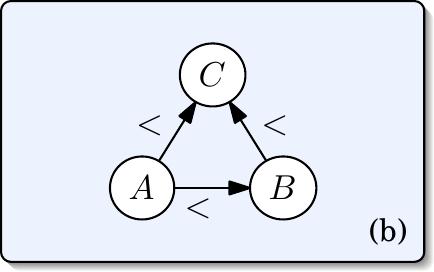}%
		\phantomsubcaption\label{fig:desc_qual_full}%
	\end{subfigure}%
	\hspace*{\fill}%
	\begin{subfigure}[b]{.32\linewidth}
		\centering
		\includegraphics{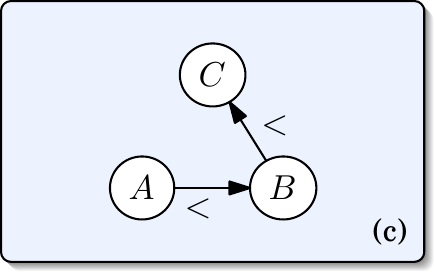}%
		\phantomsubcaption\label{fig:desc_qual_partial}%
	\end{subfigure}%
	\par\smallskip
	\caption{%
		One geometric (a) and two qualitative descriptions of a spatial scene,
		obtained via complete (b) or incomplete (c) abstraction.
		Furthermore, (b) can be obtained from (c) via constraint-based reasoning.%
	}
	\label{fig:descriptions}
\end{figure}

\paragraph{Consistency checking}
This decision problem is considered the fundamental QSTR task \cite{Renz_Nebel_2007_Qualitative}:
given an input QCN $Q$, decide whether a solution exists.
Applicable algorithms  depend on the kind of constraints that occur in $Q$ and are addressed in Sections \ref{sec:reasoning} and \ref{sec:existing_representations}.

\paragraph{Model generation}
This task determines a solution for a QCN $Q$, i.e., a concrete assignment of a domain entity for each node in $Q$.
This may be computationally harder than merely deciding the existence of a solution.
For instance, Fig.~\ref{fig:desc_quant} could be the result of the model generation for the QCN shown in Fig.~\ref{fig:desc_qual_partial}.
Typically, a single QCN has infinitely many solutions, due to the abstract nature of qualitative representations.
Implementations of model generation may thus choose to introduce further parameters for controlling the kind of solution determined.
Techniques are described, e.g., in \cite{DBLP:conf/ecai/SchultzB12,KreutzmannW:2014:AndOrLP,SL15}.



\paragraph{Equivalence transformation}
Taking a QCN $Q$ as input, equivalence transformation methods determine a QCN $Q'$ that has exactly the same solutions but meets additional criteria.
Two variants are commonly considered.


{\slshape Smallest equivalent network representation} 
determines the strongest refinement of the input $Q$ by modifying its constraints in order to remove redundant information.
Figure \ref{fig:desc_qual_full} depicts a refinement of Figure \ref{fig:desc_qual_partial} since in \ref{fig:desc_qual_partial} the relation between $A$ and $C$ is not constrained at all (i.e., being ``$<, =, >$''), whereas \ref{fig:desc_qual_full} involves the tighter constraint ``$<$''.
Thus, the QCN $Q$ in \ref{fig:desc_qual_partial} contains 5 base relations, whereas the QCN $Q'$ in \ref{fig:desc_qual_full} contains only 3. 
Methods are addressed, e.g., by \citeN{Beek1991}, and \citeN{amaneddine-condotta-FLAIRS:13}.

{\slshape Most compact equivalent network representation} 
determines a QCN $Q'$ with a minimal number of constraints: 
it removes \emph{whole} constraints that are redundant. 
In that sense, Figure \ref{fig:desc_qual_partial} shows a more compact network than Figure \ref{fig:desc_qual_full}.
This task is addressed, e.g., by \citeN{Wallgruen_12_Exploiting}, and \citeN{matt_duckham_redundant_2014}.

\par\medskip\noindent
With this taxonomy in mind, the next section studies properties of qualitative representations and their reasoning operations.

\section{Qualitative Spatial and Temporal Calculi for Domain Representations}
\label{sec:qualitative_representations}

The notion of a qualitative spatial (or temporal) calculus
is a formal construct which, in one form or another,
underlies virtually every language
for qualitative domain representations.
In this section, we survey this fundamental construct,
formulate minimal requirements to a qualitative calculus,
discuss their relevance to spatial and temporal representation and reasoning,
and list calculi described in the literature. 
\begin{New}
As mentioned in Section \ref{ssec:what_is:taxonomy} domain entities can be of any type representing spatial or temporal aspects. 
As specific domain entities are rather impedimental to define a general notion of 'calculus' we omit a list of entities here. \todo{This sentence is a bit hard to understand. I think one could say: The notion of calculus has been devised to deal with any spatial and/or temporal domain entities, for example those in Table~\protect \ref{tab:calculi_new_1}. (Jae)}
Instead we refer to Table \ref{tab:calculi_new_1} listing entities which are covered by existing calculi so far.
\end{New}
%
%
%

Existing spatial and temporal calculi 
are  entirely based on binary or ternary relations
between spatial and temporal entities,
which comprise, for example, points, lines, intervals, or regions.
\emph{Binary} relations are used to represent
the location or moving direction of two entities relative to one another
\emph{without} referring to a third entity as a reference object.
Examples of relations are 
 ``overlaps with'' (for intervals or regions)
or ``move towards each other'' (for dynamic objects).
Additionally, a binary calculus is equipped with
a converse operation acting on single relation symbols
and a binary composition operation acting on pairs of relation symbols,
representing the natural converse and composition operations on the domain relations, respectively.
Converse and composition play a crucial role
for symbolic reasoning: from the knowledge that the pair $(x,y)$
of entities is in relation $r$, a symbolic reasoner can conclude
that $(y,x)$ is in the converse of $r$;
and if it is additionally known that the pair $(y,z)$ is in $s$,
then the reasoner can conclude that $(x,z)$ is in the relation
resulting in composing $r$ and $s$.
In addition, most calculi provide an identity relation
which allows to represent the (explicit or derived)
knowledge that, for example, $x$ and $y$ represent the same entity. 
\begin{example}%
  \begin{New}%
    \label{exa:PC_symbolic_constituents}%
    The one-dimensional point calculus \PC1 \cite{vilain-kautz-aaai:86}
    symbolically represents the relations $<,=,>$ between points on a line
    (which may model points in time), see Figure~\ref{fig:relations_PC1+RCC5+CYCb}\,(a).
    These three relations are called \emph{base relations} in Def.\ \ref{def:JEPD};
    \PC1 additionally represents all their unions and intersections:
    the empty relation and $\leqslant, \geqslant, \neq, \lesseqqgtr$.
    The calculus provides the relation symbols
    \texttt{<}, \texttt{=}, and \texttt{>};
    sets of symbols represent unions of base relations,
    e.g., $\{\texttt{<}, \texttt{=}\}$ represents $\leqslant$.
    The symbol $\texttt{=}$ represents the identity $=$\,.

    \PC1 further provides converse and composition.
    For example, the converse of  \texttt{\textless} is~\texttt{>}:
    whenever $x<y$, it follows that $y>x$;
    the composition of \texttt{<} with itself is again \texttt{<}:
    whenever $x<y$ and $y<z$, we have $x<z$.
    \PC1 represents the converse as a list of size 3 (the converses of all relation symbols)
    and the composition as a table of size $3 \times 3$ (one composition result for each pair of relation symbols).
    \Endofexa
  \end{New}%
\end{example}%
  \begin{figure}[ht]
    \begin{New}%
      \begin{centering}
        \includegraphics{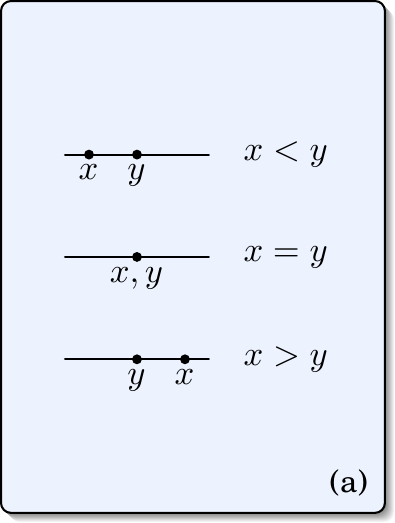}%
        \hspace*{\fill}%
        \includegraphics{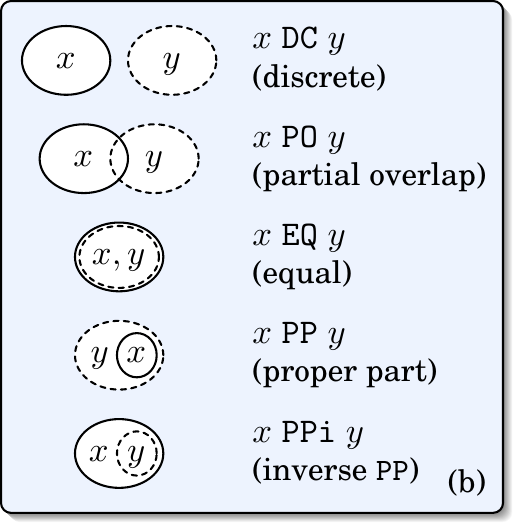}%
        \hspace*{\fill}%
        \includegraphics{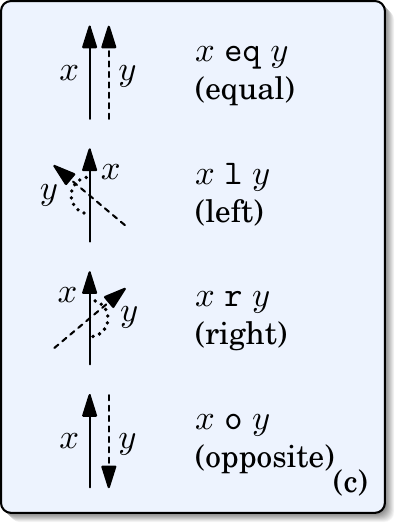}%
      \end{centering}

      \caption{\new{Illustration of the base relations for the calculi (a) \PC1, (b) \calc{RCC-5}, and (c) \CYCb}}
      \label{fig:relations_PC1+RCC5+CYCb}
    \end{New}%
  \end{figure}
\begin{example}%
  \begin{New}%
    \label{exa:RCC5_symbolic_constituents}%
    The calculus \calc{RCC-5} \cite{randell-cui-cohn-KR:92}
    symbolically represents five binary topological relations
    between regions in space (which may model objects):
    ``is discrete from'', ``partially overlaps with'', ``equals'', ``is proper part of'', and ``has proper part'',
    plus their unions and intersections,
    see Figure~\ref{fig:relations_PC1+RCC5+CYCb}\,(b).
    For this purpose, \calc{RCC-5} provides the relation symbols
    \texttt{DC}, \texttt{PO}, \texttt{EQ}, \texttt{PP}, and \texttt{PPi}.
    The latter two are each other's converses;
    the first three are their own converses.
    The composition of \texttt{DC} and \texttt{PO} is $\{\texttt{DC},\texttt{PO},\texttt{PP}\}$
    because, whenever region $x$ is disconnected from $y$ and $y$ partially overlaps with $z$,
    there are three possible configurations between $x$ and $z$: those represented by $\texttt{DC},\texttt{PO},\texttt{PP}$.
    \Endofexa
  \end{New}%
\end{example}%
\begin{example}%
  \begin{New}%
    \label{exa:CYCb_symbolic_constituents}%
    The calculus \CYCb \cite{DBLP:journals/ai/IsliC00}
    symbolically represents four binary topological relations
    between orientations of points in the plane (which may model observers and their lines of vision):
    ``equals'', ``is opposite to'', and ``is to the left/right of'',
    plus their unions and intersections, see Figure~\ref{fig:relations_PC1+RCC5+CYCb}\,(c).
    For this purpose, \CYCb provides the relation symbols
    \texttt{e}, \texttt{o}, \texttt{l}, and \texttt{r}.
    The latter two are each other's converses;
    \texttt{e} and \texttt{o} are their own converses.
    The composition of \texttt{l} and \texttt{r} is $\{\texttt{e},\texttt{l},\texttt{r}\}$:
    whenever orientation $x$ is to the left of $y$ and $y$ is to the left of $z$,
    then $x$ can be equal to, to the left of, or to the right of $z$.
    \Endofexa
  \end{New}%
\end{example}%
Depending on the properties postulated
for converse and composition, notions of a calculus of \new{varying strengths} 
exist \cite{DBLP:journals/ki/NebelS02,LigozatR04}.
The algebraic properties of binary calculi are well-understood,
see Section \ref{sec:relation_algebras}.

The main motivation for using \emph{ternary} relations is the requirement of directly
capturing \emph{relative frames of reference} which occur in natural language
semantics \cite{Levinson03:space}. 
In these frames of reference, the location of a target object is
described from the perspective of an observer with respect to a reference object.
For example, a hiker may describe a mountain peak to be to the left of a lake
with respect to her own point of view. \new{Another important motivation is the ability
to express that an object is located between two others.} Thus,
ternary calculi typically contain projective relations for describing relative orientation
\new{and/or betweenness}.
The commitment to ternary (or $n$-ary) relations
complicates matters significantly:
instead of a single converse operation,
there are now five (or $n!-1$) nontrivial permutation operations,
and there is no longer a unique choice for a natural composition operation.
For capturing the algebraic structure of
$n$-ary relations, \citeN{Condotta2006} proposed an algebra
but there are other arguably natural choices,
and they lead to different algebraic properties, as shown
in Section \ref{sec:relation_algebras}. These difficulties may be the main reason
why algebraic properties of ternary calculi are not as deeply studied as for binary calculi.
Fortunately, this will not prevent us from establishing our general notion of a qualitative
spatial (or temporal) calculus with relation symbols of arbitrary arity.
However, we will then restrict our algebraic study to binary calculi;
a unifying algebraic framework for $n$-ary calculi
has yet to be established.


\subsection{Requirements to Qualitative Spatial and Temporal Calculi}
\label{sec:requirements}

We start with minimal requirements used in the literature.
We use the following standard notation.
A \emph{universe} is a non-empty set \Univ.
With $X^n$ we denote the set of all $n$-tuples with elements from $X$.
An \emph{$n$-ary domain relation} is a subset
$r \subseteq \Univ^n$. We use the prefix notation $r(x_1,\dots,x_n)$
to express $(x_1,\dots,x_n) \in r$;
in the binary case we will often use the infix notation 
$x\,r\,y$ instead of $r(x,y)$.

\paragraph*{Abstract partition schemes}
\citeN{LigozatR04} note that most spatial and temporal calculi
are based on a set of JEPD (jointly exhaustive and pairwise disjoint) domain relations.
The following definition is predominant in the QSTR literature \cite{LigozatR04,Cohn:2008vn}.
\begin{definition}
  \label{def:JEPD}
  Let $\Univ$ be a universe and $\URel$ a set of non-empty domain relations of the same arity $n$.
  $\URel$ is called a set of \emph{JEPD relations} over $\Univ$ if
  the relations in $\URel$ are jointly exhaustive, i.e., $\Univ^n = \bigcup_{r \in \URel} r$, and pairwise disjoint.

  An \emph{$n$-ary abstract partition scheme} is a pair $(\Univ,\URel)$
  where $\URel$ is a set of JEPD relations over the universe $\Univ$.
  The relations in $\URel$ are called \emph{base relations}. 
\end{definition}
\begin{example}%
  \begin{New}%
    \label{exa:PC_underlying_APS}%
    The calculus \PC1 is based on the binary abstract partition scheme 
    $\PSPC1 := (\mathbb{R},\{<,=,>\})$
    where $\mathbb{R}$ is the set of reals
    and $\{<,=,>\}$ is clearly JEPD.
    For \calc{RCC-5}, the universe is often chosen to be the set of all regular closed subsets
    of the 2- or 3-dimensional space $\mathbb{R}^2$ or $\mathbb{R}^3$.
    The five base relations from Figure~\ref{fig:relations_PC1+RCC5+CYCb}\,(b) are JEPD.
    For \CYCb, the universe is the set of all oriented line segments in the plane $\mathbb{R}^2$,
    given by angles between 0° and 360°.
    The four base relations from Figure~\ref{fig:relations_PC1+RCC5+CYCb}\,(c) are JEPD.
    \Endofexa
  \end{New}%
\end{example}%
In Definition \ref{def:JEPD},
the universe \Univ represents the set of all spatial (or temporal) entities.
The main ingredients of a calculus will be
relation symbols representing the base relations
in the underlying partition scheme.
A constraint linking an $n$-tuple $t$ of entities via a relation symbol
will thus represent complete information (modulo the qualitative abstraction underlying the partition scheme)
about $t$.
Incomplete information is modeled by $t$ being in a \emph{composite relation},
which is a set of relation symbols representing the union of the corresponding base relations.
The set of all relation symbols represents the \emph{universal relation} (the union of all base relations)
and indicates that no information is available. 
\begin{example}%
\begin{New}%
  \label{exa:PC_constraints}%
  In \PC1, ``$x \texttt{~<~} y$'' represents the relationship
  $a<b$, which holds complete information because $<$ is atomic in \PSPC1.
  The statement ``$x ~\{\texttt{<},\texttt{=}\}~ y$'' represents the coarser relationship
  $a\leqslant b$ holding the incomplete information ``$a<b$ or $a=b$''.
  Clearly ``$x ~\{\texttt{<},\texttt{=},\texttt{>}\}~ y$''
  holds no information:
  ``$a<b$ or $a=b$ or $a>b$''
  is always true.
  \Endofexa
\end{New}%
\end{example}%
The requirement that all base relations are JEPD
ensures that 
every $n$-tuple of entities belongs to exactly one base relation.
Thanks to PD (pairwise disjointness), there is a unique way to represent
any composite relation using relation symbols and, due to JE (joint exhaustiveness), the empty relation
can never occur in a consistent set of constraints,
which is relevant for
reasoning, see Section \ref{sec:reasoning}. 
\begin{example}%
\begin{New}%
  \label{exa:PC_non-JEPD}%
  Consider the modification $\PC1'$ based on the non-PD set $\{\leqslant,\geqslant\}$.
  Then the relationship $a=b$ can be expressed in two ways
  using relation symbols \texttt{<=} and \texttt{>=} representing $\leqslant$ and $\geqslant$:
  ``$x \texttt{~<=~} y$'' and 
  ``$x \texttt{~>=~} y$''.

  Conversely, consider the variant $\PC1''$ based on the non-JE set $\{<,>\}$.
  Then the constraint $a=b$ cannot be expressed.
  Therefore, in any given set of constraints
  where it is known that $x,y$ stand for identical entities,
  we would find the empty relation between $x,y$.
  The standard reasoning procedure described in Section \ref{sec:reasoning}
  would declare such sets of constraints to be inconsistent,
  although they are not -- we have simply not been able to express $x=y$.
  \Endofexa
\end{New}%
\end{example}%
\paragraph*{Partition schemes, identity, and converse}
\citeN{LigozatR04} base their definition of a
(binary) qualitative calculus on the notion of a \emph{partition scheme},
which imposes additional requirements on an abstract partition scheme.
In particular, it requires that the set of base relations
contains the identity relation and is closed under the converse operation.
The analogous definition by \citeN{Condotta2006}
captures relations of arbitrary arity.
Before we define the notion of a partition scheme,
we discuss the generalization of identity and converse
to the $n$-ary case.

The binary identity relation is given as usual by 
\begin{equation}
  \label{eq:id_binary}
  \id^2 = \{(u,u) \mid u \in \Univ\}.
\end{equation}%

\begin{example}%
\begin{New}%
  \label{exa:PC_identity}
  Clearly, $=$ in \PSPC1 and ``equals'' in \PSRCC5 and \PSCYCb are the identity relation over the respective domain.
  \Endofexa
\end{New}%
\end{example}%
The most inclusive way to generalize \eqref{eq:id_binary} to the $n$-ary case
is to fix a set $M$ of
numbers of all positions where tuples in $\id^n$ are required to agree.
Thus, an \emph{$n$-ary identity relation} is a domain relation $\id^n_M$
with $M \subseteq \{1,\dots,n\}$ and $|M| \geqslant 2$,
which is defined by
\begin{align*}
  \id^n_M    & = \{(u_1,\dots,u_n) \in \Univ^n \mid u_i=u_j \text{~for all $i,j \in M$}\}.
\end{align*}
This definition subsumes the ``diagonal elements'' $\Delta_{ij}$ of \citeN{Condotta2006}
for the case $|M|=2$.
\begin{New}
  However, it is not enough to restrict attention to $|M|=2$
  because there are ternary calculi which contain all identities
  $\id^3_{1,2}$, $\id^3_{1,3}$, $\id^3_{2,3}$, and $\id^3_{1,2,3}$,
  an example being the LR calculus, which was described as
  ``the finest of its class'' \cite{SN05}.
  Since the relations in an $n$-ary abstract partition scheme are JEPD,
  all identities $\id^n_M$ are either base relations
  or subsumed by those.
  The stronger notion of a \emph{partition scheme}
  should thus require that all identities be made explicit.
\end{New}

For binary relations, 
$\id^2$ from \eqref{eq:id_binary} is the \emph{unique} identity relation
$\id^2_{\{1,2\}}$.

\par\medskip\noindent
The standard definition for the converse operation $\breve{~}$ on binary relations is 
\begin{equation}
  \label{eq:converse_binary}
  r\breve{~} = \{(v,u) \mid (u,v) \in r\}.
\end{equation}

\begin{example}%
\begin{New}%
  \label{exa:PC_converse}%
  In \PSPC1 we have that
  $\mathord{<}\breve{~}$ is $\mathord{>}$;~
  $\mathord{=}\breve{~}$ is $\mathord{=}$;~
  $\mathord{>}\breve{~}$ is $\mathord{<}$.
  The converses of the base relations in \PSRCC5 and \PSCYCb
  were named in Examples~\ref{exa:RCC5_symbolic_constituents} and~\ref{exa:CYCb_symbolic_constituents}.
  \Endofexa
\end{New}%
\end{example}%
In order to generalize the reversal of the pairs $(u,v)$ in \eqref{eq:converse_binary}
to $n$-ary tuples,
we consider arbitrary permutations of $n$-tuples.
An \emph{$n$-ary permutation} is a bijection $\pi : \{1,\dots,n\} \to \{1,\dots,n\}$.
We use the notation $\pi : (1,\dots,n) \mapsto (i_1,\dots,i_n)$ as an
abbreviation for ``$\pi(1) = i_1$, \dots, $\pi(n) = i_n$''.
The identity permutation $\iota : (1,\dots,n) \mapsto (1,\dots,n)$
is called \emph{trivial}; all other permutations are \emph{nontrivial}.

A finite set $P$ of $n$-ary permutations is called \emph{generating}
if each $n$-ary permutation is a composition of permutations from $P$.
For example, the following two permutations form a (minimal) generating set:
\begin{align*}
  \shortcut & : (1,\dots,n) \mapsto (2,\dots,n,1)       & & \text{(shortcut)} \\
  \homing   & : (1,\dots,n) \mapsto (1,\dots,n-2,n,n-1) & & \text{(homing)}
  \intertext{The names have been introduced in \citeN{FZ92utilization} for ternary permutations,
  together with a name for a third distinguished permutation:}
  \inversion  & : (1,\dots,n) \mapsto (2,1,3\dots,n)      & & \text{(inversion)}
\end{align*}
\citeN{Condotta2006} call
shortcut ``rotation'' ($r^\curvearrowright$)
and homing ``permutation'' ($r^\looparrowright$).
\begin{New}
\begin{example}
In Figure \ref{fig:DCCperm} we depict the permutations
\shortcut\ (rotation), \homing\ (permutation), and \inversion\
for one relation from the ternary Double Cross Calculus (\calc{2-cross}) \cite{FZ92utilization}. 
The \calc{2-cross} relations specify the location of a point $P_3$
relative to an oriented line segment given by two points $P_1,P_2$.
Figure \ref{fig:DCCperm}\,a shows the relation \texttt{right-front}.
The relations resulting from applying the permutations
are depicted in Figure~\ref{fig:DCCperm}\,b;
e.g., $\shortcut(\texttt{right-front}) = \texttt{right-back}$
because the latter is $P_1$'s position relative to the line segment $\overrightarrow{P_2P_3}$.
Figure~\ref{fig:DCCperm}\,c will be relevant later.
%
\end{example}
\end{New}

\usetikzlibrary{backgrounds}
\usetikzlibrary{arrows}
\newcommand*{\cwidth}{.2}
\newcommand*{\dy}{.8}
\newcommand*{\dcu}{2.4*\cwidth}
\newcommand*{\dcx}{1.2*\dy}
\newcommand*{\dcy}{.8*\dy}
\newcommand*{\rota}{26.6}
\newcommand*{\db}{.46*\dy}
\pgfmathsetmacro{\cwf}{\cwidth/sqrt(5)}
\pgfmathsetmacro{\dyf}{\dy/sqrt(5)}
\pgfmathsetmacro{\rotb}{45+\rota/2}
\tikzstyle{bordered} = [draw,outer sep=0,inner sep=0,minimum size=10]
\tikzstyle{mcs_plain} = [fill=white]
\tikzstyle{mcs_emph} = [red,draw=red,fill=white]

\newcommand{\setcoordinates}{%
  \coordinate (o0) at (0,0);
  \coordinate (o1) at ($ (o0)+(2.7*\dy,0) $);
  \coordinate (o2) at ($ (o1)+(2.9*\dy,0) $);
  \coordinate (o3) at ($ (o2)+(2.9*\dy,0) $);
  \coordinate (o4) at ($ (o3)+(2.8*\dy,0) $);
  \coordinate (o5) at ($ (o4)+(3.3*\dy,0) $);

  \coordinate (p2) at (0,0);
  \coordinate (p1) at (0,-\dy);
  \coordinate (p3) at (\rota:\dy);
  \coordinate (p4) at ($ (p3)+(270+2*\rota:\dy) $);

  \coordinate (p21) at ($ (p2)+(0,-.8*\cwidth) $);
  \coordinate (p32) at ($ (p3)+(\rota:-.8*\cwidth) $);
  \coordinate (p31) at ($ (p3)+(\rotb:-.8*\cwidth) $);
  \coordinate (p12) at ($ (p1)+(0,.8*\cwidth) $);

  \coordinate (p2m_id) at ($ (p2)+(0,\dcy) $);
  \coordinate (p1m_id) at ($ (p1)+(0,-\dcu) $);
  \coordinate (p2a_id) at ($ (p2)+(-\dcu,0) $);
  \coordinate (p2b_id) at ($ (p2)+(\dcx,0) $);
  \coordinate (p1a_id) at ($ (p1)+(-\dcu,0) $);
  \coordinate (p1b_id) at ($ (p1)+(\dcx,0) $);

  \coordinate (p3m_sc) at ($ (p3)+(\rota:\dcu) $);
  \coordinate (p3a_sc) at ($ (p3)+(\rota+90:\dcu) $);
  \coordinate (p3b_sc) at ($ (p3)+(\rota-90:\dcx) $);
  \coordinate (p2a_sc) at ($ (p2)+(\rota+90:\dcu) $);
  \coordinate (p2m_sc) at ($ (p2)+(\rota+180:\dcy) $);
  \coordinate (p2b_sc) at ($ (p2)+(\rota+270:\dcx) $);

  \coordinate (p3m_hm) at ($ (p3)+(\rotb:\dcu) $);
  \coordinate (p3a_hm) at ($ (p3)+(\rotb+90:\dcy) $);
  \coordinate (p3b_hm) at ($ (p3)+(\rotb-90:\dcu) $);
  \coordinate (p1a_hm) at ($ (p1)+(\rotb+90:\dcy) $);
  \coordinate (p1m_hm) at ($ (p1)+(\rotb+180:\dcu) $);
  \coordinate (p1b_hm) at ($ (p1)+(\rotb+270:\dcu) $);

  \coordinate (p2m_c2) at ($ (p2)+(0,\dcu) $);
  \coordinate (p1m_c2) at ($ (p1)+(0,-\dcu) $);
  \coordinate (p2a_c2) at ($ (p2)+(-\dcu,0) $);
  \coordinate (p2b_c2) at ($ (p2)+(\dcu,0) $);
  \coordinate (p1a_c2) at ($ (p1)+(-\dcu,0) $);
  \coordinate (p1b_c2) at ($ (p1)+(\dcu,0) $);

  \coordinate (sub_id) at ($ (p1)+(-\rota:\dy)+(0,-.12*\dcu) $);
  \coordinate (sub_sc) at ($ (sub_id)+(.6*\dy,0) $);
  \coordinate (sub_hm) at ($ (sub_id)+(.3*\dy,0) $);
  \coordinate (sub_inv) at (sub_id);


  \coordinate (bunch3l) at ($ (p3)+(-\db,0) $);
  \coordinate (bunch3r) at ($ (p3)+( \db,0) $);
  \coordinate (bunch3b) at ($ (p3)+(0,-\db) $);
  \coordinate (bunch3t) at ($ (p3)+(0, \db) $);

  \coordinate (bunch4r) at ($ (p4)+(\rota:\db) $);
  \coordinate (bunch4t) at ($ (p4)+(90+\rota:\db) $);
  \coordinate (bunch4l) at ($ (p4)+(180+\rota:\db) $);
  \coordinate (bunch4b) at ($ (p4)+(270+\rota:\db) $);
  
  \coordinate (P4boxA) at ($ (o5)+(p1)+( .05*\dcu,-1.1*\dcu) $);
  \coordinate (P4boxB) at ($ (o5)+(p1)+(1.1*\dcu,-1.1*\dcu) $);
  \coordinate (P4boxC) at ($ (o5)+(p2)+(1.1*\dcu, 1.1*\dcu) $);
  \coordinate (P4boxD) at ($ (o5)+(p2)+( .05*\dcu, 1.1*\dcu) $);

  \coordinate (boxA1) at ($ (o0)+(p1)+(-\dcu,-1.2*\dcu)+(-1mm,-1mm) $);
  \coordinate (boxB1) at ($ (o0)+(p1)+( \dcx,-1.2*\dcu)+( 1mm,-1mm) $);
  \coordinate (boxC1) at ($ (o0)+(p2)+( \dcx,1.2*\dcy)+( 1mm,1mm) $);
  \coordinate (boxD1) at ($ (o0)+(p2)+(-\dcu,1.2*\dcy)+(-1mm,1mm) $);

  \coordinate (boxA2) at ($ (o1)+(p1)+(-\dcu,-1.2*\dcu)+(-2mm,-1mm) $);
  \coordinate (boxB2) at ($ (o3)+(p1)+( \dcx,-1.2*\dcu)+( 1mm,-1mm) $);
  \coordinate (boxC2) at ($ (o3)+(p2)+( \dcx,1.2*\dcy)+( 1mm,1mm) $);
  \coordinate (boxD2) at ($ (o1)+(p2)+(-\dcu,1.2*\dcy)+(-2mm,1mm) $);

  \coordinate (boxA3) at ($ (o4)+(p1)+(-1.1*\dcu,-1.2*\dcu)+(-2mm,-1mm) $);
  \coordinate (boxB3) at ($ (o5)+(p1)+( 1.2*\dcu,-1.2*\dcu)+( 1mm,-1mm) $);
  \coordinate (boxC3) at ($ (o5)+(p2)+( 1.2*\dcu,1.2*\dcy)+( 1mm,1mm) $);
  \coordinate (boxD3) at ($ (o4)+(p2)+(-1.1*\dcu,1.2*\dcy)+(-2mm,1mm) $);
}

\begin{figure}%
\centering
  \includegraphics{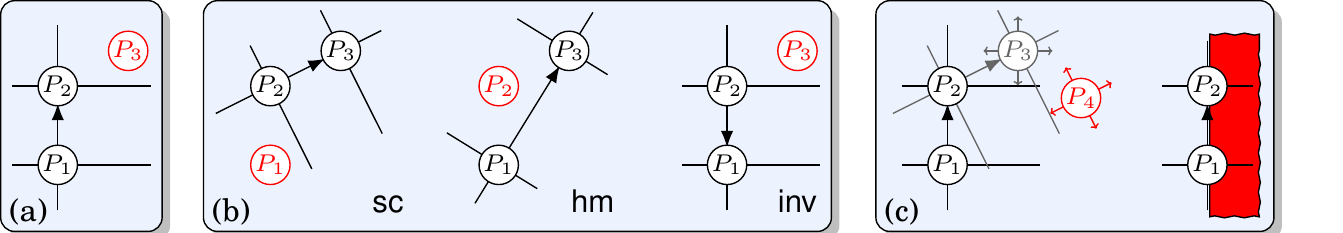}

  \caption{%
    ~(a) The \calc{2-cross} relation \texttt{right-front};
    ~(b) permutations of \texttt{right-front};
    ~(c) the composition $\texttt{right-front} \circ \texttt{right-front}$%
  }
  \label{fig:DCCperm}
\end{figure}

For $n=2$, \shortcut, \homing and \inversion coincide;
indeed, there is a unique minimal generating set, which consists of the single
permutation $\breve{~} : (1,2) \mapsto (2,1)$. For $n \geqslant 3$,
there are several generating sets, e.g., $\{\shortcut,\homing\}$ and $\{\inversion,\homing\}$.

Now an \emph{$n$-ary permutation operation} is a map $\cdot^\pi$
that assigns to each $n$-ary domain relation $r$ an $n$-ary domain relation
denoted by $r^\pi$, where $\pi$ is an $n$-ary permutation and the following holds:
$
r^\pi = \{(u_{\pi(1)},\dots,u_{\pi(n)}) \mid (u_1,\dots,u_n ) \in r\}
$

\par\medskip\noindent
We are now ready to give our definition of a partition scheme,
lifting Ligozat and Renz's binary version to the $n$-ary case,
and generalizing Condotta et al.'s $n$-ary version to arbitrary generating sets.
\begin{definition}
  \label{def:partition_schemes}
  An \emph{$n$-ary partition scheme}
  $(\Univ,\URel)$
  is an $n$-ary abstract partition scheme
  with the following two additional properties.
  \begin{enumerate}
    \item
      $\URel$ contains all identity relations $\id^n_M$, $M \subseteq \{1,\dots,n\}$, $|M| \geqslant 2$.
    \item
      There is a generating set $P$ of permutations
      such that, for every $r \in \URel$ and every $\pi \in P$,
      there is some $s \in \URel$ with $r^\pi = s$.
  \end{enumerate}
\end{definition}
In the special case of binary relations, we have the following.
\begin{observation}
  \label{obs:binary_partition_scheme}
  A binary partition scheme
  $(\Univ,\URel)$
  is a binary abstract partition scheme
  with the following two additional properties.
  \begin{enumerate}
    \item
      $\URel$ contains the identity relation $\id^2$.
    \item
      For every $r \in \URel$, there is some $s \in \URel$ such that $r\breve{~} = s$. 
  \end{enumerate}
\end{observation}
\begin{example}%
\begin{New}%
  \label{exa:PC_is_PS}%
  It follows that \PSPC1, \PSRCC5, and \PSCYCb are even partition schemes.
  In contrast, the abstract partition scheme $(\mathbb{R},\{\leqslant,>\})$
  is not a partition scheme: it violates both conditions of 
  Observation \ref{obs:binary_partition_scheme} (and thus of Definition \ref{def:partition_schemes}).
  \Endofexa
\end{New}%
\end{example}%
\begin{example}%
\begin{New}%
  \label{exa:CDR_is_not_PS}%
  As an example of an intuitive and useful abstract partition scheme that is \emph{not} a partition scheme,
  consider the calculus
  \emph{Cardinal Direction Relations} (\calc{CDR}) \cite{SkiadoK05}.
  CDR describes the placement of regions in a 2D space (e.g., countries on the globe)
  relative to each other, and with respect to a fixed coordinate system.
  The axes of the bounding box of the \emph{reference region} $y$
  divide the space into nine \emph{tiles}, see Fig.~\ref{fig:relations_CDR}\,(a).
  The binary relations in \PSCDR now determine which tiles relative to $y$
  are occupied by a \emph{primary region} $x$: e.g., in Fig.~\ref{fig:relations_CDR}\,(b),
  tiles N, W, and B of $y$ are occupied by $x$; hence we have \mbox{$x$ N:W:B $y$}.
  Simple combinatorics yields $2^9-1=511$ base relations.

  Now \PSCDR is not a partition scheme
  because it violates Condition 2 of
  Observation \ref{obs:binary_partition_scheme} (and thus of Definition \ref{def:partition_schemes}):
  e.g., the converse of the base relation S (south) is not a base relation.
  To justify this claim, assume the contrary. Take two specific regions $x,y$ with \mbox{$x$ S $y$},
  namely two unit squares, where $y$ is exactly above $x$.
  Then we also have \mbox{$y$ N $x$}; therefore the converse of S is N.
  Now stretch the width of $x$ by any factor $>1$. Then we still have \mbox{$y$ N $x$},
  but no longer \mbox{$x$ S $y$}. Hence the converse of S cannot comprise all of N,
  which contradicts the assumption that the converse of S is a base relation.

  The related calculus \calc{RCD} \cite{NavarreteEtAl13}
  abstracts away from the concrete shape of regions and replaces them with their bounding boxes,
  see Fig.~\ref{fig:relations_CDR}\,(c).
  \PSRCD is not a partition scheme, with the same argument from above.
  \Endofexa
\end{New}%
\end{example}
\begin{figure}[ht]
\begin{New}%
  \begin{centering}
    \includegraphics{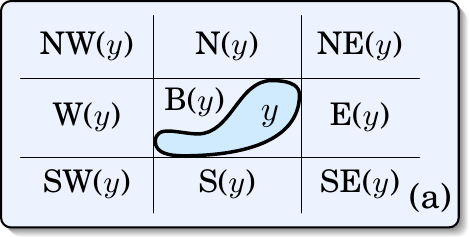}%
    \hspace*{\fill}%
    \includegraphics{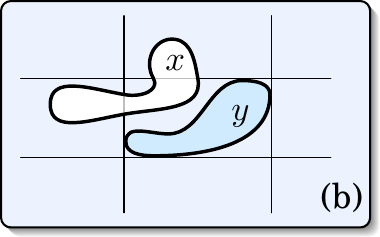}%
    \hspace*{\fill}%
    \includegraphics{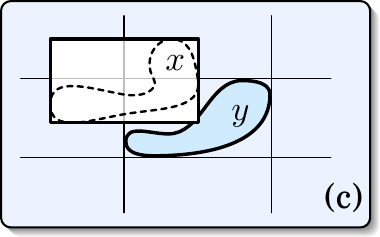}%
  \end{centering}

  \caption{\new{%
    Calculi \calc{CDR} and \calc{RCD}:
    (a) reference tiles;
    (b) the CDR base relation \mbox{$x$ N:W:B $y$;}
    (c) the RCD base relation \mbox{$x$ NW:N:W:B $y$}%
  }}
  \label{fig:relations_CDR}%
\end{New}%
\end{figure}%
\begin{New}%
It is important to note that violations of Definition~\ref{def:partition_schemes}
such as those reported in Example~\ref{exa:CDR_is_not_PS}
are not necessarily bugs in the design of the respective calculi --
in fact they are often a feature of the corresponding representation language,
which is deliberately designed to be just as fine as necessary,
and may thus omit some identity relations or converses/compositions of base relations.
To turn, say, \calc{CDR} into a partition scheme,
one would have to break down the 511 base relations into smaller ones,
resulting in even more, less cognitively plausible ones.
Thus violations of Definition~\ref{def:partition_schemes} are unavoidable,
and we adopt the more general notion of an abstract partition
scheme.%
\end{New}%

\paragraph*{Calculi}
Intuitively, a qualitative spatial (or temporal) calculus
is a symbolic representation of an abstract partition scheme
and additionally represents the composition operation on the relations involved.
As before, we need to discuss the generalization
of binary composition to the $n$-ary case
before we can define it precisely.

For binary domain relations, the standard definition of composition is: 
\begin{equation}
  r \circ s = \{(u,w) \mid \exists \underline{v} \in \Univ :
  (u,\underline{v}) \in r \text{ and } (\underline{v},w) \in s\}
  \label{eq:binary_composition}
\end{equation}

\begin{example}%
\begin{New}%
  \label{exa:PC_composition}%
  In \PSPC1 we have, e.g., that
  $\mathord{<} \circ \mathord{<}$ equals $\mathord{<}$
  because $a<b$ and $b<c$ imply $a<c$.
  Furthermore, $\mathord{<} \circ \mathord{>}$ yields the universal relation,
  i.e., the union of $<$, $=$, and $>$,
  because ``$a<b$ and $b>c$'' is consistent with each of
  $a<c$, $a=c$, and $a>c$.
%
  \Endofexa
\end{New}%
\end{example}%
We are aware of three ways to generalize \eqref{eq:binary_composition} to higher arities.
The first is a binary operation on the ternary relations of the calculus \calc{2-cross} \cite{Freksa92,FZ92utilization},
see also Fig.~\ref{fig:DCCperm}:
\begin{equation}
  r \mathbin{\circ^3_{\text{FZ}}} s = \{(u,v,w) \mid \exists \underline{x} :
  (u,v,\underline{x}) \in r \text{ and } (v,\underline{x},w) \in s\}
  \notag
\end{equation}
\begin{New}
%
  It says:
  if the location of $x$ relative to $u$ and $v$ is determined by $r$
  and the location of $w$ relative to $v$ and $x$ is determined by $s$,
  then the location of $w$ relative to $u$ and $v$ is determined by $r \mathbin{\circ^3_{\text{FZ}}} s$.
  Fig.~\ref{fig:DCCperm}\,c
  shows the composition of the \calc{2-cross} relation $\texttt{right-front}$
  with itself; i.e., $\texttt{right-front} \mathbin{\circ^3_{\text{FZ}}} \texttt{right-front}$.
  The red area indicates the possible locations of the point $P_4$;
  hence the resulting relation is $\{\texttt{right-front},\texttt{right-middle},\texttt{right-back}\}$.
  A generalization to other calculi and arities $n > 3$ is obvious.
\end{New}



A second alternative results in $n(n-1)$ binary operations $\mathbin{{}_i\circ_j^n}$ \cite{DBLP:journals/ai/IsliC00,SN05}:
the composition of $r$ and $s$ consists of those $n$-tuples
that belong to $r$ (respectively, $s$)
if the $i$-th (respectively, $j$-th) component
is replaced by some uniform element $v$.
\begin{align*}
  r \mathbin{{}_i\circ_j^n} s = \{(u_1,\dots,u_n) \mid \exists \underline{v} :~
  & (u_1,\dots,u_{i-1},\underline{v},u_{i+1},\dots,u_n) \in r \text{ and } \\
  & (u_1,\dots,u_{j-1},\underline{v},u_{j+1},\dots,u_n) \in s\qquad \}
\end{align*}
In the ternary case, this yields, for example:
\begin{equation}
  \label{eq:composition_3_2_3}
  r \mathbin{{}_3\circ_2^3} s = \{(u,v,w) \mid \exists \underline{x} :
  (u,v,\underline{x}) \in r \text{ and } (u,\underline{x},w) \in s\}
\end{equation}
If we assume, for example, that the underlying partition scheme speaks about the relative position of points,
we can consider \eqref{eq:composition_3_2_3} to say:
\new{if the position of $x$ relative to $u$ and $v$ is determined by the relation $r$
(as given by $(u,v,x) \in r$)
and the position of $w$ relative to $u$ and $x$ is determined by the relation $s$
(as given by $(u,x,w) \in s$),}
then the position of $w$ relative to $u$ and $v$ can be inferred to be determined by $r \mathbin{{}_3\circ_2^3} s$.

The third is perhaps the most general,
resulting in an $n$-ary operation \cite{Condotta2006}:
$\circ(r_1,\dots,r_n)$
consists of those $n$-tuples which, for every $i=1,\dots,n$, belong to the relation $r_i$
whenever their $i$-th component is replaced by some uniform $v$.
\begin{align}
  \circ(r_1, \dots, r_n) = \{ & (u_1,\dots,u_n) \mid \exists \underline{v} \in \Univ : (u_1,\dots,u_{n-1},\underline{v}) \in r_1 ~\text{and}~ \notag \\
                              & (u_1,\dots,u_{n-2},\underline{v},u_n) \in r_2 ~\text{and}~
                                \dots ~\text{and}~
                                (\underline{v},u_2,\dots,u_n) \in r_n\}
  \label{eq:n-ary_composition}
\end{align}
\begin{New}
  As an example,
  consider again $n=3$ 
  and \calc{2-cross}.
  Equation~\eqref{eq:n-ary_composition}
  says that the composition result
  of the relations \texttt{right-front}, \texttt{right-front}, and \texttt{left-back}
  is the set of all triples $(u_1,u_2,u_3)$ such that there is an element $v$
  with $(u_1,u_2,v) \in \texttt{right-front}$, $(u_1,v,u_3) \in \texttt{right-front}$,
  and $(v,u_2,u_3) \in \texttt{right-back}$.
  That set is exactly the relation \texttt{right-front},
  which can be seen drawing pictures similar to Fig.~\ref{fig:DCCperm}.%
\end{New}

%

For binary domain relations, all these alternative approaches collapse to \eqref{eq:binary_composition}.

In the light of the diverse views on composition,
we define a \emph{composition operation on $n$-ary domain relations} to be an operation of arity $2 \leqslant m \leqslant n$ on $n$-ary domain relations,
without imposing additional requirements.
Those are not necessary for the following definitions, which are independent of the particular choice of composition.

\par\medskip\noindent
We now define our minimal notion of a spatial
calculus, which provides a set of symbols
for the relations in an \emph{abstract} partition scheme ($\Rel$),
and for \emph{some choice of} nontrivial permutation operations ($\breve{~}^1,\dots,\breve{~}^k$)
and \emph{some} composition operation ($\diamond$).
\begin{definition}
  \label{def:qualitative_calculus}
  An \emph{$n$-ary qualitative calculus}
  is a tuple $(\Rel,\Int,\breve{~}^1,\dots,\breve{~}^k,\diamond)$ with $k\geqslant 1$ and the following properties.
  \begin{Itemize}
    \item
      $\Rel$ is a finite, non-empty set of $n$-ary \emph{relation symbols} (denoted $r,s,t,\dots$).
      The subsets of $\Rel$, including singletons, are called \emph{composite relations} (denoted $R,S,T,\dots$).
    \item
      $\Int = (\Univ, \varphi, \cdot^{\pi_1},\dots,\cdot^{\pi_k}, \circ)$ is an \emph{interpretation} with
      the following properties.
      \begin{Itemize}
        \item
          $\Univ$ is a universe.
        \item
          $\varphi : \Rel \to 2^{\Univ^n}$ is an injective map
          assigning an $n$-ary relation over $\Univ$ to each relation symbol, 
          such that $(\Univ, \{\varphi(r) \mid r \in \Rel\})$ is an abstract partition scheme.
          The map $\varphi$ is extended to composite relations $R \subseteq \Rel$ by setting
          $\varphi(R) = \bigcup_{r\in R}\varphi(r)$.
        \item
          $\{\cdot^{\pi_1},\dots,\cdot^{\pi_k}\}$ is a set of $n$-ary nontrivial permutation operations.
        \item
          $\circ$ is a composition operation on $n$-ary domain relations that has arity $2 \leqslant m \leqslant n$.
      \end{Itemize}
    \item
      Every \emph{permutation operation} $\breve{~}^i$ is a map
      $\breve{~}^i : \Rel \to 2^\Rel$ that satisfies
      \begin{equation}
        \varphi(r\breve{~}^i) \supseteq \varphi(r)^{\pi_i}
        \label{eq:abstract_converse}
      \end{equation}
      for every $r \in \Rel$.
      The operation $\breve{~}^i$ is extended to composite relations $R \subseteq \Rel$ by setting
      $R\breve{~}^i = \bigcup_{r\in R} r\breve{~}^i$.
    \item
      The \emph{composition operation} $\diamond$ is a map
      $\diamond : \Rel^m \to 2^\Rel$ that satisfies
      \begin{equation}
        \varphi(\diamond(r_1, \dots, r_m)) \supseteq \circ(\varphi(r_1), \dots, \varphi(r_m))
        \label{eq:abstract_composition}
      \end{equation}
      for all $r_1,\dots,r_m \in \Rel$.
      The operation $\diamond$ is extended to composite relations $R_1,\dots,R_m \subseteq \Rel$ by setting
      $\diamond(R_1, \dots, R_m) = \bigcup_{r_1\in R_1} \dots \bigcup_{r_m\in R_m} \diamond(r_1, \dots, r_m)$.
  \end{Itemize}  
\end{definition}
In the special case of binary relations, the natural converse is the only non-trivial permutation operation.
Hence $k=1$ and we have the following.
\begin{observation}
  A binary qualitative calculus
  is a tuple $(\Rel,\Int,\breve{~},\diamond)$ with the following properties.
  \begin{Itemize}
    \item
      $\Rel$ is as in Definition \ref{def:qualitative_calculus}.
    \item
      $\Int = (\Univ, \varphi, \pi, \circ)$ is an \emph{interpretation} with
      the following properties.
      \begin{Itemize}
        \item
          $\Univ$ is a universe.
        \item
          $\varphi : \Rel \to 2^{\Univ \times \Univ}$ is an injective map
          as in Definition \ref{def:qualitative_calculus}.
        \item
          $\pi$ is the standard converse operation on binary domain relations from \eqref{eq:converse_binary}.
        \item
          $\circ$ is the standard composition operation on binary domain relations from \eqref{eq:binary_composition}.
      \end{Itemize}
    \item
      The \emph{converse operation} $\breve{~}$ is a map
      $\breve{~} : \Rel \to 2^\Rel$ that satisfies  
      \begin{equation*}
        \forall r \in \Rel : \quad
        \varphi(r\breve{~}) \supseteq \varphi(r)^{\pi}\,.
      \end{equation*}
    \item
      The composition operation $\diamond$ is a map
      $\diamond : \Rel \times \Rel \to 2^\Rel$ that satisfies 
      \begin{equation*}
        \forall r,s \in \Rel : \quad
        \varphi(\diamond(r,s)) \supseteq \circ(\varphi(r), \varphi(s))\,.
      \end{equation*}
  \end{Itemize}  
\end{observation}
%
%
\new{Due to the last sentence of Definition~\ref{def:qualitative_calculus}},
the composition operation of a calculus is uniquely determined
by the composition of each pair of relation \emph{symbols}.
This information is usually stored in an $m$-dimensional table, the \emph{composition table}.
%
\begin{example}%
\begin{New}%
  \label{exa:PC_is_a_binary_calculus}%
  We can now observe that \PC1 is indeed a binary calculus with the following components.
  \begin{itemize}
    \item
      The set of relation symbols is $\Rel = \{\texttt{<},\texttt{=},\texttt{>}\}$,
      denoting the relations depicted in Figure~\ref{fig:relations_PC1+RCC5+CYCb}\,a.
      The $2^3=8$ composite relations include, for example,
      $R_1 = \{\texttt{<},\texttt{=}\}$ and
      $R_2 = \{\texttt{<},\texttt{=},\texttt{>}\}$.
    \item
      There are several possible interpretations, depending largely on the chosen universe.
      One of the most natural choices leads to the interpretation $\Int = \{\Univ,\varphi,\pi,\circ\}$
      with the following components.
      \begin{itemize}
        \item 
          The universe \Univ is the set of reals.
        \item
          The map $\varphi$ maps \texttt{<}, \texttt{=}, and \texttt{>}
          to $<$, $=$, and $>$, respectively; see Figure~\ref{fig:relations_PC1+RCC5+CYCb}\,a.
          Its extension to composite relations maps, for example,
          $R_1$ from above to $\geqslant$
          and $R_2$ to the universal relation.
        \item
          The operations $\pi$ and $\circ$ are the standard binary converse and composition operations
          from \eqref{eq:converse_binary} and \eqref{eq:binary_composition}.
      \end{itemize}
    \item
      The converse operation $\breve{~}$ is given by Table~\ref{tab:cc_tables_for_PC1}\,a.
      For its extension to composite relations,
      we have, e.g.,
      $R_1\breve{~} = \{\texttt{>},\texttt{=}\}$
      and $R_2\breve{~} = R_2$.
      \begin{table}[ht]
        \begin{center}
          \parbox[t]{.25\textwidth}{%
            \raisebox{-.8\baselineskip}{(a)~~}
            \begin{tabular}[t]{ll}
              \hline
              $r$        & $r\breve{~}$ \\
              \hline
              \texttt{<} & \texttt{>}   \\
              \texttt{=} & \texttt{=}   \\
              \texttt{>} & \texttt{<}   \\
              \hline
            \end{tabular}
          }
          \parbox[t]{.4\textwidth}{%
            \raisebox{-.8\baselineskip}{(b)~~}
            \begin{tabular}[t]{l|lll}
              \hline\rule{0pt}{11pt}%
              r\textbackslash \raisebox{2pt}{s} & ~~\raisebox{2pt}{\texttt{<}}           & ~~\raisebox{2pt}{\texttt{=}}  & ~~\raisebox{2pt}{\texttt{>}}           \\
              \hline\rule{0pt}{10pt}%
              \texttt{<}                        & $\{\texttt{<}\}$                       & $\{\texttt{<}\}$              & $\{\texttt{<},\texttt{=},\texttt{>}\}$ \\[2pt]
              \texttt{=}                        & $\{\texttt{<}\}$                       & $\{\texttt{=}\}$              & $\{\texttt{>}\}$                       \\[2pt]
              \texttt{>}                        & $\{\texttt{<},\texttt{=},\texttt{>}\}$ & $\{\texttt{>}\}$              & $\{\texttt{>}\}$                       \\[2pt]
              \hline
            \end{tabular}
          }
        \end{center}

        \caption{\new{Converse and composition tables for the point calculus \PC1.}}
        \label{tab:cc_tables_for_PC1}
      \end{table}
    \item
      The composition operation $\diamond$ is given by a $3\times3$ table
      where each cell represents $r\diamond s$, see Table~\ref{tab:cc_tables_for_PC1}\,b.
      For its extension to composite relations,
      we have, for example:
      \begin{align*}
        R_1 \diamond R_2
        & = \{\texttt{<},\texttt{=}\} \diamond \{\texttt{<},\texttt{=},\texttt{>}\} \\
        & = \{\texttt{<}\} \diamond \{\texttt{<}\} ~\cup~ \{\texttt{<}\} \diamond \{\texttt{=}\} ~\cup~ \dots ~\cup~ \{\texttt{=}\} \diamond \{\texttt{>}\} \\
        & = \{\texttt{<}\} \cup \{\texttt{<}\} \cup \dots \cup \{\texttt{>}\} \\
        & = R_2
      \end{align*}
      \Endofexa
  \end{itemize}
\end{New}%
\end{example}%
\begin{example}%
\begin{New}%
  \label{exa:RCC5_is_a_binary_calculus}%
  \calc{RCC-5} too is a binary calculus, with the following components.
  \begin{itemize}
    \item
      The set of relation symbols is $\Rel = \{\texttt{EQ},\texttt{DC},\texttt{PO},\texttt{PP},\texttt{PPi}\}$,
      denoting the relations depicted in Figure~\ref{fig:relations_PC1+RCC5+CYCb}\,b.
      The $2^5=32$ composite relations include, for example,
      $R_1 = \{\texttt{DC},\texttt{PO},\texttt{PP},\texttt{PPi}\}$ (``both regions are distinct'') and
      $R_2 = \{\texttt{PP},\texttt{PPi}\}$ (``one region is a proper part of the other'').
    \item
      Similarly to \PC1, there are several possible interpretations,
      a natural choice being $\Int = \{\Univ,\varphi,\pi,\circ\}$
      with the following components.
      \begin{itemize}
        \item 
          The universe \Univ is the set of all regular closed subsets of $\mathbb{R}^2$.
        \item
          The map $\varphi$ maps, for example,
          \texttt{DC} to all pairs of regions that are disconnected or externally connected.
          Figure~\ref{fig:relations_PC1+RCC5+CYCb}\,b
          illustrates $\varphi(r)$ for all relation symbols $r = \texttt{EQ},\texttt{DC},\texttt{PO},\texttt{PP},\texttt{PPi}$.
        \item
          The operations $\pi$ and $\circ$ are the standard binary converse and composition operations
          from \eqref{eq:converse_binary} and \eqref{eq:binary_composition}.
      \end{itemize}
    \item
      The converse operation $\breve{~}$ is given by Table~\ref{tab:cc_tables_for_RCC5}\,a.
      we have, e.g.,
      $R_2\breve{~} = R_2$.
      \begin{table}[ht]
        \begin{small}
        \begin{center}
            \raisebox{-.8\baselineskip}{(a)~~}
            \begin{tabular}[t]{ll}
              \hline
              $r$          & $r\breve{~}$ \\
              \hline
              \texttt{EQ}  & \texttt{EQ}   \\
              \texttt{DC}  & \texttt{DC}   \\
              \texttt{PO}  & \texttt{PO}   \\
              \texttt{PP}  & \texttt{PP}   \\
              \texttt{PPi} & \texttt{PPi}   \\
              \hline
            \end{tabular}
          \par\smallskip\noindent
            \raisebox{-.8\baselineskip}{(b)~~}
            \begin{tabular}[t]{@{}l|l@{~~}l@{~~}l@{~~}l@{~~}l@{}}
              \hline\rule{0pt}{11pt}%
              r\textbackslash \raisebox{2pt}{s} 
              & ~~\raisebox{2pt}{\texttt{EQ}}
              & ~~\raisebox{2pt}{\texttt{DC}}
              & ~~\raisebox{2pt}{\texttt{PO}}
              & ~~\raisebox{2pt}{\texttt{PP}}
              & ~~\raisebox{2pt}{\texttt{PPi}} \\
              \hline\rule{0pt}{10pt}%
              \texttt{EQ}  
              & $\{\texttt{EQ}\}$
              & $\{\texttt{DC}\}$
              & $\{\texttt{PO}\}$
              & $\{\texttt{PO}\}$
              & $\{\texttt{PPi}\}$ \\[2pt]
              \texttt{DC}
              & $\{\texttt{DC}\}$
              & $\{\texttt{EQ},\texttt{DC},\texttt{PO},\texttt{PP},\texttt{PPi}\}$
              & $\{\texttt{DC},\texttt{PO},\texttt{PP}\}$
              & $\{\texttt{DC},\texttt{PO},\texttt{PP}\}$
              & $\{\texttt{DC}\}$ \\[2pt]
              \texttt{PO}
              & $\{\texttt{PO}\}$
              & $\{\texttt{DC},\texttt{PO},\texttt{PPi}\}$
              & $\{\texttt{EQ},\texttt{DC},\texttt{PO},\texttt{PP},\texttt{PPi}\}$
              & $\{\texttt{PO},\texttt{PP}\}$
              & $\{\texttt{DC},\texttt{PO},\texttt{PPi}\}$ \\[2pt]
              \texttt{PP}
              & $\{\texttt{PP}\}$
              & $\{\texttt{DC}\}$
              & $\{\texttt{DC},\texttt{PO},\texttt{PP}\}$
              & $\{\texttt{PP}\}$
              & $\{\texttt{EQ},\texttt{DC},\texttt{PO},\texttt{PP},\texttt{PPi}\}$\\[2pt]
              \texttt{PPi}
              & $\{\texttt{PPi}\}$
              & $\{\texttt{DC},\texttt{PO},\texttt{PPi}\}$
              & $\{\texttt{PO},\texttt{PPi}\}$
              & $\{\texttt{EQ},\texttt{PO},\texttt{PP},\texttt{PPi}\}$
              & $\{\texttt{PPi}\}$ \\[2pt]
              \hline
            \end{tabular}
        \end{center}
        \end{small}

        \caption{\new{Converse and composition tables for the point calculus \calc{RCC-5}.}}
        \label{tab:cc_tables_for_RCC5}
      \end{table}
    \item
      The composition operation $\diamond$ is given by a $5\times5$ table
      where each cell represents $r\diamond s$, see Table~\ref{tab:cc_tables_for_RCC5}\,b.
      For its extension to composite relations,
      we have, for example:
      \begin{align*}
        \{\texttt{PP},\texttt{PPi}\} \diamond \{\texttt{DC}\}
        & = \{\texttt{PP}\} \diamond \{\texttt{DC}\} ~\cup~ \{\texttt{PPi}\} \diamond \{\texttt{DC}\} \\
        & = \{\texttt{DC}\} ~\cup~ \{\texttt{DC},\texttt{PO},\texttt{PPi}\} \\
        & = \{\texttt{DC},\texttt{PO},\texttt{PPi}\}
      \end{align*}
      \Endofexa
  \end{itemize}
\end{New}%
\end{example}%
\begin{example}%
\begin{New}%
  \label{exa:CYCb_is_a_binary_calculus}%
  \CYCb too is a binary calculus, with the following components.
  \begin{itemize}
    \item
      The set of relation symbols is $\Rel = \{\texttt{e},\texttt{o},\texttt{l},\texttt{r}\}$,
      denoting the relations depicted in Figure~\ref{fig:relations_PC1+RCC5+CYCb}\,c.
      The $2^4=16$ composite relations include, for example,
      $R_1 = \{\texttt{e},\texttt{l}\}$ (``the orientation $y$ is to the left of, or equal to, $x$'') and
      $R_2 = \{\texttt{e},\texttt{o}\}$ (``both orientations are equal or opposite to each other'').
    \item
      The standard interpretation for \CYCb is $\Int = \{\Univ,\varphi,\pi,\circ\}$
      with the following components.
      \begin{itemize}
        \item 
          The universe \Univ is the set of all \emph{2D-orientations},
          which can equivalently be viewed as either the set of radii of a given fixed circle $C$,
          or the set of points on the periphery of $C$,
          or the set of directed lines through a given fixed point (the origin of $C$).
        \item
          The map $\varphi$ maps, for example,
          \texttt{l} to all pairs $(x,y)$ of directed lines
          where the angle $\alpha$ from $x$ to $y$, in counterclockwise fashion,
          satisfies $0^\circ < \alpha < 180^\circ$.
          Analogously \texttt{o} is mapped to those pairs where that angle is exactly $180^\circ$.
          Figure~\ref{fig:relations_PC1+RCC5+CYCb}\,c
          illustrates $\varphi(r)$ for all relation symbols $r = \texttt{e},\texttt{o},\texttt{l},\texttt{r}$.
        \item
          The operations $\pi$ and $\circ$ are the standard binary converse and composition operations
          from \eqref{eq:converse_binary} and \eqref{eq:binary_composition}.
      \end{itemize}
    \item
      The converse operation $\breve{~}$ is given by Table~\ref{tab:cc_tables_for_CYCb}\,a.
      For its extension to composite relations,
      we have, e.g.,
      $R_1\breve{~} = \{\texttt{e},\texttt{l}\}$.
      \begin{table}[ht]
        \begin{center}
          \parbox[t]{.25\textwidth}{%
            \raisebox{-.8\baselineskip}{(a)~~}
            \begin{tabular}[t]{ll}
              \hline
              $r$        & $r\breve{~}$ \\
              \hline
              \texttt{e} & \texttt{e}   \\
              \texttt{o} & \texttt{o}   \\
              \texttt{l} & \texttt{r}   \\
              \texttt{r} & \texttt{l}   \\
              \hline
            \end{tabular}
          }
          \parbox[t]{.55\textwidth}{%
            \raisebox{-.8\baselineskip}{(b)~~}
            \begin{tabular}[t]{l|llll}
              \hline\rule{0pt}{11pt}%
              r\textbackslash \raisebox{2pt}{s} 
              & ~~\raisebox{2pt}{\texttt{e}}
              & ~~\raisebox{2pt}{\texttt{o}}
              & ~~\raisebox{2pt}{\texttt{l}}
              & ~~\raisebox{2pt}{\texttt{r}} \\
              \hline\rule{0pt}{10pt}%
              \texttt{e}  
              & $\{\texttt{e}\}$
              & $\{\texttt{o}\}$
              & $\{\texttt{l}\}$
              & $\{\texttt{r}\}$ \\[2pt]
              \texttt{o}
              & $\{\texttt{o}\}$
              & $\{\texttt{e}\}$
              & $\{\texttt{r}\}$
              & $\{\texttt{l}\}$ \\[2pt]
              \texttt{l}
              & $\{\texttt{l}\}$
              & $\{\texttt{r}\}$
              & $\{\texttt{l},\texttt{o},\texttt{r}\}$
              & $\{\texttt{e},\texttt{l},\texttt{r}\}$ \\[2pt]
              \texttt{r}
              & $\{\texttt{r}\}$
              & $\{\texttt{l}\}$
              & $\{\texttt{e},\texttt{l},\texttt{r}\}$
              & $\{\texttt{l},\texttt{o},\texttt{r}\}$ \\[2pt]
              \hline
            \end{tabular}
          }
        \end{center}

        \caption{\new{Converse and composition tables for the point calculus \CYCb.}}
        \label{tab:cc_tables_for_CYCb}
      \end{table}
    \item
      The composition operation $\diamond$ is given by a $4\times4$ table
      where each cell represents $r\diamond s$, see Table~\ref{tab:cc_tables_for_CYCb}\,b.
      For its extension to composite relations,
      we have, for example:
      \begin{align*}
        R_1 \diamond R_1
        & = \{\texttt{e},\texttt{l}\} \diamond \{\texttt{e},\texttt{l}\} \\
        & = \{\texttt{e}\}\diamond\{\texttt{e}\} ~\cup~ \{\texttt{e}\}\diamond\{\texttt{l}\} ~\cup~ \{\texttt{l}\}\diamond\{\texttt{e}\} ~\cup~ \{\texttt{l}\}\diamond\{\texttt{l}\} \\
        & = \{\texttt{e}\} ~\cup~ \{\texttt{l}\} ~\cup~ \{\texttt{l}\} ~\cup~ \{\texttt{e},\texttt{l},\texttt{r}\} \\
        & = \{\texttt{e},\texttt{l},\texttt{r}\}
      \end{align*}
      \Endofexa
  \end{itemize}
\end{New}%
\end{example}%
\paragraph*{Abstract versus weak and strong operations}
We call permutation and composition operations with Properties~\eqref{eq:abstract_converse} and~\eqref{eq:abstract_composition}
\emph{abstract permutation} and \emph{abstract composition}, following Ligozat's naming in the binary case \cite{Lig05}.
\new{For reasons explained further below,}
our notion of a qualitative calculus imposes weaker requirements on the permutation
operation than Ligozat and Renz's notions of a weak (binary) representation \cite{Lig05,LigozatR04}
or the notion of a (binary) constraint algebra \cite{DBLP:journals/ki/NebelS02}.
The following definition specifies those stronger variants, see, e.g., \citeN{LigozatR04}.
\begin{definition}
  \label{def:stronger_versions_of_comp+conv}
  Let $(\Rel,\Int,\breve{~}^1,\dots,\breve{~}^k,\diamond)$ be a qualitative calculus
  based on the interpretation $\Int = (\Univ, \varphi, \cdot^{\pi_1},\dots,\cdot^{\pi_k}, \circ)$.
  \par\smallskip\noindent
      The permutation operation $\breve{~}^i$ is a
      \emph{weak permutation} if, for all $r \in \Rel$:
      \begin{equation}
        r\breve{~}^i = \bigcap\{S \subseteq \Rel \mid \varphi(S) \supseteq \varphi(r)^{\pi_i}\}
        \label{eq:weak_converse}
      \end{equation}
      The permutation operation $\breve{~}^i$ is a
      \emph{strong permutation} if, for all $r \in \Rel$:
      \begin{equation}
        \varphi(r\breve{~}^i) = \varphi(r)^{\pi_i}
        \label{eq:strong_converse}
      \end{equation}
      The composition operation $\diamond$ is a
      \emph{weak composition} if, for all $r_1,\dots,r_m \in \Rel$:
      \begin{equation}
        \diamond(r_1,\dots,r_m) = \bigcap\{S \subseteq \Rel \mid \varphi(S) \supseteq \circ(\varphi(r_1),\dots,\varphi(r_m))\}
        \label{eq:weak_composition}
      \end{equation}
      The composition 
      $\diamond$ is a
      \emph{strong composition} if, for all $r_1,\dots,r_m \in \Rel$: 
      \begin{equation}
        \varphi(\diamond(r_1,\dots,r_m)) = \circ(\varphi(r_1),\dots,\varphi(r_m))
        \label{eq:strong_composition}
      \end{equation}
\end{definition}
In the literature, the equivalent variant
$r\breve{~}^i = \{s \in \Rel \mid \varphi(s) \cap \varphi(r)^{\pi_i} \neq \emptyset\}$ of
Equation \eqref{eq:weak_converse} is sometimes found;
analogously for Equation \eqref{eq:weak_composition}.
\begin{example}%
\begin{New}%
  \label{exa:PC_strong_operations}%
  The converse and permutation operation in \PC1 are both strong
  because \eqref{eq:strong_converse} holds for all three relation symbols
  (e.g., $\varphi(\texttt{<}\breve{~}) = \varphi(\texttt{>}) = \mathord{>} = \mathord{<}\breve{~} = \varphi(\texttt{<})\breve{~}$),
  and the binary version of \eqref{eq:strong_composition}, namely
  \[
    \varphi(r_1 \diamond r_2) = \varphi(r_1) \circ \varphi(r_2),
  \]
  holds for all nine pairs of relation symbols
  (e.g., $\varphi(\texttt{>} \diamond \texttt{>}) = \varphi(\texttt{>}) = \mathord{>} = \mathord{>} \circ \mathord{>} = \varphi(\texttt{>}) \circ \varphi(\texttt{>})$).
  \Endofexa
\end{New}%
\end{example}%
\begin{example}%
\begin{New}%
  \label{exa:PC_N_weak_operations}%
  Consider the variant $\PC1^{\mathbb{N}}$ of \PC1 that is interpreted over the universe $\mathbb{N}$.
  It contains the same base relations with the usual interpretation
  and, obviously, the same converse operation, see Example~\ref{exa:PC_is_a_binary_calculus}.
  However, composition is no longer strong because $\mathord{<} \circ \mathord{<} \subsetneq \mathord{<}$ holds:
  for ``$\subseteq$'' observe that, whenever $x < y < z$ for three points $x,y,z \in \mathbb{N}$,
  it follows that $x < z$; and ``$\nsupseteq$'' holds because there are points $x,z$ with $x < z$ for which there is no $y$
  with $x < y < z$, for example, $x=0$, $z=1$.
  More precisely, the result of the composition $\texttt{<} \diamond \texttt{<}$
  should be the relation $<^1 = \{(x,z) \mid x+1 < z\}$.
  Since $<^1$ is not expressible by a union of base relations,
  we cannot endow this calculus with a strong symbolic composition operation.
  Consequently we have a choice as to the composition result in question.
  Regardless of that choice,
  the composition table will incur a \emph{loss of information}
  because it cannot capture that the pair $(x,z)$ is in $<^1$.

  If we opt for weak composition, then Equation~\eqref{eq:weak_composition} requires us to
  generate the result of $\texttt{<} \diamond \texttt{<}$ from the symbols for exactly those relations
  that overlap with the domain-level composition $\mathord{<} \circ \mathord{<}$.
  From the above it is clear that this is exactly $\texttt{<}$.
  One can now easily check that, for the case of weak composition,
  we get precisely Table~\ref{tab:cc_tables_for_RCC5}\,b.

  On the contrary, if we do not care about composition being weak,
  then abstract composition (Inequality~\eqref{eq:abstract_composition}) requires us to
  generate the result of $\texttt{<} \diamond \texttt{<}$ from the symbols for \emph{at least} those relations
  that overlap with $\mathord{<} \circ \mathord{<}$.
  This means that we can postulate $\texttt{<} \diamond \texttt{<} = \{\texttt{<}\}$ as before
  or, for example, $\texttt{<} \diamond \texttt{<} = \{\texttt{<},\texttt{=},\texttt{>}\}$.

  The difference between weak and abstract composition is that
  abstract composition allows us to make the composition result arbitrarily \emph{general},
  whereas weak composition forces us to take exactly those relations into account
  that contain possible pairs of $(x,z)$.
  Weak composition therefore restricts the loss of information to an unavoidable minimum,
  whereas abstract composition does not provide such a guarantee:
  the more base relations are included in the composition result,
  the more information we lose on how $x$ and $z$ are interrelated.

  In this connection, it becomes clear why we require composition to be \emph{at least} abstract:
  without this requirement, we could omit, for example, $\texttt{<}$ from the above composition result.
  This would result in \emph{adding spurious information} because we would suddenly be able to conclude that
  the constellation $x<y<z$ is impossible, just because $\texttt{<} \diamond \texttt{<} = \emptyset$.
  This insight, in turn, is particularly important for ensuring soundness of
  the most common reasoning algorithm, a-closure, see Section~\ref{sec:reasoning}.
  \Endofexa
\end{New}%
\end{example}%
In terms of composition tables,
abstract composition requires that each cell corresponding to $\diamond(r_1,\dots,r_m)$ contains \emph{at least}
those relation symbols $t$ whose interpretation intersects with $\circ(\varphi(r_1),\dots,\varphi(r_m))$.
Weak composition additionally requires that each cell contains \emph{exactly} those $t$.
\begin{New}
  Strong composition, in contrast, implies a requirement to the underlying \emph{partition scheme:}
  whenever $\varphi(t)$ intersects with $\circ(\varphi(r_1),\dots,\varphi(r_m))$, it has to be \emph{contained} in $\circ(\varphi(r_1),\dots,\varphi(r_m))$.
  Analogously for permutation.%

  These explanations and those in Example~\ref{exa:PC_N_weak_operations} show that abstractness
  as in Properties~\eqref{eq:abstract_converse} and~\eqref{eq:abstract_composition}
  captures \emph{minimal requirements} to the operations in a qualitative calculus:
  it ensures that, whenever the symbolic relations cannot capture the converse or composition of some domain relations exactly,
  the symbolic converse (composition) approximates its domain-level counterpart \emph{from above},
  thus avoiding the introduction of spurious information. \todo{How about "the symbolic converse (composition) approximates its domain-level counterpart by retaining all its information, but potentially introducing redundant information? " (Jae)}
  Weakness (Properties~\eqref{eq:weak_converse} and~\eqref{eq:weak_composition})
  additionally ensures that the loss of information is kept to the unavoidable minimum.
  This last observation is presumably the reason why existing calculi 
  (see Section~\ref{sec:existing_representations}) typically have at least
  weak operations -- we are not aware of any calculus with only abstract operations.

  In Section \ref{sec:reasoning}, we will see that abstract composition
  is a minimal requirement for ensuring soundness of the
  most common reasoning algorithm, a-closure,
  and review the impact
  of the various strengths of the operations on reasoning algorithms.%
\end{New}

The three notions form a hierarchy:%
\begin{fact}
  \label{fact:abstract_vs_weak_vs_strong}
  Every strong permutation (composition) is weak,
  and every weak permutation (composition) is abstract.
  \hfill
  $\lhd$~\emph{\ref{app:abstract_vs_weak_vs_strong}}
\end{fact}
%
It suffices to postulate the properties weakness and strongness with respect to relation symbols only:
they carry over to composite relations as shown in Fact~\ref{fact:weak+strong_conv+comp_general}.
\todo{TS: Fact~\ref{fact:weak+strong_conv+comp_general} can go to appendix, too.}
%
\begin{fact}
  \label{fact:weak+strong_conv+comp_general}
  Given a qualitative calculus $(\Rel,\Int,\breve{~}^1,\dots,\breve{~}^k,\diamond)$
  the following holds.
  \par\smallskip\noindent
  For all composite relations $R \subseteq \Rel$ and $i=1,\dots,k$:
  \begin{equation}
    \varphi(R\breve{~}^i) \supseteq \varphi(R)^{\pi_i}
    \label{eq:abstract_converse_general}
  \end{equation}
  For all composite relations $R_1,\dots,R_m \subseteq \Rel$:
  \begin{equation}
    \varphi(\diamond(R_1,\dots,R_m))  \supseteq \circ(\varphi(R_1),\dots,\varphi(R_m))
    \label{eq:abstract_composition_general}
  \end{equation}
      If $\breve{~}^i$ is a
      weak permutation, then, for all $R \subseteq \Rel$:
      \begin{equation}
        R\breve{~}^i = \bigcap\{S \subseteq \Rel \mid \varphi(S) \supseteq \varphi(R)^{\pi_i}\} 
        \notag
      \end{equation}
      If $\breve{~}^i$ is a
      strong permutation, then, for all $R \subseteq \Rel$:
      \begin{equation}
        \varphi(R\breve{~}^i) = \varphi(R)^{\pi_i}
        \notag
      \end{equation}
      If $\diamond$ is a
      weak composition, then, for all $R_1,\dots,R_m \subseteq \Rel$:
      \begin{equation}
        \diamond(R_1,\dots,R_m) = \bigcap\{S \subseteq \Rel \mid \varphi(S) \supseteq \circ(\varphi(R_1),\dots,\varphi(R_m)\}
        \notag
      \end{equation}
      If $\diamond$ is a
      strong composition, then, for all $R_1,\dots,R_m \subseteq \Rel$:
      \begin{center}
        ~\hspace*{\fill}
        $\varphi(\diamond(R_1,\dots,R_m)) = \circ(\varphi(R_1),\dots,\varphi(R_m))$
        \hspace*{\fill}
        $\lhd$~\text{\emph{\ref{app:weak+strong_conv+comp_general}}}
      \end{center}
\end{fact}
%
\new{Suppose that we want to achieve that the symbolic permutation operations provided by a calculus $C$
capture all permutations at the domain level.
Then $C$ needs to be permutation-complete
in the sense that at least weak permutation operations for all $n!-1$ nontrivial permutations
can be derived uniquely by composing the ones defined.}
%
%

In the binary case, where the converse is the unique nontrivial (and generating) permutation,
every calculus is permutation-complete. However, as noted above,
the converse is not strong for the binary \calc{CDR} and \calc{RCD} calculi (cf.\ Definition~\ref{def:partition_schemes} ff.).
There are also ternary calculi whose permutations are not strong:
e.g., the shortcut, homing, and inversion operations
in the \calc{single-cross} and \calc{double-cross} calculi \cite{Freksa92,FZ92utilization}
are only weak. Since these calculi provide no further permutation operations,
they are not permutation-complete.
However, it is easy to compute the two missing permutations and thus make both calculi permutation-complete.

\par\medskip\noindent
Ligozat and Renz' \citeyear{LigozatR04} basic notion of a binary qualitative calculus is based on a \emph{weak representation}
which requires an identity relation, abstract composition, and the converse being strong\new{, thus excluding, for example, \calc{CDR} and \calc{RCD}}.
A \emph{representation} is a weak representation with a strong composition
and an injective map $\varphi$.
Our basic notion of a
qualitative calculus is more general than a weak representation by not requiring an identity
relation, and by only requiring abstract permutations and composition\new{, thus including \calc{CDR} and \calc{RCD}}.
On the other hand, it is slightly more restrictive
by requiring the map $\varphi$ to be injective -- however, since base relations are JEPD,
the only way for $\varphi$ to violate injectivity is to give multiple names to the
same relation, which is not really intuitive. It is even problematic because it leads to unintended behavior of the notion of weak composition (or permutation):
if there are two relation symbols for every domain relation, then the intersections in Equations \eqref{eq:weak_converse} and \eqref{eq:weak_composition}
will range over disjoint composite relations $S$ and thus become empty.

Recently, \citeN{WHW14} gave a new definition of a qualitative calculus that does not explicitly use a map -- in our case the interpretation \Int\ -- that connects the symbols with their semantics. Instead, they employ the “notion of consistency”~\cite[p.~211]{WHW14} for generating a weak algebra from the Boolean algebra of relation symbols. As with \cite{LigozatR04} their definition of a qualitative calculus is confined to binary relations only.

\subsection{Spatial and Temporal Reasoning}
\label{sec:reasoning}

As in the area of classical constraint satisfaction problems (CSPs), we are given a set of variables and constraints:
a constraint network or a \emph{qualitative CSP}.\footnote{In the CSP domain, ``CSP'' usually refers to a single instance, not the decision or computation problem.}
The task of constraint satisfaction is to decide whether there exists a valuation of all variables that satisfies the constraints. 
In calculi for spatial and temporal reasoning, all variables range over the entities of the specific spatial (or temporal) domain of a qualitative calculus.
The relation symbols defined by the calculus serve to express constraints between the entities.
More formally, we have:

\begin{definition}[$\text{QCSP}$]
  \label{def:QCSP}
  Let $C = (\Rel,\Int,\breve{~}^1,\dots,\breve{~}^k,\diamond)$ be an $n$-ary qualitative calculus
  with $\Int = (\Univ, \varphi, \cdot^{\pi_1},\dots,\cdot^{\pi_k}, \circ)$,
  and let $X$ be a set of variables ranging over $\Univ$.
  An \emph{$n$-ary qualitative constraint in $C$} is a formula
  $R(x_1,\dots,x_n)$ with variables $x_1,\dots,x_n\in X$ and a relation $R \subseteq \Rel$.
  We say that a valuation $\psi : X \to \Univ$ \emph{satisfies} $R(x_1,\dots,x_n)$
  if $\left(\psi(x_1),\dots,\psi(x_n)\right) \in \varphi(R)$ holds.

  A \emph{qualitative constraint satisfaction problem} (QCSP) is the task to decide whether there is a valuation $\psi$ for a set of variables satisfying a set of constraints.
\end{definition}
\begin{example}%
\begin{New}%
  \label{exa:PC_QCSP}%
  In \PC1 we may have the two constraints
  $c_1 = x_1 \mathbin{\texttt{<}} x_2$ and $c_2 = x_2 ~\{\texttt{<},\texttt{=}\}~ x_3$.
  The valuation $\psi : X \to \mathbb{R}$ with $\psi(x_1) = \sqrt{2}$, $\psi(x_2) = 3.14$ and $\psi(x_3) = 42$
  satisfies both constraints. If we set $\psi(x_3) = 3.14$, then both constraints remain satisfied by $\psi$;
  if we set $\psi(x_3) = 2.718$, then $\psi$ no longer satisfies $c_2$.
  \Endofexa
\end{New}%
\end{example}%
For simplicity and without loss of generality,
we assume that \new{every set of constraints contains
exactly one constraint per set of $n$ variables.
Thus, of binary constraints either $r_{x_{1},x_{2}}$ or $r'_{x_{2},x_{1}}$ is assumed to be given -- the other can be derived using converse; multiple constraints regarding variables $x_{1},x_{2}$ can be integrated via intersection.}
In the following,
$r_{x_1,\dots,x_n}$ stands for the unique constraint between the variables $x_1,\dots,x_n$. 

Several techniques originally developed for finite-domain CSPs can be adapted to spatial and temporal QCSPs.
Since deciding CSP instances is already NP-complete for search problems with finite domains, heuristics are important.
One particularly valuable technique is constraint propagation which aims at making implicit constraints explicit in order to identify variable assignments that would violate some constraint. 
By pruning away these variable assignments, a consistent valuation can be searched more efficiently.
A common approach is to enforce $k$-consistency; the following definition is standard in the CSP literature \cite{dechter}.

\begin{definition}
A QCSP with variables $X$ is \emph{$k$-consistent} if, for all subsets $X'\subsetneq X$
of size $k-1$,
we can extend any valuation of $X'$ that satisfies the constraints to a valuation of $X' \cup \{z\}$  also satisfying the constraints, for any additional variable $z \in X\setminus X'$. \label{def:k-consistency}
\end{definition}

QCSPs  are naturally 1-consistent as universes are nonempty and there are no unary constraints.
An $n$-ary QCSP is $n$-consistent if $r\breve{~}^i_{\!\!\!x_1,\dots,x_k} = r_{\pi_i(x_1,\dots,x_k)}$ for all $i$ and $r_{x_1,\dots,x_k} \neq \emptyset$:
domain relations are typically \emph{serial}, that is, for any $r$ and $x_1,\dots,x_{k-1}$, there is some $x_k$ with $r(x_1,\dots,x_k)$.
In the case of binary relations, this means that $2$-consistency is guaranteed in calculi with a strong converse
by $r\breve{~}_{\!\!\!x,y} = r_{y,x}$ and $r_{x,y} \neq \emptyset$,
and seriality of $r$ means that, for every $x$, there is a $y$ with $r(x,y)$.

Already examining $(n+1)$-consistency may provide very useful information.
The following is best explained for binary relations and then generalized to higher arities.
A \mbox{3-consistent} binary QCSP is called {\em path-consistent},
and Definition \ref{def:k-consistency} can 
 be rewritten using binary composition as
\begin{equation}
  \label{eq:3-consistency}
\forall x,y \in X\qquad  r_{x,y} \subseteq \bigcap_{z\in X} r_{x,z} \circ r_{z,y}.
\end{equation}
We can enforce 3-consistency by computing the fixpoint of the refinement operation 
\begin{equation}\label{eq:revise}
  r_{x,y} \gets r_{x,y} \cap \left( r_{x,z} \circ r_{z,y}\right),
\end{equation}
applied to all variables $x,y,z \in X$.
In finite CSPs with variables ranging over finite domains, composition is also finite and the procedure always terminates since the refinement operation is monotone and there can thus only be finitely many steps until reaching the fixpoint.
Such procedures are called path-consistency algorithms and require $\Landau{|X|^3}$ time \cite{dechter}.
%
\begin{example}%
\begin{New}%
  \label{exa:PC_path_consistency}%
  The QCSP in Figure \ref{fig:desc_qual_partial} based on \PC1
  is not path-consistent because $r_{A,C}$ implicitly takes on the universal relation,
  and thus Equation \eqref{eq:3-consistency} is violated for $x=A, y=C, z=B$.
  By contrast, the QCSP in Figure \ref{fig:desc_qual_full} is path-consistent,
  which can be verified by considering each permutation of $A,B,C$ in turn.
  \Endofexa
\end{New}%
\end{example}%
Enforcing path-consistency with QCSPs may not be possible using a symbolic algorithm since Equation (\ref{eq:revise}) may lead to relations not expressible in $2^\Rel$.
This problem occurs when composition in a qualitative calculus is not strong.
It is however straightforward to weaken Equation (\ref{eq:revise}) using weak composition:
\begin{equation}
r_{x,y} \gets r_{x,y} \cap \left( r_{x,z} \diamond r_{z,y}\right)
  \label{eq:revise2}
\end{equation}
The resulting procedure is called enforcing {\em algebraic closure} or {\em a-closure} for short.
The QCSP obtained as a fixpoint of the iteration is called \emph{algebraically closed}.
\begin{example}%
\begin{New}%
  Consider the \PC1 QCSP in Figure~\ref{fig:desc_qual_partial}.
  The missing edge between variables $A$ and $C$ indicates an implicit
  constraint via the universal relation $u=\{\texttt{<},\texttt{=},\texttt{>}\}$.
  Enforcing a-closure as per~\eqref{eq:revise2} updates this constraint with
  $u \cup \texttt{<} \diamond \texttt{<}$, which yields $\texttt{<}$,
  resulting in Figure~\ref{fig:desc_qual_full}.
  Further applications of~\eqref{eq:revise2} do not cause any more changes;
  hence the QCSP in Figure~\ref{fig:desc_qual_full} is algebraically closed.
  \Endofexa
\end{New}%
\end{example}%
If composition in a qualitative calculus is strong, a-closure and path-consistency coincide.
Since there are finitely many relations in a qualitative calculus, a-closure shares all computational properties with the finite CSP case.

A natural generalization from binary to $n$-ary relations can be achieved by considering $(n+1)$-consistency (recall that path-consistency is 3-consistency).
In context of symbolic computation with qualitative calculi we thus need to lift Equations \eqref{eq:3-consistency} and \eqref{eq:revise} to the particular composition operation available. 
For composition as defined by \eqref{eq:n-ary_composition} one obtains
\begin{equation}
  \forall x_1, \ldots, x_n \in X\qquad r_{x_1,\dots,x_n} \subseteq \bigcap_{y\in X} \circ(r_{x_1,\dots,x_{n-1},y},~ r_{x_1,\dots,x_{n-2},y,x_n},~ \dots,~ r_{y,x_2,\dots,x_n}),
  \notag
\end{equation}
and the symbolic refinement operation \eqref{eq:revise2} becomes
\begin{equation}
  r_{x_1,\dots,x_n} \gets r_{x_1,\dots,x_n} \cap \diamond(r_{x_1,\dots,x_{n-1},y},~ r_{x_1,\dots,x_{n-2},y,x_n},~ \dots,~ r_{y,x_2,\dots,x_n}).
  \label{eq:rel-refinement}
\end{equation}

The reason why, in Definition \ref{def:qualitative_calculus}, we require composition to be at least abstract
is that Inclusion \eqref{eq:abstract_composition} guarantees that reasoning via a-closure is sound:
enforcing $k$-consistency or a-closure does not change the solutions of a CSP, as only impossible valuations are locally removed.
If application of a-closure results in the empty relation, then the QCSP is known to be inconsistent.
By contrast, an algebraically closed QCSP may not be consistent though. 
However, for several qualitative calculi (or at least sub-algebras thereof) a-closure and consistency coincide, see also Section \ref{sec:existing_representations}.
%
\begin{example}%
\begin{New}%
  \label{exa:PC_alg_closure}%
  Consider the modification $\PC1'''$ based on the binary abstract partition scheme
  $\PSPCPPP1 = (\{0,1,2\},\{<,=,>\})$, i.e., the domain now has 3 elements.
  Then the QCSP containing 4 nodes and the constraints $\{x_0 < x_1,~ x_1 < x_2,~ x_2 < x_3\}$
  has the algebraic closure $\{x_i < x_j \mid 0 \leqslant i < j \leqslant 3\}$,
  which has no solution in the 3-element domain.
  \Endofexa
\end{New}%
\end{example}%
%
%
Since domain relations are JEPD, deciding QCSPs with arbitrary composite relations can be reduced to deciding QCSPs with only \emph{atomic relations} (i.e., relation symbols) by means of search (cf. \cite{Renz_Nebel_2007_Qualitative}). 
The approach to reason in a full algebra is thus to {\em refine} a composite relation $R \cup S$ to either $R$ or $S$ in 
a
backtracking search fashion, until 
a dedicated decision procedure becomes applicable.
Computationally, reasoning with the complete algebra is typically NP-hard due to the exponential number of possible refinements to atomic relations.
For investigating reasoning algorithms, one is thus interested in the complexity of reasoning with atomic relations. 
If they can be handled in polynomial time, maximal tractable sub-algebras that extend the set of atomic relations are of interest too.
Efficient reasoning algorithms for atomic relations and the existence of large tractable sub-algebras suggest efficiency in handling practical problems.
The search for maximal tractable sub-algebras can be significantly eased by applying the automated methods proposed by \citeN{renz-efficient}.
These exploit algebraic operations to derive tractable composite relations and, complementary, search for embeddings of NP-hard problems.
Using a-closure plus refinement search has been regarded as the prevailing reasoning method. 
Certainly, a-closure provides an efficient cubic time method for constraint propagation, but
Table \ref{tab:calculi_reasoning}  clearly shows that the majority of calculi require further methods as decision procedures.

%
%

\subsection{Tools to Facilitate Qualitative Reasoning}
There are several tools that facilitate one or more of the reasoning tasks.
The most prominent plain-QSTR tools are \system{GQR} \cite{gqr}, a constraint-based reasoning system for checking consistency using a-closure and refinement search, and the \system{\SparQ} reasoning toolbox \cite{wolter-wallgruen:10},\footnote{available at \url{https://github.com/dwolter/sparq}} which addresses various tasks from constraint- and similarity-based reasoning.
Besides general tools, there are implementations addressing specific aspects (e.g., reasoning with \calc{CDR} \cite{Zhang-etal-AIJ:10}) or tailored to specific problems (e.g., \system{Phalanx} for sparse \calc{RCC-8} QCSPs \cite{sioutis-condotta:SETN}).
In the contact area of qualitative and logical reasoning, the DL reasoners \system{Racer} \cite{racer} and \system{PelletSpatial} \cite{pelletSp} offer support for handling a selection of qualitative formalisms.
For logical reasoning about qualitative domain representations, the tools \system{Hets} \cite{MML07}, \system{SPASS} \cite{spass}, and \system{Isabelle} \cite{isabelle} have been applied,
supporting the first-order Common Algebraic Specification Language CASL~\cite{astesiano_casl:_2002} as well as its higher-order variant HasCASL (see \cite{Woelfl06}).

\subsection{Existing Qualitative Spatial and Temporal Calculi}
\label{sec:existing_representations}
\begin{New}
  In the following, we present an overview of existing calculi
  obtained from a systematic literature survey,
  covering publications in the relevant conferences and journals
  in the past 25 years, and following their citation graphs.
  To be included in our overview, a qualitative calculus has to
  be based on a spatial and/or temporal domain,
  fall under our general definition of a qualitative calculus
  (Def.\ \ref{def:qualitative_calculus}: provide symbolic relations, the required symbolic operations,
  and semantics based on an abstract partition scheme),
  and be described in the literature
  either with explicit composition/converse tables,
  or with instructions for computing those.
  These selection critera exclude sets of qualitative relations
  that have been axiomatized in the context of logical theories,
  see Section~\ref{sec:combination_with_classical_logics},
  or qualitative calculi designed for other domains,
  such as ontology alignment \cite{IE15}).

  Tables \ref{tab:calculi_new_1}--\ref{tab:calculi_new_3} list, to the best of our knowledge,
  all calculi satisfying these criteria.%
\end{New}
Table \ref{tab:calculi_new_1} lists the names of families of calculi
and their domains. Tables \ref{tab:calculi_new_2} and \ref{tab:calculi_new_3}
list all variants of these families with original references,
arity and number of their base relations \new{(which is an indicator for the level of granularity offered and for the average branching factor to expect in standard reasoning procedures). 
Additionally we indicate which calculi are implemented in \system{\SparQ} and can be obtained from there.}

\begin{New}%
  Representational aspects of calculi are shown in Figures~\ref{fig:calculi-temp-overview} and \ref{fig:calculi-spat-overview}, grouping calculi by the type of their basic entities and the key aspects captured. 
  For all temporal and selected spatial calculi we iconographically show one exemplary base relation to illustrate the kind of statements it permits.
  For a complete understanding of the respective calculus, the interested reader is referred to the original research papers
  cited in Tables~\ref{tab:calculi_new_2} and \ref{tab:calculi_new_3}.
  We sometimes use a more descriptive relation name than the original work.

  Figure~\ref{fig:calculi-expressivity-overview} shows the known relations
  between the expressivity of existing calculi.
  There are several ways to measure these,
  via the existence of faithful translations not only between base relations over the same domain,
  but also between representations of related domains or between representations concerned with a different domain. 
  For example, the dependency calculus \calc{DepCalc} representing dependency between points is isomorphic to \calc{RCC-5} representing topology of regions. 
  Both calculi feature the same algebraic structure representing partial-order relationships in the domain.

  Since expressivity of qualitative representations solely relies on how relations are defined,
  there exist distinct calculi which exhibit the same expressivity when Boolean combinations of constraints are considered.
  These connections are particularly interesting, not only from the perspective of selecting an appropriate representation, but also in view of computational properties.
  For example, deciding consistency of atomic constraint networks over the point calculus \calc{PC} is polynomial. 
  Using Boolean combinations of \calc{PC} relations one can simulate \calc{Allen} interval relations.
  \citeN{nebel-buerckert-ACM:95} have exploited this relationship to lift a tractable subset to \calc{Allen}.
  In Figure~\ref{fig:calculi-expressivity-overview} we indicate by an arrow $A \to B$ that relations in $A$ can be expressed by Boolean combinations of relations in $B$.
  For clarity we only show direct relations, not their transitive closure.%
\end{New}

\newsavebox{\pcbox}\sbox{\pcbox}{%
\includegraphics{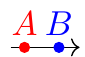}{\footnotesize $A$ before $B$}
}
\newsavebox{\depcalcbox}\sbox{\depcalcbox}{%
\includegraphics{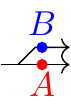}\raisebox{4mm}{\parbox{10mm}{\raggedright {\footnotesize $A$ joint past $B$}}}
}
\newsavebox{\allenbox}\sbox{\allenbox}{%
\includegraphics{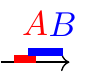}{\footnotesize $A\; \mathrm{overlaps}\; B$}
}
\newsavebox{\sicbox}\sbox{\sicbox}{%
\includegraphics{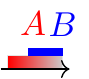}{\footnotesize $A\; \mathrm{older}\; B$}
}
\newsavebox{\diabox}\sbox{\diabox}{%
\includegraphics{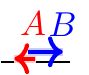}\parbox{2cm}{\raggedright \footnotesize $A$ overlaps from behind $B$}
}
\newsavebox{\eiabox}\sbox{\eiabox}{%
\includegraphics{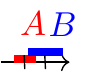}{\footnotesize less part of $A$ overlaps with less part of $B$}
}
\newsavebox{\genintbox}\sbox{\genintbox}{%
\raisebox{-4mm}{%
\includegraphics{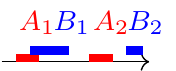}}}
\newsavebox{\indubox}\sbox{\indubox}{%
\raisebox{-2mm}{%
\includegraphics{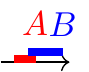}{\footnotesize shorter $A$ overlaps with longer $B$}
}}
\newsavebox{\visbox}\sbox{\visbox}{%
	\includegraphics{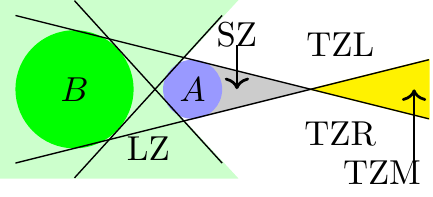}}

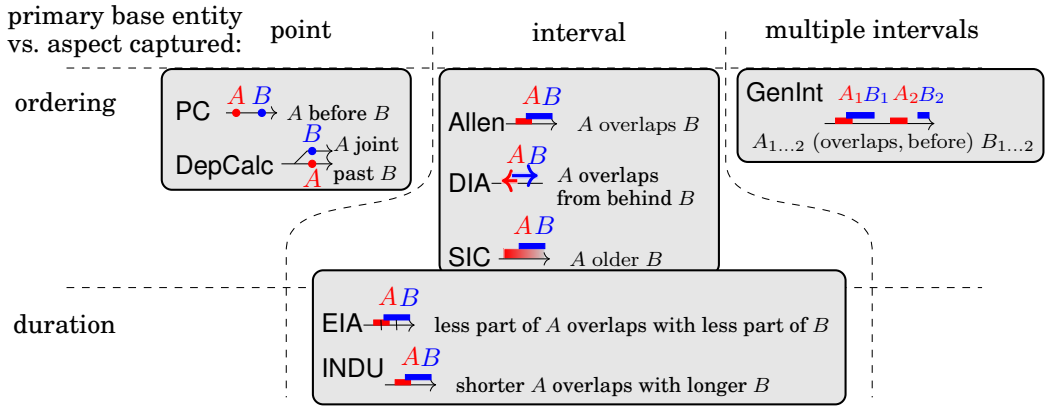
\begin{figure}
\centering
\resizebox{\textwidth}{!}{%
\makebox{%
\begin{tikzpicture}
\tikzstyle{calculi}=[rectangle,rounded corners, thick,draw=black,fill=black!10,align=left,minimum width=34mm]
\tikzstyle{entity}=[]
\tikzstyle{aspect}=[align=flush right,anchor=east,text width=15mm]
\draw (0.8,5) {node [align=left] (entity) {primary base entity\\ vs. aspect captured:}};
\draw (11,5) {node [entity] (mult) {multiple intervals}};
\draw (7,5) {node [entity] (interval) {interval}};
\draw (3.2,5) {node [entity] (point) {point}};
\draw (0.8,4) { node[aspect] (ordering) {ordering} };
\draw (0.8,1.0) { node[aspect] (duration) {duration} };
\draw[dashed] (0,4.5) -- (12.5,4.5);
\draw[dashed] (0,1.5) -- (12.5,1.5);
\draw[dashed,out=270, in=90] (5,5) --(5,3.25) to (3,2) -- (3,0.0);
\draw[dashed,out=270, in=90] (9,5) -- (9,3.25) to (11,2) -- (11,0.0);
\draw (7,4.5) node[calculi,anchor=north] (allen) {\calc{Allen}\usebox{\allenbox}\\  \calc{DIA}\usebox{\diabox}\\ \calc{SIC} \usebox{\sicbox}};
\draw (3,4.5) node[calculi,anchor=north] {\calc{PC} \usebox{\pcbox}\\ \raisebox{1mm}{\calc{DepCalc}} \setlength{\unitlength}{1mm}\begin{picture}(10,8)\put(0,-2){\usebox{\depcalcbox}}\end{picture} };
\draw (11.25,4.5) node[calculi,anchor=north] {\calc{GenInt}\usebox{\genintbox}\\ {\footnotesize $A_{1\ldots 2}\; (\mathrm{overlaps,before})\; B_{1\ldots 2}$}};
\draw (7,1.75) node [calculi,anchor=north] {\calc{EIA}\usebox{\eiabox}\\\calc{INDU}\usebox{\indubox}};
\end{tikzpicture}}}
\caption{\label{fig:calculi-temp-overview}Classification of temporal calculi by representable statements and examples.} 
\end{figure}

\usetikzlibrary{shapes,calc}

\newsavebox{\visibilitybox}\sbox{\visibilitybox}{%
\includegraphics{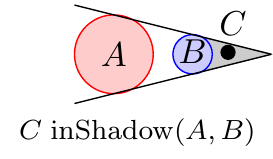}}
\newsavebox{\rccbox}\sbox{\rccbox}{%
\includegraphics{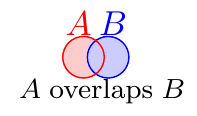}}
\newsavebox{\losbox}\sbox{\losbox}{%
\includegraphics{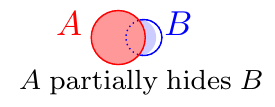}}
\newsavebox{\mcbox}\sbox{\mcbox}{%
\includegraphics{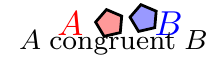}}
\newsavebox{\cdrbox}\sbox{\cdrbox}{%
\includegraphics{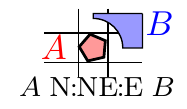}}
\newsavebox{\drabox}\sbox{\drabox}{%
\includegraphics{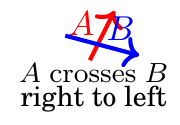}}
\newsavebox{\starbox}\sbox{\starbox}{%
\includegraphics{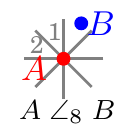}}
\newsavebox{\ffbox}\sbox{\ffbox}{%
\includegraphics{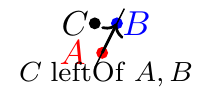}}
\newsavebox{\cibox}\sbox{\cibox}{%
\includegraphics{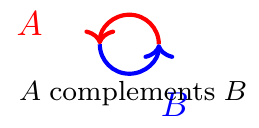}} 
\newsavebox{\eprabox}\sbox{\eprabox}{%
\includegraphics{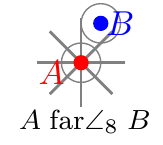}}
\newsavebox{\nineplusbox}\sbox{\nineplusbox}{%
\includegraphics{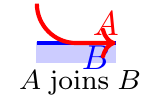}}
\tikzstyle{calculi}=[rectangle,rounded corners, thick,draw=black,fill=black!10,align=left,minimum width=34mm]
\tikzstyle{entity}=[]
\tikzstyle{aspect}=[align=flush right,anchor=east,text width=15mm]
\begin{figure}
\centering
\resizebox{\textwidth}{!}{%
\makebox{%
\begin{tikzpicture}[node distance=1.1cm]
\draw (1,5) {node [align=left] (entity) {primary base entity\\ vs. aspect captured:}};
\draw (11,5) {node [entity] (mult) {region}};
\draw (7,5) {node [entity] (interval) {curve, line}};
\draw (3.5,5) {node [entity] (point) {point}};
\draw (1,1.4) { node[aspect] (cardir) {cardinal direction} };
\draw (1,-0.7) { node[aspect] (direction) {relative direction} };
\draw (1,3.5)  { node[aspect] (topology) {topology} };
\draw (1,-3.0) { node[aspect] (distance) {distance} };
\draw (1,-4.75) { node[aspect] (shape) {shape} };
\draw[dashed] (0,4.6) -- (12.5,4.6);
\draw[dashed] (0,2.5) -- (12.5,2.5);
\draw[dashed] (0,0.5) -- (12.5,0.5);
\draw[dashed] (0,-1.7) -- (12.5,-1.7);
\draw[dashed] (0,-4.0) -- (12.5,-4.0);
\draw[dashed] (5,5) -- (5,-5.3);
\draw[dashed] (9,5) -- (9,-5.3);
\draw (7,4.5) node[calculi, anchor=north] {\small \calc{DRA-con}};
\draw (7,2.65) node[calculi,minimum width=4.5cm, anchor=south] {{\small \calc{CBM}, \calc{CDA}, }, \underline{\calc{9${}^+$-Int}}\raisebox{-4.5mm}{\usebox{\nineplusbox}}};
\draw (7,2.65) node[calculi, anchor=north] {\calc{CI}\raisebox{-5mm}{\usebox{\cibox}}};
\draw (11,4.5) node[calculi,anchor=north] {\underline{\calc{RCC-$n$}}\raisebox{-4mm}{\usebox{\rccbox}}\\  \calc{9-Int}};
\draw (11.1,-0.5) node[calculi] {\calc{VR}\raisebox{-4mm}{\usebox{\visibilitybox}}};
\draw (9.0,-0.5) node[calculi,minimum width=5mm] {\rotatebox{90}{\calc{RfDL-3-12}}};
\draw (6.9,-0.5) node[calculi] {\calc{DRA}\raisebox{-2mm}{\usebox{\drabox}}\\ \calc{ABA$_{23}^8$}, \calc{CYC}};
\draw (3,-1.8) node[calculi] (eopra) {\small \calc{EOPRA},  \calc{QTC}{\footnotesize($\Delta$ dist.)}
};
%
\draw (3,-3.0) node[calculi] (epra) {\parbox{8mm}{\underline{\calc{EPRA}}\newline {\scriptsize (\calc{STAR} + dist.)}}\raisebox{-6mm}{\usebox{\eprabox}}};
%
%
\draw (3,1.5) node[calculi] {\underline{\calc{STAR}} \raisebox{-4mm}{\usebox{\starbox}}\\ {\small \calc{CDC},\calc{PC}}};
\draw (3,-0.5) node[calculi] {\underline{\calc{LR}} \raisebox{-4mm}{\usebox{\ffbox}}\\ {\small \calc{OPRA},\calc{TPCC},\calc{SV},}\\ {\small \calc{1-/2-cross}, \calc{OM-3D}}};
\draw (11,1.4) node[calculi] {\underline{\calc{CDR}} \raisebox{-2mm}{\usebox{\cdrbox}}\\ {\small \calc{RCD},\calc{BA}}};
\draw (9,-3.5) node[calculi,anchor=west] {\calc{LOS} \raisebox{-2mm}{\usebox{\losbox}} \\
\calc{ROC}, \calc{OCC}, \calc{(V)RCC-3D(+)}}; 
\draw (11,-4.9) node[calculi,minimum width=3cm] {\calc{MC-4}\raisebox{-2mm}{\usebox{\mcbox}}};
%
%
\end{tikzpicture}}}
\caption{\label{fig:calculi-spat-overview}Classification of spatial calculi by representable statements with selected example relations.} 
\end{figure}
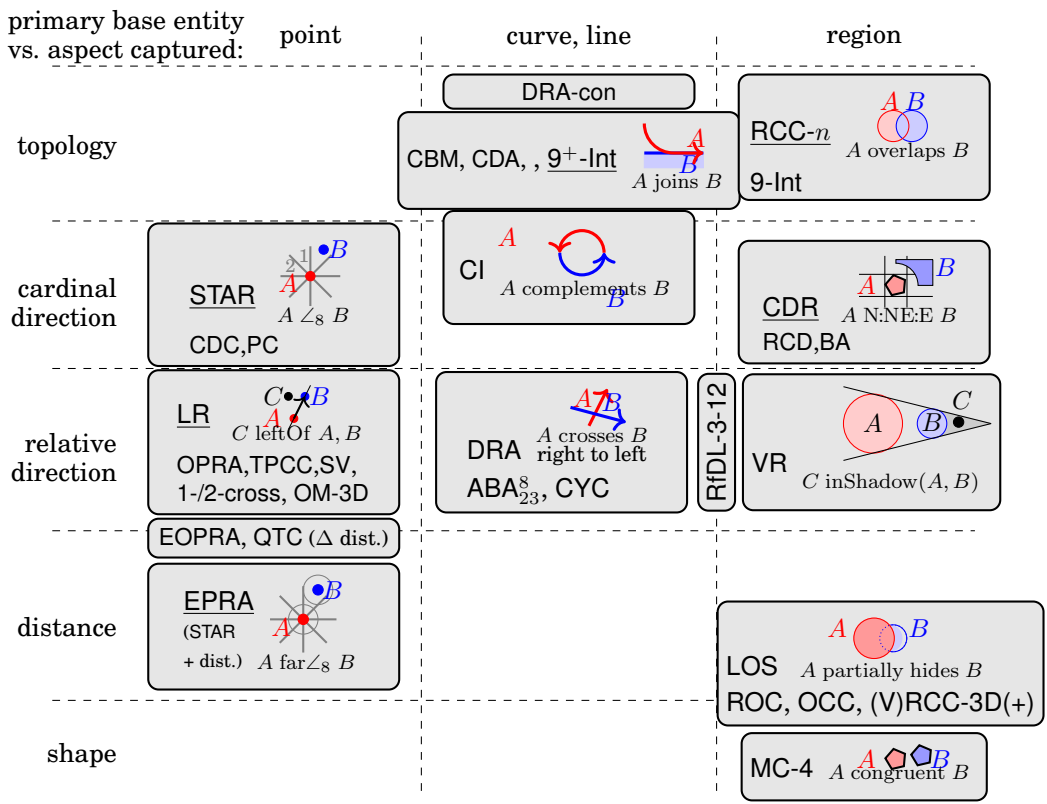

\begin{figure}
\centering
%
%
%

\newcommand{\myCalc}[1]{\calc{\small #1}}
\begin{tikzpicture}[node distance=1.5cm,>=stealth',auto,scale=1.0] 

  \tikzstyle{tempCalc}=[rectangle,thick,draw=black,fill=blue!20,minimum size=6mm]
  \tikzstyle{spatCalc}=[rectangle,thick,draw=black,minimum size=6mm]
  \tikzstyle{spattempCalc}=[rectangle,thick,draw=black,minimum size=6mm,left color=blue!20, right color=white]
  \tikzstyle{models}=[->, thick]
  \tikzstyle{modeled}=[<-, thick]
  \tikzstyle{equiv}=[<->, very thick]

\node[tempCalc] (ia) {\myCalc{PC},\myCalc{IA},\myCalc{SIC},\myCalc{DIA},\myCalc{GenInt}};
\node[tempCalc] (depcalc) [below of=ia] {\myCalc{DepCalc}} edge [modeled] (ia);
\node[tempCalc] (indu) [left of=depcalc] {\myCalc{INDU}} edge [modeled] (ia);
\node[tempCalc] (eia) [left of=ia,node distance=2.5cm] {\myCalc{EIA}} edge [models] (indu);
\begin{pgfonlayer}{background}
	\coordinate (p1) at ($ (ia.north  -| eia.west) + (0,0.5) $);
    \filldraw [line width=4mm,join=round,black!10]
      (p1)  rectangle (depcalc.south  -| ia.east) ;
     \node[anchor=west] at (p1) {temporal and...};
     \node[anchor=west,xshift=4mm] at (p1.north  -| ia.east) {...spatial calculi};
  \end{pgfonlayer}

\node[spattempCalc] (rfdl) at ($(depcalc.south -| eia.west) - (0,1.5)$) [anchor=west] {\myCalc{RfDL-3-12}};
\node[spattempCalc] (qrpc) [right of=rfdl,node distance=2cm] {\myCalc{QRPC}};
\node[spattempCalc] (qtc) [right of=qrpc] {\myCalc{QTC}};
\begin{pgfonlayer}{background}
	\coordinate (pst) at ($ (rfdl.north  -| rfdl.west) + (0,0.5) $);
    \filldraw [line width=4mm,join=round,black!10]
      (pst)  rectangle (qrpc.south  -| qtc.east) ;
     \node[anchor=west] at (pst) {spatio-temporal calculi};
\end{pgfonlayer}

\node[spatCalc] (los) at ($(ia.east) + (1.0,0)$) [anchor=west] {\myCalc{LOS}};
\node[spatCalc] (roc) [right of=los] {\myCalc{ROC}} edge[modeled] (los);
\node[spatCalc] (vr) [right of=roc] {\myCalc{VR}};
\draw ($ (los.north -| los.west) + (-0.3,0.3) $) rectangle ($ (vr.south -| vr.east) - (-0.8,0.4) $) node [xshift=-5mm,yshift=2mm] {shape} ;

\node[spatCalc] (rcc5) [below=of los.west,anchor=west] {\myCalc{RCC-5}} edge [equiv] (depcalc);
\node[spatCalc] (rcc8) [right of=rcc5] {\myCalc{RCC-8}} edge [modeled] (rcc5);
\node[spatCalc] (9int) [right of=rcc8,node distance=1.7cm] {\myCalc{9-Int}};
\draw [equiv,blue] (rcc8) -- (9int) node [midway,yshift=4.5mm,xshift=-14.0mm,color=blue,anchor=west] {\footnotesize for 2D connected regions only} ;
\node[spatCalc] (9intplus) [right of=9int] {\myCalc{9-Int${}^+$}} edge [modeled] (9int);
\draw ($ (rcc5.north -| rcc5.west) + (-0.3,0.3) $) rectangle ($ (9intplus.south -| 9intplus.east) - (-0.3,0.5) $) node [xshift=-10mm,yshift=2mm] {topology} ;
\draw[equiv,blue] (los) -- (rcc8); 

\node[spatCalc] (aba) [below=of rcc5.west,anchor=west] {\myCalc{ABA}${}^8_{23}$};
\node[spatCalc] (rcd) [right of=aba] {\myCalc{RCD}};
\node[spatCalc] (cdr) [right of=rcd] {\myCalc{CDR}};
\node[spatCalc] (epra) [below =of aba.west,anchor=west] {\myCalc{EPRA}};
\node[spatCalc] (star) [right of=epra,] {\myCalc{STAR}};
\node[spatCalc] (cdc) [right of=star,node distance=1.9cm] {\myCalc{CDC}, \myCalc{BA}, \myCalc{CI}} edge [modeled] (rcd);
\draw[models] (cdc) -- (cdr);
\draw[models] (star) -- (epra);
\draw[modeled] (star) -- (cdc);
\draw ($ (aba.north -| aba.west) + (-0.3,0.2) $) rectangle ($ (epra.south -| cdc.east) - (-0.3,0.4) $) node [xshift=-17mm,yshift=2mm] {cardinal directions} ;

\node[spatCalc] (cyc) at ($(rfdl) + (-0.25,-1.5)$) {\myCalc{CYC}};
\node[spatCalc] (lr) [right of=cyc] {\myCalc{LR}};
\node[spatCalc] (sv) [right of=lr] {\myCalc{SV}};
\node[spatCalc] (dra) [below of=cyc] {\myCalc{DRA}${}_f$};
\node[spatCalc] (opra) [right of=dra,anchor=west] {\myCalc{TPCC,OPRA,EOPRA,1-,2-cross}};
\draw[models] (cyc) -- (dra); 
\draw[models] (lr) -- (dra);
\draw ($ (cyc.north -| cyc.west) + (-0.3,0.3) $) -- ($ (cyc.north -| epra.west) + (-0.5,0.3) $) -- ($ (epra.south -| epra.west) + (-0.5,-0.7) $) -- ($(epra.south -| opra.east) + (0.3,-0.7) $) -- ($(opra.south -| opra.east) + (0.3,-0.5) $) -- ($ (opra.south -| cyc.west) + (-0.3,-0.5) $) node [near start,above,xshift=0mm]{relative direction} -- ($ (cyc.north -| cyc.west) + (-0.3,0.3) $)  ;


\draw[models] (qtc) to (opra);
\draw[models,out=160,in=20] (star.north) to (sv);
\draw[out=270,in=90,thick] (aba.south) to ($ (epra) - (1.0,0) $);
\draw[models,out=270,in=45] ($ (epra) - (1.0,0) $) to (dra);
\draw[equiv,out=280,in=135] (ia.east) -- ($ (epra) + (-1.0,1.5) $) to (cdc.north west);
\draw[models,out=330,in=210] (dra) to (opra);
\draw[models] (vr)  -- ($ (vr) + (1.5,0) $) node [spatCalc,dashed,anchor=west]{\myCalc{LR}};
\draw[models] (epra) -- (opra);
\draw[models] (sv) -- (opra);
\end{tikzpicture}


\caption{Expressivity relations between calculi. A directed arrow $A \to B$ says that configurations expressible as CSPs over relations of calculus $A$ can also be expressed by Boolean formulae of constraints over relations from calculus $B$. Calculi in a joint box are of equivalent expressivity.
Proof sketches that do not directly follow from original calculus definition papers are given in Appendix \ref{app:expressivity}.}
\label{fig:calculi-expressivity-overview}
\todo[inline]{As far as I know \calc{ABA823} is a relative direction calculus. The location of \calc{CI} doesn't seem to be correct. (Jae)}
\end{figure}
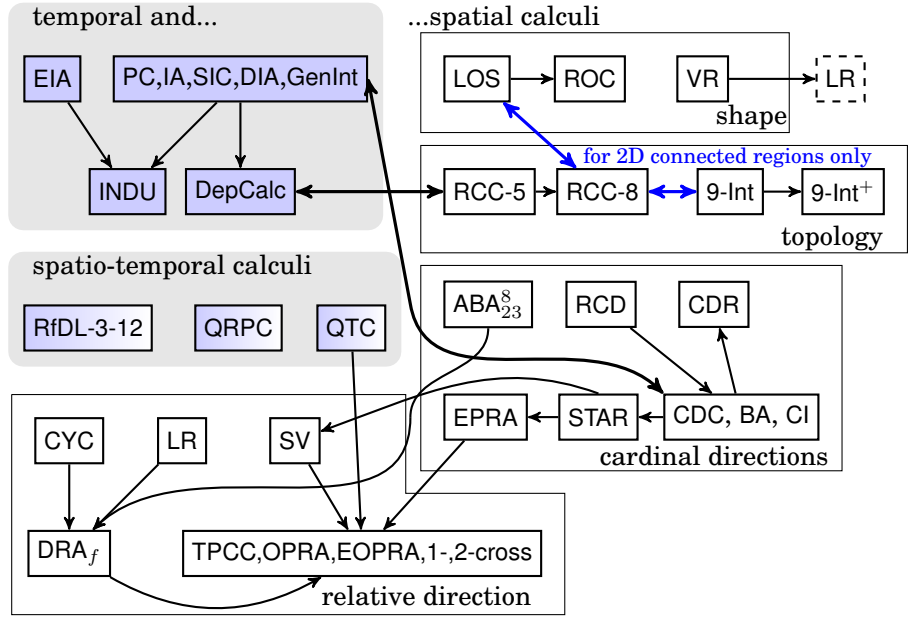


%


Computational aspects of calculi are shown in Table \ref{tab:calculi_reasoning}, as far as they have already been identified.
Some fairly straightforward supplements have been made while compiling this table; their proofs are in Appendix \ref{app:complexity_proofs}.
According to the discussion in the previous section, we give the computational complexity \new{for deciding consistency}  with atomic QCSPs
and the best known \new{complete} decision procedure\new{, which is different from a-closure in those cases where a-closure is incomplete}.
We only indicate the type of algorithm applicable (e.g., linear programming), but not its most efficient realization.
We furthermore list tractable subalgebras that cover at least all atomic relations --
these subalgebras allow for reasoning in the full algebra via combining the named decision procedure with a search for a refinement.
The complexity is given as ``P'' (in polynomial time), ``NPc'' (NP-complete), and ``NPh'' (NP-hard, NP-membership unknown).

\begin{table}
  \begin{small}
    \begin{centering}
      \rowcolors*{1}{lightblue}{}%
      \renewcommand{\tabcolsep}{3pt}
      \begin{tabular}{llll}
      	\hline
      	\rowcolor{medblue}\rule{0pt}{10pt}%
        Abbrev.\ & Name                                                                    & Domain                                 & Aspect                \\[1pt] \hline
      	\rule{0pt}{10pt}%
        \eestrut \calc{1-,2-cross} & \calc{Single/Double Cross Calculus}                                     & points in 2-d                          & relative location     \\
      	\eestrut \calc{9-int}                                & \calc{Nine-Intersection Model}                                          & simple $n$-d regions                   & topology              \\
      	\eestrut \calc{9${}^{\text{(+)}}$-int}               & \calc{9-} and \calc{9${}^\text{+}$-Intersection Calculi}                & \multicolumn{2}{c}{9-int \& bodies, lines, points in 2-d/3-d}  \\
      	\eestrut \calc{\ABA}                                 & Alg.\ of Bipartite Arrangements                                         & 1-d intervals in 2-d                   & rel. loc./orientation \\
      	\nrsc \eestrut \calc{BA}                             & \multicolumn{3}{l}{\calc{Block algebra}\quad  {\footnotesize (aka \calc{Rectangle Algebra} or \calc{Rectangle Calculus})}}        \\
      	                                                     &                                                                         & $n$-d blocks                           & order                 \\
      	\eestrut \calc{CBM}                                  & \calc{Calculus Based Method}                                            & 2-d regions, lines, and points         & topology				 \\
      	\eestrut \calc{CDA}                                  & \calc{Closed Disk Algebra}                                              & 2-d closed disks                       & topology				 \\
      	\eestrut \calc{CDC}                                  & \calc{Cardinal Direction Calculus}                                      & points in 2-d                         & cardinal directions   \\
      	\eestrut \calc{CDR}                                  & \calc{Cardinal Direction Relations}                                     & 2-d regions                            & cardinal directions   \\
      	\eestrut \calc{CI}                                   & \calc{Algebra of Cyclic Intervals}                                      & intvls. on closed curves               & cyclic order          \\
      	\nrsc \eestrut \calc{\CYC}                           & \multicolumn{3}{l}{\calc{Cyclic Ordering}\quad {\footnotesize (\calc{\CYCb} aka \calc{Geometric Orientation})}} \\
                                                             &                                                                         & oriented lines in 2-d                  & relative orientation  \\
      	\eestrut \calc{DepCalc}                              & \calc{Dependency Calculus}                                              & partially ordered points               & partial order         \\
      	\eestrut \calc{DIA}                                  & \calc{Directed Intervals Algebra}                                       & directed 1-d intvls. in 1-d            & order/orientation     \\
      	\eestrut \calc{\DRA}                                 & \calc{Dipole Calculus}                                                  & oriented line segms. in $\mathbb{R}^2$ & rel. loc./orientation \\
      	\eestrut \calc{\DRA-conn}                            & \calc{Dipole connectivity}                                              & connectivity of the above              & connectivity          \\
      	\eestrut \calc{EIA}                                  & \calc{Extended Interval Algebra}                                        & 1-d intervals in 1-d                   & order                 \\
        \eestrut \calc{EOPRA}                                & \calc{Elevated Oriented Point Rel.\ Alg.}                               &           \multicolumn{2}{c}{\calc{OPRA} \& local distance}            \\
      	\eestrut \calc{EPRA}                                 & \calc{Elevated Point Relation Algebra}                                  &           \multicolumn{2}{c}{\calc{CDC} \& local distance}            \\
      	\eestrut \calc{GenInt}                               & \calc{Generalized Intervals}                                            & unions of 1-d intvls.                  & order                 \\
      	\eestrut \calc{IA}                                   & (Allen's) \calc{Interval Algebra}                                       & 1-d intervals in 1-d                   & order                 \\
      	\eestrut \calc{\INDU}                                & \calc{Intvl. and Duration Network}                                      &          \multicolumn{2}{c}{\calc{IA} \& relative duration}    \\
      	\eestrut \calc{LOS}                                  & \calc{Lines of Sight}                                                   & 2-d regions in 3-d                     & obscuration           \\
      	\eestrut \calc{\LR}                                  & \calc{\LR Calculus}\quad {\footnotesize (aka \calc{Flip-Flop})}  & points in 2-d                          & relative location     \\
      	\eestrut \calc{MC-4}                                 & \calc{MC-4}                                                             & regions in 2-d                         & congruence            \\
      	\eestrut \calc{OCC}                                  & \calc{Occlusion Calculus}                                               & 2-d regions in 3-d                     & obscuration           \\
      	\eestrut \calc{OM-3D}                                & \calc{3-D Orientation Model}                                            & points in 3-d                          & relative location     \\
      	\eestrut \calc{OPRA}                                 & \calc{Oriented Point Rel. Algebra}                                      & oriented points in 2-d                 & rel. loc./orientation \\
      	\eestrut \calc{PC}                                   & \calc{Point Calculus} {\footnotesize (aka \calc{Point Algebra})} & points in $n$-d                        & total order           \\
      	\nrsc\eestrut \calc{QRPC}                            & \multicolumn{3}{l}{\calc{Qualitative Rectilinear Projection Calculus}}                                                                          \\
      	                                                     &                                                                         & oriented points in 2-d                 & relative motion       \\
      	\eestrut \calc{QTC}                                  & \calc{Qualitative Trajectory Calculus}                                  & moving points in 1-d/2-d               & relative motion       \\
      	\eestrut \calc{RCC}                                  & \calc{Region Connection Calculus}                                       & general regions                        & topology              \\
      	\eestrut \calc{RCD}                                  & \calc{Rectang.\ Card.\ Dir.\ Calculus}                                  & bounding boxes in 2-d                  & cardinal directions   \\
      	\nrsc\eestrut \calc{RfDL-3-12}                       & \multicolumn{3}{l}{\calc{Region-in-the-frame-of-Directed-Line}}                                                                                 \\
      	                                                     &                                                                         & regions \& paths in 2-d                & relative motion       \\
      	\eestrut \calc{ROC}                                  & \calc{Region Occlusion Calculus}                                        & 2-d regions in 3-d                     & obscuration           \\
      	\eestrut \calc{SIC}                                  & \calc{Semi-Interval Calculus}                                           & 1-d intervals in 1-d                   & order                 \\
      	\eestrut \calc{STAR}                                 & \calc{Star Calculi}                                                     & points in 2-d                          & direction             \\
      	\eestrut \calc{SV}                                   & \calc{StarVars}                                                         & oriented points in 2-d                 & relative direction    \\
      	\eestrut \calc{TPCC}                                 & \calc{Ternary Point Config. Calc.}                                      & points in 2-d                          & relative location     \\
      	\eestrut \calc{TPR}                                  & \calc{Ternary Projective Relations}                                     & points or regions in 2-d               & relative location     \\
      	\eestrut \calc{VR}                                   & \calc{Visibility Relations}                                             & convex regions                         & obscuration           \\ \hline
      \end{tabular}
    \end{centering}
    \par
  \end{small}
  \caption{Existing families of spatial and temporal calculi}
  \label{tab:calculi_new_1}
\end{table}

\begin{table}
  \begin{small}
    \begin{centering}
      \renewcommand{\tabcolsep}{5.5pt}
      \rowcolors*{1}{lightblue}{}%
      \begin{tabular}{lllll}
        \hline\rowcolor{medblue}\rule{0pt}{10pt}%
        Variant                              & Specifics                & Reference(s)                                   & Params                   & St                     \\[1pt]
        \hline
        \rule{0pt}{10pt}%
              \eestrut \calc{1-, 2-cross}           &                            & \cite{FZ92utilization}                         & \ter \brels{8,\,15}      & \STCC \STSP            \\
              \eestrut \calc{9-int}                 &                            & \cite{DBLP:conf/ssd/Egenhofer91}               & \bin \brels{  8}         & \STCC \STSP            \\
              \eestrut \calc{9${}^{\text{(+)}}$-int}& 10 variants${}^{\text{a}}$ & \cite{Kur10}                                   & \bin \brels{$\leqslant$\,233} & \STOO{}${}^{\text{c}}$ \\
              \eestrut \calc{\ABA}                  & ${}^{\text{b}}$            & \cite{Got04}                                   & \bin \brels{125}         & \STOO{}${}^{\text{c}}$ \\
              \eestrut \calc{BA${}_n$}              & $n$ dimensions             & [\citeNP{BCF98}; \citeyearNP{BCF99_block_alg}] & \bin \brels{$13^n$}      & \STCC \STSP{}${}^{1,2}$\hspace*{-2pt} \\
              \eestrut \calc{CBM}                   &                            & \cite{CDv93}                                   & \bin \brels{  7}         & \STOO                  \\
              \eestrut \calc{CDA}                   &                            & \cite{ES93}                                    & \bin \brels{  8}         & \STCT                  \\
              \eestrut \calc{CDC}                   &                            & \cite{Frank91,ligozat-JVLC:98}                 & \bin \brels{  9}         & \STCC \STSP            \\
        \nrsc \eestrut \calc{CDR}                   & original version           & \cite{SkiadoK04}                               & \bin \brels{511}         & \STCX{}${}^{\text{c}}$ \\
              \eestrut \calc{cCDR}                  & connected variant          & \cite{SkiadoK05}                               & \bin \brels{289}         & \STCC \STSP            \\
              \eestrut \calc{CI}                    &                            & \cite{BO00}                                    & \bin \brels{ 16}         & \STCC                  \\
        \nrsc \eestrut \calc{\CYCb}                 & binary                     & \cite{DBLP:journals/ai/IsliC00}                & \bin \brels{  4}         & \STCT \STSP            \\
              \eestrut \calc{\CYCt}                 & ternary                    & ibid.                                          & \ter \brels{ 24}         & \STCC \STSP            \\
              \eestrut \calc{DepCalc}               &                            & \cite{ragni-scivos-KI:05}                      & \bin \brels{  5}         & \STCC \STSP            \\
              \eestrut \calc{DIA}                   &                            & \cite{Ren01}                                   & \bin \brels{ 26}         & \STOO{}${}^{\text{c}}$ \\
        \nrsc \eestrut \calc{\DRAc}                 & coarse-grained${}^{\text{b}}$ & \cite{moratz-renz-wolter-ECAI:00}           & \bin \brels{ 24}         & \STCT \STSP            \\
        \nrsc \eestrut \calc{\DRAf}                 & fine-grained               & ibid.                                          & \bin \brels{ 72}         & \STCC \STSP            \\
              \eestrut \calc{\DRAfp}                & \emph{f}$+$parallelism     & \cite{MoratzEtAl2011}                          & \bin \brels{ 80}         & \STCC \STSP            \\
              \eestrut \calc{\DRA-conn}             &                            & \cite{Wallgruen_Wolter_Richter_10_Qualita}     & \bin \brels{  7}         & \STCC \STSP            \\
              \eestrut \calc{EIA}                   &                            & \cite{ZhangRenz_2014_AngryBirds}               & \bin \brels{ 27}         & \STCX{}${}^{\text{c}}$ \\
              \eestrut \calc{EOPRA$_n$}             & granularity $n$            & \cite{MoW12}                                   & \bin \brels{$\Oof(n^3)$} & \STOO{}${}^{\text{c}}$ \\
              \eestrut \calc{EPRA$_n$}              & granularity $n$            & \cite{MoW12}                                   & \bin \brels{$\Oof(n^3)$} & \STOO{}${}^{\text{c}}$ \STSP{}${}^2$ \\
        \nrsc \eestrut \calc{IA}$\times$\calc{EIA}  & coarser variant            & \cite{ZhangRenz_2014_AngryBirds}               & \bin \brels{351}         & \STOO{}${}^{\text{c}}$ \\
              \eestrut \calc{EIA}$\times$\calc{EIA} & finer variant              & ibid.                                          & \bin \brels{729}         & \STOO{}${}^{\text{c}}$ \\
              \eestrut \calc{GenInt}                &                            & \cite{condotta-ECAI:00}                        & \bin \brels{ 13}         & \STCX{}${}^{\text{c}}$ \\
              \eestrut \calc{IA}                    &                            & \cite{allen:83}                                & \bin \brels{ 13}         & \STCC \STSP            \\
              \eestrut \calc{\INDU}                 &                            & \cite{PKS99}                                   & \bin \brels{ 25}         & \STCC \STSP            \\
              \eestrut \calc{LOS-14}                & convex regions             & \cite{Gal94}                                   & \bin \brels{ 14}         & \STOO{}${}^{\text{c}}$ \\
              \eestrut \calc{\LR}                   &                            & \cite{SN05,DBLP:conf/cosit/Ligozat93}          & \ter \brels{  9}         & \STCC \STSP            \\
              \eestrut \calc{MC-4}                  &                            & \cite{Cristani99}                              & \bin \brels{  4}         & \STCC \STSP            \\
              \eestrut \calc{OCC}                   & convex regions             & \cite{Koe02}                                   & \bin \brels{  8}         & \STCT                  \\
              \eestrut \calc{OM-3D}                 &                            & \cite{PET01}                                   & \ter \brels{ 75}         & \STCX{}${}^{\text{c}}$ \\
        \nrsc \eestrut \calc{OPRA$_n$}              & granularity $n$            & [\citeNP{Moratz06_ECAI}; Mossakowski \& M.\ \citeyearNP{MossakowskiMoratz2011}] & \bin \brels{$\Oof(n^2)$} & \STCC \STSP         \\
              \eestrut \calc{OPRA$_n^*$}            & plus alignment             & \cite{DyL10}                                   & \bin \brels{$\Oof(n^2)$} & \STCC                  \\
        \hline
      \end{tabular}
    \end{centering}
    \par\smallskip
    \rowcolors*{1}{}{}%
    \begin{tabular}{l@{~}l}
    	Legend                          &  \\[1pt] \hline
    	\rule{0pt}{10pt}%
    	Params                          & Arity -- (b)inary, (t)ernary -- and number of relation symbols                                        \\
    	St                              & Status of availability: \STOO base relations, \STCT composition table, \STCX complexity results       \\
    	                                & \STCC table \emph{and} complexity, \STSP \system{SparQ} implementation, \url{https://github.com/dwolter/sparq} \\
    	${}^{\text{a}}$                 & 2 variants over 5 domains each                                                                        \\
    	${}^{\text{b}}$                 & Not based on abstract partition scheme (violates JEPD over $\Univ\times\Univ$)                        \\
    	${}^{\text{c}}$                 & Original work describes how to compute the composition table                                          \\
    	${}^1$                          & For $n=1$                                                                                             \\
    	${}^2$                          & For $n=2$                                                                                             \\
    	${}^4$                          & For $n=4$, regular version only                                                                       \\ \hline
    \end{tabular}
    \par
  \end{small}
  \caption{Overview of existing spatial and temporal calculi, Part 1}
  \label{tab:calculi_new_2}
\end{table}

\begin{table}
  \begin{small}
    \begin{centering}
      \renewcommand{\tabcolsep}{5pt}
      \rowcolors*{1}{}{lightblue}%
      \begin{tabular}{lllll}
        \hline\rowcolor{medblue}\rule{0pt}{10pt}%
        Variant                              & Description              & Reference(s)                                   & Params                   & St                     \\[1pt]
        \hline
        \rule{0pt}{10pt}%
        \nrsc \eestrut \calc{PC${}_n$}              & $n$ dimensions             & \cite{vilain-kautz-aaai:86}                    & \bin \brels{$3^n$}       & \STCC \STSP{}${}^1$    \\
                                                    &                            & \cite{BC02}                                    &                          &                        \\
              \eestrut \calc{QRPC}             &                          & \cite{GAD13}                                   & \bin \brels{ 48}         & \STOO                  \\
        \nrsc \eestrut \calc{QTC-B1$x$}, $x\!\!\:=\!\!\:1,2$ & 1-d variants    & \cite{DBLP:conf/geos/WegheKBM05}               & \bin \brels{9,\,27}      & \STCT \STSP            \\
        \nrsc \eestrut \calc{QTC-B2$x$}, \calc{-C2$x$}     & 2-d variants             & ibid.                                          & \bin \brels{9--305}      & \STCT \STSP            \\
              \eestrut \calc{QTC-N}                 & network variant          & \cite{DBC+11}                                  & \bin \brels{ 17}         & \STOO{}${}^{\text{c}}$ \\
        \nrsc \eestrut \calc{RCC-5}                 & without tangentiality    & \cite{randell-cui-cohn-KR:92}                  & \bin \brels{  5}         & \STCC \STSP            \\
        \nrsc \eestrut \calc{RCC-8}                 & with tangentiality       & ibid.                                          & \bin \brels{  8}         & \STCC \STSP            \\
        \nrsc \eestrut \calc{RCC-15, -23}           & concave regions          & \cite{cohn-GeoI:97}                            & \bin \brels{15,\,23}     & \STOO                  \\
        \nrsc \eestrut \calc{RCC-62}                & ''                       & \cite{OFL07}                                   & \bin \brels{ 62}         & \STOO                  \\
        \nrsc \eestrut \calc{RCC*-7, -9}            & $+$ lower-dim. features  & \cite{CC14}                                    & \bin \brels{7,\,9}       & \STCT                  \\
              \eestrut \calc{(V)RCC-3D(+)}          & with occlusion           & \cite{SL14}                                    & \bin \brels{13--37}      & \STOO{}${}^{\text{c}}$ \\
              \eestrut \calc{RCD}                   &                          & \cite{NavarreteEtAl13}                         & \bin \brels{ 36}         & \STCC \STSP            \\
              \eestrut \calc{RfDL-3-12}             &                          & \cite{KS08}                                    & \bin \brels{1772}        & \STOO                  \\
              \eestrut \calc{ROC-20}                &                          & \cite{Randell:2001ww}                          & \bin \brels{ 20}         & \STOO                  \\
              \eestrut \calc{SIC}                   &                          & \cite{Freksa92b}                               & \bin \brels{ 13}         & \STOO{}${}^{\text{c}}$ \\
        \nrsc \eestrut \calc{STAR$_n$}              & granularity $n$          & \cite{renz-mitra-PRICAI:04}                    & \bin \brels{$\Oof(n)$}   & \STCX{}${}^{\text{c}}$ \\
              \eestrut \calc{STAR$_n^r$}            & revised variants         & ibid.                                          & \bin \brels{$\Oof(n)$}   & \STCC \STSP{}${}^4$    \\
              \eestrut \calc{SV$_n$}                & granularity $n$          & \cite{lee-renz-wolter-IJCAI:13}                & \bin \brels{$\Oof(n)$}   & \STCX                  \\
              \eestrut \calc{TPCC}                  &                          & \cite{MoratzR08}                               & \ter \brels{ 25}         & \STCC \STSP            \\
        \nrsc \eestrut \calc{TPR-p}                 & for points               & [Clementini et al.\ \citeyearNP{CB06,CSBT10}]  & \ter \brels{  7}         & \STCT                  \\
              \eestrut \calc{TPR-r}                 & for regions              & ibid.                                          & \ter \brels{ 34}         & \STOO{}${}^{\text{c}}$ \\
              \eestrut \calc{VR}                    &                          & \cite{TDFC07}                                  & \ter \brels{  7}         & \STCT                  \\
        \hline
      \end{tabular}
    \end{centering}
    \par
  \end{small}
  \caption{Overview of existing spatial and temporal calculi, Part 2. Legend in Tab.~\ref{tab:calculi_new_2}}%
  \label{tab:calculi_new_3}%
\end{table}

\begin{table}
  \begin{small}
    \begin{adjustbox}{scale=.88}
    \parbox{\textwidth}{%
    \rowcolors*{1}{}{lightblue}%
    \begin{tabular}{llllll}
      \hline\rowcolor{medblue}\rule{0pt}{10pt}%
      Abbrev.\                           & \multicolumn{2}{l}{Complexity${}^1$}             & Decision procedure${}^2$           & Largest known                    & and its                                                  \\
      \rowcolor{medblue}                 & \multicolumn{2}{l}{(atomic QCSP)}                & (atomic QCSP)                      & tractable subalgebra${}^3$       & coverage${}^4$                                           \\[1pt]
      \hline\rule{0pt}{10pt}%
            \estrut \calc{1,2-cross}     & NPh            & [WL10]                          & PS                             &    --                            &    --                                                    \\[1pt]
      \nrsc \estrut \calc{9-int}         & NPc            & [SSD03]                         & recognizing     & --                                &                                                 -- \\
            \estrut                      &                &                                 &  string graphs [SSD03]                              &                                  &                                                          \\[1pt]
      \nrsc \estrut \calc{BA${}_n$}      & $\Landau{n^3}$ & [BCC02]                         & AC                                 & Strongly preconvex               &                                                          \\
                                         &                &                                 &                                    & relations [BCF99]                &                                                          \\[1pt]
            \estrut \calc{CDC}                  & $\Landau{n^3}$ & [Lig98]                         & AC                                 &  pre-convex relations            &  $\geq$ 25\%                                             \\[1pt]
      \nrsc \estrut \calc{CDR}                  & $\Landau{n^3}$ & [LZLY10]                        & dedicated [LZLY10]                 &                                  &                                                          \\
            \estrut \calc{cCDR}                 & NPc            & [LL11]                          & dedicated [LZLY10]                 & --                               & --                                                       \\[1pt]
            \estrut \calc{CI}                   & $\Landau{n^3}$ & [BO00]                          & AC                                 & nice relations                   & 0.75\textperthousand                                     \\
            \estrut \calc{\CYCt}                & $\Landau{n^4}$ & [IC00]                          & strong 4-consistency               & $\mathcal{CT}_t$                 & 0.01\textperthousand                                     \\
            \estrut \calc{DepCalc}              & $\Landau{n^3}$ & [RS05]                          & AC                                 &  $\tau_{28}$ [RS05]              & 87.5\% [RS05]                                            \\[1pt]
            \estrut \calc{DIA}                  & $\Landau{n^3}$ & [Ren01]                         & AC                                 & $\mathcal{H}^\pm$ (M) (ORD-Horn) &                                                          \\[1pt]
      \estrut \calc{DRA$_{\textit{c/f/fp}}$}  & NPh            & [WL10]                          & PS                             & --                               &  --                                                      \\[1pt]
            \estrut \calc{\DRA-conn}            & $\Landau{n^3}$ & \factum{dra-conn}               & AC                                 & \calc{\DRA-conn}                        & 100\%                                                    \\[1pt]
            \estrut \calc{EIA}                  & P              & \factum{fact:eia}               & translation to \calc{\INDU}                &                                  &                                                          \\[1pt]
      \nrsc \estrut \calc{GenInt}               & P              & [Con00]                         & AC                                 & strongly pre-convex              & $\ll 1$\textperthousand\ for 3-intvls                    \\
                                         &                &                                 &                                    & general relations                & \factum{GenInt}                                          \\[1pt]
      \nrsc \estrut \calc{IA}            & $\Landau{n^3}$ & [VKvB89]                        & AC                                 & ORD-Horn                         & 10.6\%                                                   \\
                                         &                &                                 &                                    & [NB95,\,KJJ03]                   &                                                          \\[1pt]
      \nrsc \estrut \calc{\INDU}         & P              & [BCL06]                         & translation to                     & strongly pre-convex              & 13.6\%                                                   \\
                                         &                &                                 & Horn-ORD SAT                       & relations                        &                                                          \\[1pt]
            \estrut \calc{\LR}           & NPh            & [WL10]                          & PS                             & --                               &  --                                                      \\[1pt]
            \estrut \calc{MC-4}          & P              &                                 & dedicated [Cri99]                  & M-99                             & 75.0\%                                                   \\[1pt]
            \estrut \calc{OM-3D}         & NPh            & \factum{fact:om3d}              & PS                             & --                               &  --                                                      \\
            \estrut \calc{OPRA$_1^{(*)}$}& NPh            & [WL10]                          & PS                             & --                               &  --                                                      \\[1pt]
            \estrut \calc{PC$_m$}        & $\Landau{n^2}$ & [vB92]                          & dedicated                          & \calc{PC$_m$}                    & 100\% [VK86]                                             \\[1pt]
      \nrsc     \estrut \calc{RCC-5${}^{\text{a}}$} & $\Landau{n^3}$ & [Ren02]                         & AC [JD97]                          & $R_5^{28}$ [JD97]                & 87.5\% [JD97]                                            \\
            \estrut \calc{RCC-8${}^{\text{a}}$} & $\Landau{n^3}$ & [Ren02]                         & AC [Ren02]                         & $\widehat{\mathcal{H}}_8$ [Ren99] & 62.6\% [Ren99]                                          \\[1pt]
            \estrut \calc{RCD}           & $\Landau{n^3}$ & [NMSC13]                        & translat.\ to \calc{IA}; AC        & convex relations                 & $\lll$ 0.01\textperthousand                              \\[1pt]
      \nrsc \estrut \calc{STAR${}_m$}    & P              & [LRW13]                         & LP                                 & convex relations \factum{star-convex}& $m=4:$ $<$1\%                                        \\
      \nrsc \estrut \calc{STAR${}^r_m$} ${}^{\text{b}}$ 
                                         & $\Landau{n^3}$ & [RM04]                          & AC                                 & convex relations                 & $m=3:$ 28\%                                              \\
      \nrsc \estrut \calc{STAR${}^r_m$} ${}^{\text{c}}$ 
                                         & $\Landau{n^4}$ & [RM04]                          & 4-consistency                      & convex relations                 & $m=4:$ 12.5\%                                            \\
                                         &                &                                 &                                    &                                  & $m=8:$ $<$1\%                                            \\[1pt]
            \estrut \calc{SV$_{m}$}      & NPc            & [LRW13]                         & LP with search                     &  --                              &  --                                                      \\[1pt]
            \estrut \calc{TPCC}          & NPh            & [WL10]                          & PS                             &  --                              &  --                                                      \\[1pt]
      \hline
    \end{tabular}
    \par\smallskip
    \rowcolors*{1}{}{}%
    \begin{tabular}{l@{~}l}
      ${}^1$ & Complexity of deciding consistency (atomic relations plus universal relation)      \\
      ${}^2$ & Best known algorithm                                                               \\
      ${}^3$ & Name of largest known tractable subalgebra that includes all base relations (LKTS) \\
      ${}^4$ & Percentage of LKTS compared to the complete algebra                                \\[4pt]
      ${}^{\text{a}}$ & For unconstrained regions; connectedness constraints can increase complexity up to PSpace [KPWZ10]              \\
      ${}^{\text{b}}$ & for $m < 3$                                                               \\
      ${}^{\text{c}}$ & for $m \geqslant 3$
    \end{tabular}
    \par\smallskip
    }
    \end{adjustbox}


    \par
  \end{small}
  \caption{Overview of the known complexity landscape of deciding consistency for existing spatial and temporal calculi. Legend: see Table \ref{tab:calculi_reasoning_legend}}
  \label{tab:calculi_reasoning}
\end{table}
  
\begin{table}
  \begin{small}
    \hrule
    \begin{adjustbox}{scale=.885}
    \parbox{\textwidth}{%
    \rowcolors*{1}{}{}%
    \begin{tabular}{ll}
      \rule{0pt}{9pt}%
      AC              & Algebraic closure                                                             \\
      ACS             & Algebraic closure plus search                                                 \\
      PS          & (Multivariate) polynomial systems solving \cite{basu_algorithms_2006}     \\
      LP              & Reducible to linear programming and thus polynomial                           \\
      NPc; NPh        & NP-complete; NP-hard (NP-membership unknown)                                  \\
      P; PSpace       & In polynomial time; in polynomial space                                       \\
    \end{tabular}
    \par\medskip
    \rowcolors*{1}{}{}%
    \begin{tabular}{llll}
      {}[BCC02]  & \cite{BalbianiCC02}             & [LZLY10] & \cite{Zhang-etal-AIJ:10}       \\
      {}[BCF99]  & \cite{BCF99_block_alg}          & [NB95]   & \cite{nebel-buerckert-ACM:95}  \\
      {}[BCL06]  & \cite{BCL06}                    & [NMSC13] & \cite{NavarreteEtAl13}         \\
      {}[BO00]   & \cite{BO00}                     & [Ren99]  & \cite{renz-IJCAI:99}           \\
      {}[Con00]  & \cite{condotta-ECAI:00}         & [Ren01]  & \cite{Ren01}                   \\
      {}[Cri99]  & \cite{Cristani99}               & [Ren02]  & \cite{renz:02}                 \\
      {}[GPP95]  & \cite{GPP95}                    & [RM04]   & \cite{renz-mitra-PRICAI:04}    \\
      {}[IC00]   & \cite{DBLP:journals/ai/IsliC00} & [RS05]   & \cite{ragni-scivos-KI:05}      \\
      {}[JD97]   & \cite{JD97}                     & [SSD03]  & \cite{Schaefer03NP}            \\ 
      {}[KJJ03]  & \cite{krokhin-etal-ACM:03}      & [vB92]   & \cite{vanBeek-AI:92}           \\
      {}[KPWZ10] & \cite{KPWZ10}                   & [VK86]   & \cite{vilain-kautz-aaai:86}    \\
      {}[Lig98]  & \cite{ligozat-JVLC:98}          & [VKvB89] & \cite{Vilain:1989qv}           \\
      {}[LL11]   & \cite{Liu-Li-AIJ:11}            & [WL10]   & \cite{WolterL:2010:Realization4Direction} \\
      {}[LRW13]  & \cite{lee-renz-wolter-IJCAI:13} \\
    \end{tabular}%
    }
    \end{adjustbox}
    \hrule
    \par
  \end{small}
  \caption{Legend for Table \ref{tab:calculi_reasoning}}
  \label{tab:calculi_reasoning_legend}
\end{table}

\section{Algebraic Properties of Spatial and Temporal Calculi}
\label{sec:relation_algebras}

Algebraic properties have been recognized as a formal tool for measuring
the information preservation properties of a calculus
and for providing the theoretical underpinnings for vital optimizations
to reasoning procedures \cite{DBLP:journals/ai/IsliC00,LigozatR04,Due05,DMSW13}.

To start with information preservation,
it is important to distinguish two sources for a loss of information:
one is qualitative abstraction, which maps the perceived, continuous domain
to a symbolic, discrete representation using $n$-ary domain relations and operations
on them (such as composition and permutation operations).
The loss of information associated with this mapping is mostly intended.
To understand the other,
we recall that a spatial (or temporal) calculus consists of
\emph{symbolic} relations and operations, representing
the domain relations and operations.
While the domain operations are known
to satisfy strong algebraic properties, those do not necessarily carry over
to the symbolic operations -- for example, if the operation $\cdot^\homing$ representing 
homing (Section \ref{sec:requirements}) is only abstract or weak,
then there will be symbolic relations $r$ with $(r^\homing)^\homing \neq r$
although, at the domain level, $(R^\homing)^\homing = R$ holds for any $n$-ary relation $R$,
including the interpretation $\varphi(r)$ of $r$. 
This loss of information indicates an unintended structural misalignment between the domain level
and the symbolic level.
Having its roots in the abstraction step, where the set of domain relations and operations is determined,
the information loss becomes noticeable only with the symbolic representation.

If we want to measure how well the
symbolic operations in a calculus preserve information,
we can compare their algebraic properties
with those of their domain-level counterparts.
If they share all algebraic properties,
this indicates that they maximally preserve information.
In addition, algebraic properties seem to supply a finer-grained measure than the mere distinction between
abstract, weak, and strong operations:
there are 14 axioms for binary relation algebras and variants, 
each containing two inclusions or implications
that may or may not hold independently.

Several algebraic properties can be exploited
to justify and implement optimizations in constraint reasoners.
For example, associativity of the composition operation $\diamond$ for binary symbolic relations ensures that,
if the reasoner encounters a path $ArBsCtD$ of length 3,
then the relationship between $A$ and $D$ can be computed ``from left to right''.
Without associativity, it may be necessary to compute $(r \diamond s) \diamond t$ as well as $r \diamond (s \diamond t)$.

In order to study the algebraic properties of spatial and temporal calculi,
the classical notion of a \emph{relation algebra (RA)} \cite{Mad06}
plays a central role \cite{DBLP:journals/ai/IsliC00,LigozatR04,Due05,Mos07}.
The axioms in the definition of an RA reflect the algebraic properties
of the relevant operations on \emph{binary} domain relations -- 
the operations are union, intersection, complement, converse, and binary compositions;
the properties are commutativity, several variants of associativity and distributivity,
and others. These properties have been postulated for binary calculi
\cite{LigozatR04,Due05}, but it has been shown that not all existing calculi satisfy these strong
properties \cite{Mos07}. It is the main aim of this subsection to study the algebraic
properties of existing binary calculi and derive from the results a taxonomy of calculus algebras.

Unfortunately, it is far from straightforward to extend this study to arity 3 or higher:
while algebraic properties of ternary and $n$-ary calculi have been studied \cite{DBLP:journals/ai/IsliC00,SN05,Condotta2006},
we are aware of only one axiomatization for a ternary RA \cite{DBLP:journals/ai/IsliC00},
based on one particular choice of permutation (homing and shortcut)
and composition (the \emph{binary} variant \eqref{eq:composition_3_2_3}).
However, existing calculi are based on different choices of these operations,
and each choice comes with different algebraic properties at the domain level,
for example:
\begin{New}%
  \begin{itemize}
    \item
      Not all permutations are involutive:
      e.g., in the ternary case, 
      we do \emph{not} have $(R^\shortcut)^\shortcut = R$ for all domain relations $R$,
      but rather $((R^\shortcut)^\shortcut)^\shortcut = R$.
  \item
      Each variant of the composition operation has its own neutral element,
      that is, a relation $E$ such that $R \circ E = E \circ R = R$ for all relations $R$:
      e.g., in the ternary case, $\mathbin{{}_3\circ_2^3}$ (Section~\ref{sec:requirements})
      has $\id^3_{\{2,3\}}$ as the neutral element while $\mathbin{\circ_{\text{FZ}}^3}$ has $\id^3_{\{1,2\}}$.
    \item
      Some variants of the composition operation have stronger properties than others:
      e.g., $\mathbin{{}_3\circ_2^3}$ is associative while $\mathbin{\circ_{\text{FZ}}^3}$ is not.
  \end{itemize}
\end{New}%
Establishing a unifying algebraic framework for $n$-ary spatial and temporal calculi
and determining the algebraic properties of existing calculi
would require a whole new research program. In this survey article,
we will therefore restrict our attention to the binary case for the remainder of this section.

\subsection{The Notion of a Relation Algebra}
\label{sec:RA_def}

The notion of an (abstract) RA is defined in \cite{Mad06}
and makes use of the axioms listed in Table  \ref{tab:relation_algebra_axioms}.
\begin{table}[b]
  \rowcolors{1}{lightblue}{}%
  \centering
  \begin{small}
  \begin{tabular}{lrcll}
    \hline
    \eestrut \RA1       & $r \cup s$                                                    & $=$ & $s \cup r$                           & $\cup$-commutativity                  \\
    \eestrut \RA2       & $r \cup (s \cup t)$                                           & $=$ & $(r \cup s) \cup t$                  & $\cup$-associativity                  \\
    \eestrut \RA3       & $\overline{\bar r \cup \bar s} \cup \overline{\bar r \cup s}$ & $=$ & $r$                                  & Huntington's axiom                    \\
    \eestrut \RA4       & $r \diamond (s \diamond t)$                                   & $=$ & $(r \diamond s) \diamond t$          & $\diamond$-associativity              \\
    \eestrut \RA5       & $(r \cup s) \diamond t$                                       & $=$ & $(r \diamond t) \cup (s \diamond t)$ & $\diamond$-distributivity             \\
    \eestrut \RA6       & $r \diamond \id$                                              & $=$ & $r$                                  & identity law                          \\
    \eestrut \RA7       & $(r\breve{~})\breve{~}$                                       & $=$ & $r$                                  & $\breve{~}$-involution                \\
    \eestrut \RA8       & $(r \cup s)\breve{~}$                                         & $=$ & $r\breve{~} \cup s\breve{~}$         & $\breve{~}$-distributivity            \\
    \eestrut \RA9       & $(r \diamond s)\breve{~}$                                     & $=$ & $s\breve{~} \diamond r\breve{~}$     & $\breve{~}$-involutive distributivity \\
    \eestrut \RA{10}    & $r\breve{~} \diamond \overline{r \diamond s} \cup \bar s$     & $=$ & $\bar s$                             & Tarski/de Morgan axiom                \\
    \hline
    \eestrut \WA        & $((r \cap \id) \diamond 1) \diamond 1$                        & $=$ & $(r \cap \id) \diamond 1$            & weak $\diamond$-associativity         \\
    \eestrut \SA        & $(r \diamond 1) \diamond 1$                                   & $=$ & $r \diamond 1$                       & $\diamond$ semi-associativity         \\
    \eestrut \RA{6l}    & $\id \diamond r$                                              & $=$ & $r$                                  & left-identity law                     \\
    \eestrut \PL        & $(r \diamond s) \cap t\breve{~} = \emptyset$                  & $\Leftrightarrow$ & $(s \diamond t) \cap r\breve{~} = \emptyset$  & Peircean law   \\
    \hline
  \end{tabular}
  \par
  \end{small}
  \par\smallskip
 \caption{%
   Axioms for relation algebras and weaker variants \protect\cite{Mad06}.
 }
  \label{tab:relation_algebra_axioms}
\end{table}
\begin{definition}
  Let $\Rel$ be a set of relation symbols containing $\id$ and $1$ (the symbols for the identity and universal relation),
  and let $\cup$, $\diamond$ be binary and $\bar{~}$, $\breve{~}$ unary operations on $\Rel$.
  The tuple $(\Rel,\cup,\bar{~},1,\diamond,\breve{~},\id)$ is a
  \begin{Itemize}
    \item
      \emph{non-associative relation algebra (NA)} if it satisfies Axioms \RA1--\RA3, \RA5--\RA{10};
    \item
      \emph{semi-associative relation algebra (SA)} if it is an NA and satisfies
      Axiom \SA,
    \item
      \emph{weakly associative relation algebra (WA)} if it is an NA and satisfies
      \WA,
    \item 
      \emph{relation algebra (RA)} if it satisfies \RA1--\RA{10},
  \end{Itemize}
  for all $r,s,t \in \Rel$.
\end{definition}
Clearly, every RA is a WA; every WA is an SA; every SA is an NA. 

In the literature, a different axiomatization is sometimes used, for example in \cite{LigozatR04}.
The most prominent difference is that \RA{10} is replaced by \PL,
``a more intuitive and useful form, known as the Peircean law or De Morgan's Theorem K'' \cite{HH02}.
It is shown in \cite[Section 3.3.2]{HH02} that, given \RA1--\RA3, \RA5, \RA7--\RA9,
the axioms \RA{10} and \PL are equivalent. The implication $\PL \Rightarrow \RA{10}$ does not need \RA5 and \RA8.

All axioms except \PL can be weakened to only one of two inclusions,
which we denote by a superscript ${}^\supseteq$ or ${}^\subseteq$.
For example, $\RA[sup]{7}$ denotes $(r\breve{~})\breve{~} \supseteq r$.
Likewise, we use \PL[right] and \PL[left].
Furthermore, Table \ref{tab:relation_algebra_axioms} contains the redundant axiom \RA{6l}
because it may be satisfied when some of the other axioms are violated. It is straightforward to establish
that \RA6 and \RA{6l} are equivalent given \RA7 and \RA9. \hfill$\lhd$~\ref{app:R6_R6l_equiv}

Thanks to Def.\ \ref{def:qualitative_calculus},
certain axioms are satisfied by every calculus:
\begin{fact}
  \label{fact:minimal_axiom_set_implied_by_calculus_def}
  Every qualitative calculus (Def.\ \ref{def:qualitative_calculus}) satisfies
  \RA1--\RA3, \RA5, \RA[sup]7, \RA8, \WA[sup], \SA[sup]
  for all (atomic and composite) relations.
  This axiom set is maximal: each of the remaining axioms
  in Table \ref{tab:relation_algebra_axioms} is not satisfied by some
  qualitative calculus.
  \hfill$\lhd$~\ref{app:minimal_axiom_set_implied_by_calculus_def}
\end{fact}

\subsection{Discussion of the Axioms}
\label{sec:axiom_discussion}


We will now discuss the relevance of the above axioms for spatial and temporal representation and reasoning.
Due to Fact \ref{fact:minimal_axiom_set_implied_by_calculus_def}, we only need to consider
axioms \RA4, \RA6, \RA7, \RA9, \RA{10} (or \PL) and 
their weakenings \RA{6l}, \SA, \WA.

\myparagraph{\RA4 (and \SA, \WA).}
Axiom \RA4 is helpful for modeling since it allows parentheses in chains of compositions to be omitted. 
For example, consider the following statement in natural language about
the relative length and location of two intervals $A$ and $D$.
  \emph{
  Interval $A$ is before some equally long interval that is contained in some longer interval
  that meets the shorter interval $D$.%
  }
This statement is just a conjunction of relations between $A$, the unnamed intermediary intervals $B,C$,
and $D$. 
Although it intuitively does not matter whether we give priority to the composition of the relations between $A,B$ and $B,C$
or to the composition of the relations between $B,C$ and $C,D$, 
\new{there are calculi such as \calc{\INDU} which do not satisfy Axiom \RA4 -- then
the example statement needs to be interpreted as a Boolean formula consisting of a conjunction over both alternatives.}


\new{We note that violation of \RA4 is independent of composition not being strong,} as shown in Section \ref{sec:properties}. 
Presence of strong composition however implies \RA4 since composition of binary domain relations over \Univ is associative:
\begin{fact}
  Every qualitative calculus where composition is strong satisfies \RA4.
\end{fact}

Furthermore, already a weakening  \RA[sup]4 or \RA[sub]4 is useful for optimizing reasoning algorithms, allowing 
the ``finer'' composition -- say, $r \diamond (s \diamond t)$ in case of \RA[sub]4 -- to be computed when a chain of compositions needs to be evaluated.

\myparagraph{\RA6 and \RA{6l}.}
Presence of an \id relation allows the standard reduction from the correspondence problem to satisfiability:
to test whether a constraint system admits the equality of two variables $x,y$,
one can add an \id-constraint between $x,y$ and test the extended system for satisfiability.


\myparagraph{\RA7 and \RA9.}
These axioms allow for certain optimizations in symbolic reasoning, in particular algebraic closure.
\new{If a relation $r$ satisfies \RA7, then reasoning systems} do not need to store both constraints $A\,r\,B$ and $B\,r'\,A$,
since $r'$ can be reconstructed as $r\breve{~}$ if needed. 
Similarly, \RA9 grants that, when enforcing algebraic closure by using Equation (\ref{eq:revise2})
to refine constraints between variable $A$ and $B$, it is sufficient to compute composition
once and, after applying the converse, reuse it to refine the constraint between $B$ and $A$ too. 
%
%
Current reasoning algorithms and their implementations use the described optimizations;
they produce incorrect results for calculi violating \RA7 or \RA9.

\myparagraph{\RA{10} and \PL.}
These axioms reflect that the relation symbols of a calculus
indeed represent binary domain relations, i.e., pairs of elements of a universe.
This can be explained from two different points of view.
\begin{enumerate}
  \item
    If binary domain relations are considered as sets, \RA{10} is equivalent to 
    $
      r\breve{~} \diamond \overline{r \diamond s} \subseteq \bar s.
    $
    If we further assume the usual set-theoretic interpretation of the composition of two domain relations,
    the above inclusion reads as:
    \emph{%
      For any $X,Y$,
      if $Z\,r\,X$ for some $Z$ and, $Z\,r\,U$ implies not $U\,s\,Y$ for any $U$,
      then not $X\,s\,Y$.}
    This is certainly true because $X$ is one such $U$.
  \item
    Under the same assumptions,
    each side of \PL says (in a different order) that there can be no triangle $X\,r\,Y, Y\,s\,Z, Z\,t\,X$.
    The equality then means that the ``reading direction'' does not matter, see also \cite{Due05}.
    This allows for reducing nondeterminism in the a-closure procedure,
    as well as for efficient refinement and enumeration of consistent scenarios.
\end{enumerate}

\subsection{Prerequisites for Being a Relation Algebra}
The following correspondence between properties of a calculus
and notions of a relation algebra is due to \citeN{LigozatR04}:
  every calculus $C$ based on a partition scheme is an NA.
  If, in addition, the interpretations of the relation symbols are \emph{serial} base relations,
  then $C$ is an SA.
%
%
Furthermore, \RA7 is equivalent to the requirement that the converse operation is strong.
This is captured by the following lemma.

\begin{lemma}
  \label{lem:R7_and_strong_converse}
  Let $C = (\Rel,\Int,\breve{},\diamond)$ be a qualitative calculus.
  Then the following properties are equivalent.
      \begin{Enumerate}
        \item
          $C$ has a strong converse.
        \item
          Axiom \RA{7} is satisfied for all relation symbols $r \in \Rel$.
        \item
          Axiom \RA{7} is satisfied for all composite relations $R \subseteq \Rel$.
      \end{Enumerate}
\end{lemma}

\begin{proof}
      Items (2) and (3) are equivalent due to distributivity of $\breve{~}$ over $\cup$,
      which is introduced with the cases for composite relations in Definition \ref{def:qualitative_calculus}.
      
      For ``(1) $\Rightarrow$ (2)'',
      the following chain of equalities, for any $r \in \Rel$, is due to $C$ having a strong converse:
      $
          \varphi(r\breve{~}\breve{~}) = \varphi(r\breve{~})\breve{~} = \varphi(r)\breve{~}\breve{~} = \varphi(r).
      $
      Since \Rel is based on JEPD relations and $\varphi$ is injective, this implies that $r\breve{~}\breve{~} = r$.
      
      For ``(2) $\Rightarrow$ (1)'',
      we show the contrapositive.
      Assume that $C$ does not have a strong converse.
      Then $\varphi(r\breve{~}) \supsetneq \varphi(r)\breve{~}$, for some $r \in \Rel$;
      hence $\varphi(r\breve{~})\breve{~} \supsetneq \varphi(r)\breve{~}\breve{~}$.
      We can now modify the above chain of equalities replacing the first two equalities
      with inequalities, the first of which is
      due to Requirement \eqref{eq:abstract_converse} in the definition of the converse (Def.\ \ref{def:qualitative_calculus}):
      $
          \varphi(r\breve{~}\breve{~}) \supseteq \varphi(r\breve{~})\breve{~} \supsetneq \varphi(r)\breve{~}\breve{~} = \varphi(r).
      $
      Since $\varphi(r\breve{~}\breve{~}) \neq \varphi(r)$,
      we have that $r\breve{~}\breve{~} \neq r$.
\end{proof}

%
%
%
%
\subsection{Algebraic Properties of Existing Spatial and Temporal Calculi}
\label{sec:properties}

We differentiate the algebraic properties of individual calculi, aiming to identify those which are abstract relation algebras, and identifying relevant weaker algebraic properties.
%
To this end, 
we analyzed the calculi
listed in Tables \ref{tab:calculi_new_2}--\ref{tab:calculi_new_3}.
We restrict our selection to the 31 calculi\footnote{For the parametrized calculi \calc{\DRA}, \calc{OPRA}, \calc{QTC}, we count every variant separately.} that
(a) have binary relations -- because the notion of a relation algebra is best understood for binary relations --
and (b) where digital versions of the operation tables are available.

We have written a CASL specification of the axioms listed in Table \ref{tab:relation_algebra_axioms} along with weakenings thereof.
These have been used with \system{\Hets} to determine congruence of calculus and axioms.
Additionally, \system{\SparQ} and its built-in analysis tools have been employed to double-check results.
Due to Fact \ref{fact:minimal_axiom_set_implied_by_calculus_def},
it suffices to test Axioms \RA4, \RA6, \RA7, \RA9, \RA{10} (or \PL)
and, if necessary, the weakenings \SA, \WA, and \RA{6l}.



The results of our tests are depicted in Figure \ref{fig:algebra_notions}, further details are provided in Appendix \ref{app:descrOfProperties}.
The figure arranges the analyzed calculi as hierarchy, the strongest notion (relation algebra) residing at the top
and the weakest (weakly associative Boolean algebra) at the bottom.
Arrows represent the \emph{is-a} relation; i.e.,
every relation algebra is an ``RA minus id law'' as well as a semi-associative relation algebra
and, via transitivity, a weakly associative Boolean algebra.

With the exceptions of \calc{RCD}, \calc{cCDR} and all \calc{QTC} variants,
all tested calculi are at least semi-associative relation algebras;
most of them are even relation algebras.
Hence, only these calculi enjoy all advantages for representation and reasoning optimizations discussed in Section \ref{sec:axiom_discussion}.
For other groups of calculi, special care in implementations of reasoning procedures need to be taken.
In Section~\ref{subsec:alg} we present a revised algorithm to compute algebraic closure that respects all eventualities.

The three groups of calculi that are SAs but not RAs are the Dipole Calculus variant \calc{\DRAf} (refined \calc{\DRAfp} and coarsened \calc{\DRA-conn} are even RAs!), as well as \calc{\INDU} and \calc{OPRA$_m$} for at least $m=1,\dots,8$.
These calculi do not even satisfy one of the inclusions $\RA[sup]4$ and $\RA[sub]4$, which implies that the reasoning optimizations described in Section \ref{sec:axiom_discussion} for Axiom \RA4 cannot be applied. 
As a side note, 
our observations suggest that the meaning of the letter combination ``RA'' in the abbreviations ``\calc{\DRA}'' and ``\calc{OPRA}'' should stand for ``Reasoning Algebra'', not for ``Relation Algebra''.


In principle, it cannot be completely ruled out that associativity is reported to be violated due to errors in either
the operation tables published or the experimental setup. This even applies to non-violations,
although it is much more likely that errors cause sporadic violations than systematic non-violations.
In the case of \calc{\DRAf}, \calc{\INDU} and \calc{OPRA$_m$}, $m=1,\dots,8$,
the relatively high percentage of violations make implementation errors seem unlikely to be the cause.
However, to obtain certainty that these calculi indeed violate \RA4,
one has to find concrete counterexamples and verify them using the original definition of the respective calculus.
For \calc{\DRAf} and \calc{\INDU}, this has been done in the literature \cite{MoratzEtAl2011,BCL06}.
Interestingly, the violation of associativity has been attributed to the converse or composition not being strong.
We remark, however, that composition cannot be the culprit
because, for example, \calc{\DRAfp} has an associative, but only weak, composition operation. While \calc{\DRAfp} has been proven to be associative
due to strong composition in \cite{MoratzEtAl2011},
for \calc{OPRA$_m$}, it can be shown that \emph{none} of the variants for any $m$ are associative (see \cite{MossakowskiMoratzLuecke}).

The \calc{B}-variants of \calc{QTC} violate only the identity laws \RA6, \RA{6l}.
As observed in \cite{Mos07}, it is possible to add a new \id relation symbol,
modify the interpretation of the remaining relation symbols such that they become JEPD,
and adapt the converse and composition tables accordingly, thus obtaining relation algebras.

The \calc{C}-variants of \calc{QTC} additionally violate \RA4, \RA9, \RA{10}, and \PL.
Consequently, most of the reasoning optimizations described in Section \ref{sec:axiom_discussion}
cannot be applied to the \calc{C}-variants of \calc{QTC}. 
The remarkably few violations of \RA9, \RA{10}, and \PL might be due to errors in the composition table,
but the non-trivial verification is part of future work.

\calc{cCDR} and \calc{RCD} are the only calculi with a weak converse in our tests.
\calc{cCDR} satisfies only \WA in addition to the axioms that are always satisfied by a Boolean algebra with distributivity.
Hence, \calc{cCDR} enjoys none of the advantages for representation and reasoning discussed before. 
Similarly to the \calc{C}-variants of \calc{QTC}, the relatively small number of violations of \PL may be due to errors in the tables published.
\calc{RCD} additionally satisfies \RA4. 
Since both calculi satisfy neither \RA7 nor \RA9, current reasoning algorithms and their implementations
yield incorrect results for them, as seen in Section \ref{sec:axiom_discussion}.

%
\begin{figure}[t]
\centering
  \includegraphics{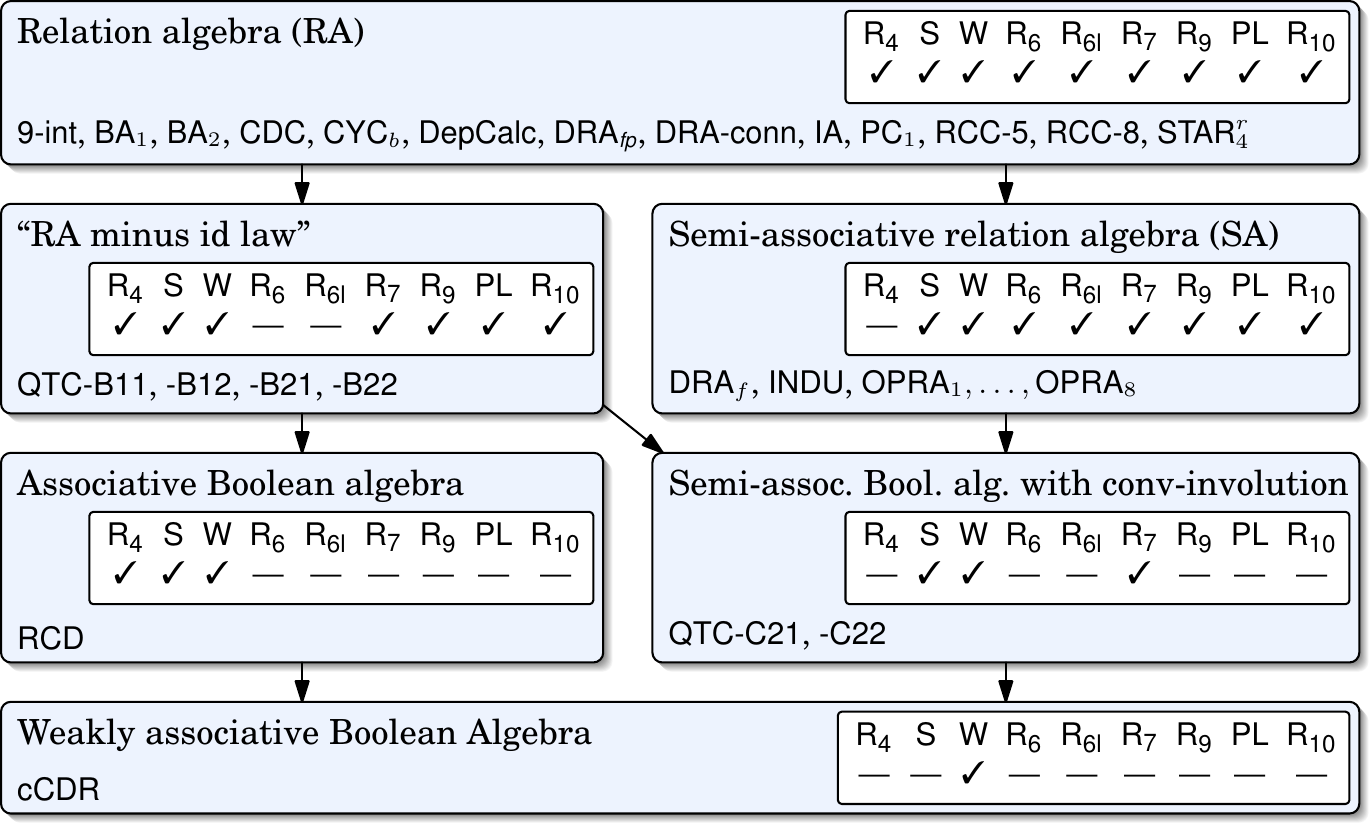}
  \caption{Overview of algebra notions and calculi tested}
  \label{fig:algebra_notions}
\end{figure}

\begin{algorithm}
\DontPrintSemicolon
\SetKwComment{Comment}{\hspace*{\fill}~}{}
\SetKw{assert}{assert}
\SetKwProg{Fn}{Function}{ :}{}
\SetKwFunction{lookup}{LOOKUP}
\SetKwFunction{revise}{REVISE}
\SetKwFunction{enqueue}{ENQUEUENEW}
\SetKwFunction{dequeue}{DEQUEUE}
\SetKwFunction{makequeue}{MAKEQUEUE}
\SetKwFunction{aclosure}{A-CLOSURE}
\SetKwRepeat{DoForEach}{do}{foreach}
\SetKwData{Update}{update}
\Fn(\Comment*[f]{---------- RETRIEVE RELATION FROM CONSTRAINT MATRIX ----------}){\lookup$(C,i,j,s)$
}{
  \eIf{$s \vee (i<j)$}{
    \Return $C_{i,j}$\Comment*[r]{complete matrix stored?}
  }{
    \Return $(C_{j,i})\breve{~}$
  }
}

\vspace*{\baselineskip}
\Fn(\Comment*[f]{---------- REVISE RELATION $r_{i,j}$ ACCORDING TO EQ.~\eqref{eq:revise} ----------}){\revise $(C,i,j,k,s)$
}{
  $u \gets \texttt{false}$ \Comment*[r]{update flag to signal whether relation was updated}
  $r \gets C_{i,j} \cap \lookup(C,i,k,s) \diamond \lookup(C,k,j,s)$\;
  \If{$\text{\RA{9} does not hold} \vee s$}{\label{alg:ra9start}
    $r' \gets \lookup(C,j,i,s) \cap (\lookup(C,j,k,s) \diamond \lookup(C,k,i,s))$\;
    $r \gets r \cap r'\breve{~}$\;
    $r' \gets r' \cap r\breve{~}$\;
    \If{$r'\neq C_{j,i}$}{
      \assert $r'\neq\emptyset$ \Comment*[r]{stop if inconsistency is detected}
      $u \gets \texttt{true}$\;
      $C_{j,i} \gets r'$\;
    }
  } \label{alg:ra9end}
  \If{$r\neq C_{i,j}$}{
    \assert $r\neq\emptyset$ \Comment*[r]{stop if inconsistency is detected}
    $u \gets \texttt{true}$\;
    $C_{i,j} \gets r$\;
  }
  \Return $(C,u)$\;
}

\vspace*{\baselineskip}
\Fn(\Comment*[f]{---------- MAIN ALGORITHM ----------}){\aclosure $(\mathcal{V}, C=\{C_{i,j}| i,j \in \mathcal{V}\})$}{
 \For(\Comment*[f]{Enforce strong 2-consistency}){$i,j\in\mathcal{V}$}{$C_{i,j}\gets C_{i,j} \cap C_{j,i}^\smile$\;}

 \eIf(\Comment*[f]{full $|\mathcal{V}|\times|\mathcal{V}|$ matrix must be stored}){$R_7$ does not hold}{
   $s \gets \mathrm{\bf True}$\;
   $Q \gets \text{ queue with elements } \{(i,j)| i,j \in \mathcal{V}\})$\;
 }(\Comment*[f]{only triangular matrix is stored}){
   $s \gets \mathrm{\bf False}$\;
   $Q \gets \text{ queue with elements } \{(i,j)| i,j \in \mathcal{V}, i<j \})$\;
 }

 \While{$Q$ not empty}{
    dequeue $(i,j)$ from $Q$\;
   \For{$k \in \mathcal{V}, k\neq i, k\neq j$}{
   $(C,u) \gets \revise(C,i,k,j,s)$\;
   \If{$u$}{
   \eIf{$s$}{
      enqueue $(i,k)$ in $Q$ unless already in queue
      }(\Comment*[f]{$R_7 \Rightarrow$ only one of $(i,k)$ and $(k,i)$ is required}){
      enqueue $(\min\{i,k\},\max\{i,k\}))$ in $Q$ unless already in queue
      }
   }
   $(C,u) \gets \revise(C,k,j,i,s)$\;
   \If{$u$}{
   \eIf{$s$}{
      enqueue $(k,j)$ in $Q$ unless already in queue
      }(\Comment*[f]{$R_7 \Rightarrow$ only one of $(i,k)$ and $(k,i)$ is required}){
      enqueue $(\min\{k,j\},\max\{k,j\}))$ in $Q$ unless already in queue
      }
      }

   }
 }
 \Return $C$\;
}

\caption{Universal algebraic closure algorithm {\tt A-CLOSURE} \label{alg:aclosure}}
\end{algorithm}

\subsection{Universal Procedure for Algebraic Closure}\label{subsec:alg}
We noted in Section \ref{sec:axiom_discussion}
that existing descriptions 
and implementations of a-closure (e.g., in \system{GQR} and \system{SparQ})
use optimizations based on certain relation algebra axioms.
Our analysis in Section \ref{sec:properties} reveals that 
there are calculi which violate some of these axioms, e.g., \RA{9};
hence those optimizations lead to incorrect results.
In Algorithm \ref{alg:aclosure} we present a universal algorithm that computes a-closure correctly for all calculi
and uses optimizations only when they are justified.
Its input is a graph $(\mathcal{V},C)$ representing a constraint network,
and $C_{i,j}$ denotes the relation between the $i$-th and $j$-th node ($r_{x,y}$ in Eq.~\eqref{eq:revise}).
Its main control structure
is that of the popular path-consistency algorithm PC-2 \cite{mackworth-AI:77}.
Algorithm \ref{alg:aclosure}
 enforces 2- and 3-consistency
and relies on its input being 1-consistent by implicitly assuming 
all $C_{i,i}$ to cover identity.

Algorithm \ref{alg:aclosure}'s main function is {\tt A-CLOSURE}, 
which employs a queue to store constraint relations that may give rise to
an application of the refinement operation according to Eq.~\eqref{eq:revise}.
The function {\tt REVISE} implements Eq.~\eqref{eq:revise}.
If \RA9 is violated (the converse is not distributive over composition), two steps are necessary
to refine $C_{i,j}$ -- one via computing $C_{j,i}\breve{~}$ independently.
In addition, both {\tt A-CLOSURE} and {\tt REVISE} exploit conformance of a calculus with \RA7 (strong converse)
to halve the space for storing the constraints. 
Flag $s$ indicates whether full storage is required.
If \RA7 is satisfied ($s$ is false), then $C_{i,j}$ can be obtained by computing $C_{j,i}\breve{~}$;
this is done in the auxiliary function {\tt LOOKUP}.




\section{Combination and Integration}
\label{sec:integration+combination}

Although qualitative calculi and constraint-based reasoning are predominant features of
qualitative knowledge representation languages,
they are rarely used by themselves in applications.
For example, many applications involve several aspects of spatial and temporal knowledge simultaneously,
e.g., topology and orientation of spatial objects.
Others require additional forms of symbolic reasoning,
such as logical reasoning. These requirements can best be solved
by combining calculi
or integrating them with other symbolic formalisms.
In this section we review the interaction of qualitative calculi with other components
of knowledge representation languages. 

\subsection{Qualitative Calculi in Constraint-Based Knowledge Representation Languages}
\label{sec:combination_of_calculi}

The simplest case of a qualitative knowledge representation language is a single qualitative calculus.
Sometimes further elements of constraint languages are used in addition,
for example, constants and difference operators as in the case of PIDN \cite{pujari-sattar-IJCAI:99},
or a restricted form of disjunction \cite{LLW13}.

\begin{New}
  If several aspects of spatial and temporal knowledge are to be modeled,
  then combinations of calculi are relevant
  that make the interdependencies between the involved domains accessible.
  \citeN{WW09} identify two general approaches to such combinations and reasoning therein:
  \emph{loose integration} 
  is based on the simple cross product of the base relations
  plus \emph{interdependency constraints} \cite{GR02,westphal-woelfl-ECAI:08},
  and \emph{tight integration} designs a new calculus internalizing the semantic interdependencies \cite{WW09}.
  For example, \calc{\INDU} combines
  \calc{IA} and \calc{PC${}_{1}$} in a tight way, reducing the $\text{13} \times \text{3}$ pairs of relations to the 25 semantically possible.
  A combination of \calc{RCC-8} with \calc{IA} was introduced in \cite{Gerevini02_RCC_IA};
  several combinations of \calc{RCC-8} with direction calculi
  have been analyzed \cite{liu-etal-IJCAI:09,Cohn14_TopCardDirRel}.
  In general, combinations do not inherit algebraic and reasoning properties from their
  constituent calculi (cf.\ Fig.~\ref{fig:calculi-expressivity-overview} for \calc{INDU}, \calc{PC${}_1$}, \calc{IA}).
\end{New}

\citeN{Hernandez94:qualitative} describes the use of topological and orientation relations,
which does not result in a dedicated calculus,
but reveals the effects of constraining one aspect on reasoning in the other.

Alternative ways to solve the combination problem include
formalizing the domain and qualitative relations in an abstract logic
-- which typically are computationally more expensive --
or applying the efficient paradigm of linear programming
to qualitative calculi over real-valued domains \cite{KreutzmannW:2014:AndOrLP}.

\subsection{Qualitative Relations and Classical Logics: Spatial Logics}
\label{sec:combination_with_classical_logics}

There have been several developments to enrich qualitative representation with concepts found in classical logics or to combine the two strands.
Domain representations purely based on qualitative relations can be viewed as quantifier-free formulae with variables ranging over a certain spatial or temporal domain.
QCSP instances can be posed as satisfiability problems of conjunctive constraint formulae in which variables are existentially quantified.
Adopting this logic view for QCSPs leads to the field of spatial logics~\cite{Aiello_Pratt-Hartmann_Benthem_07_Handbook}, which is involved with combinations of qualitative calculi and logics.
Already in the 1930s topological statements as those expressible in \calc{RCC}
were found to constitute a fragment of the modal logic S4 plus the universal modality (S4${}_{\text{u}}$),
comprehensively described by \citeN{Bennett:PhD97}.
\begin{New}
  The cartesian product of S4${}_{\text{u}}$ with linear temporal logic
  captures topological relationships that change over time \cite{Bennett02_PTL_S4u}.
  Qualitative relations and their interrelations can also be described by axiomatic systems,
  this approach has been argued to comprise the composition-table approach
  and to support the construction of composition tables \cite{eschenbach-FOIS:01}.
  Axiomatic systems can be found, e.g., in \cite{eschenbach-kulik-KI:97,Got96,HG11}.
\end{New}%
The field of spatial logics can thus be viewed as a continuum between purely qualitative knowledge representation languages and logics.
Current work is involved with understanding the computational complexity of increasing expressivity of qualitative relations, e.g., by introducing Boolean expressions of spatial variables $\textsf{PO}(A \cap B, C)$ \cite{wolter-zakharyaschev:BRCC},
introducing a temporal modality \cite{SpatialLogic+TemporalLogic},
\new{or even combining spatial and temporal logics \cite{GKK+05}}.

%

\subsection{Qualitative Calculi and Description Logics}
\label{sec:combination_with_DLs}

\begin{New}
Description logics (DLs) are a successful family of knowledge representation languages
tailored to capturing conceptual knowledge in ontologies
and reasoning over it;
see, e.g., \cite{DLHB-2007}.
The most prominent DL-based ontology language is the W3C standard OWL.\footnote{\url{http://www.w3.org/TR/owl2-overview}}
Several approaches to combining DLs 
and qualitative calculi have evolved,
aiming at describing spatial and temporal qualities of application domains.
A principal approach developed by \citeN{LM07} allows
adding qualitative calculi that satisfy certain admissibility
conditions to \ALC, the basic DL, incorporating spatial/temporal reasoning into a
standard DL reasoning procedure.
According to the authors, a practical implementation would be challenging.
%
\citeN{pelletSp} describe \system{PelletSpatial},
an extension of the DL reasoner \system{Pellet} \cite{SPC07} for query answering
over non-spatial (DL) and spatial (\calc{RCC-8}) knowledge.
Unlike the previous approach,
\system{PelletSpatial} separates the two kinds of reasoning.
%
\citeN{BP11} describe SOWL, an OWL ontology capturing 
static, spatial, and temporal information, using a DL axiomatization of
spatial relations from the calculi \calc{CDC} and \calc{RCC-8}.
Temporal and spatial reasoning are separated (a-closure and Pellet, resp.).
\end{New}
\citeN{HCBN12} 
sketch an implementation of logic programming that combines
\calc{9-int}
with OWL ontologies and constructive solid geometry.

%

\subsection{Qualitative Calculi and Situation Calculus}
The situation calculus is a popular framework for reasoning about action and change;
runtime systems such as \system{DTGolog} \cite{Ferrein:Fritz:Lakemeyer:2004:DTGolog}
and \system{ReadyLog} \cite{Ferrein:Lakemeyer:2008:ReadyLog}  are used in robotic applications.
Qualitative relations are relevant to world modeling and underlie high-level behavior specifications \cite{Schiffer:2011kr}.

\citeN{Bhatt:2006:topQSRinSitCalc} aim at general integration of QSTR into reasoning about action and change,
i.e., a general domain-independent theory,
in order to reason about dynamic and causal aspects of spatial change.
With a naive characterization of objects based on their physical properties they particularly investigate key aspects of a topological theory of space on the basis of \calc{\RCCx{8}} 
\cite{BhattLoke:2008_DSSinSitCalc}.

\section{Alternative Approaches}\label{sec:alternatives}

This section presents an overview of reasoning techniques that have also been used to address QSTR reasoning problems, but are not based on QSTR techniques. 
Since spatial reasoning connects to fields in mathematics related to geometry or topology, there are manifold possible connections to make.
In the following we only hint at fields that have already proven to provide impulses to QSTR research.


\subsection{Algebraic Topology}
%
%
Fundamental concepts of algebraic topology resemble expressivity of topological QSTR calculi such as \calc{\RCCx{8}}. 
For example, Euler's well-known polyhedron formula ``vertices - edges + faces = 2'' is a representative of
Euler characteristics that characterize topological invariants of a space or body.
The PLCA framework \cite{takahashi2012} exploits the Euler
characteristics to reason about topological space by invariants.


\subsection{Combinatorial Geometry}
%
%
%
A set of Jordan curves (i.e., sets that are homeomorphic to the interval $[0,1]$ in the plane) induce an \emph{intersection graph}. 
The \emph{string graph problem} poses the question, whether a given graph can be an intersection graph of a set of curves in the plane.
While the problem itself already is of a spatial nature, \citeN{Schaefer03Decidable} reduced reasoning about topological relations in \calc{\RCCx{8}} about planar regions to the string graph problem and later proved the string graph problem to be NP-complete \cite{Schaefer03NP}, directly contributing to QSTR research.

An alternative approach to reasoning with directional relations can be found in {\em oriented matroid} theory, which comprise several equivalent combinatorial structures such as directed graphs, point and vector configurations, pseudoline arrangements, arrangements of hyperplanes  \cite{bjorner_oriented_1999}.
Already \citeN{knuth:1992} points out the importance of oriented matroids for qualitative spatial reasoning.  
In the context of \calc{\LR}\ constraint networks, a connection to the oriented matroid axiomatization of so-called chirotopes lead to complexity results in QSTR \cite{WolterL:2010:Realization4Direction,lee_complexity_2014}. 

%
%

\subsection{Graph Theoretical Approaches}
%
%
\citeN{Worboys2013} describes topological configurations through their representation as labeled trees, called {\em map trees}.
Graph edit operations on map trees can be defined to correspond to spatial change of the topological configuration, providing an efficient approach to reason about spatial change.

Another way of representing qualitative spatial change can be achieved by describing the change on two levels of detail.
\citeN{stell13} describes a scene of regions through a bipartite graph $(U, V, E)$ in which the elements of $U$ represent regions that can be seen as connected at a coarse level of detail, while the elements of $V$ represent regions that are seen as connected when also accounting for finer details.
This way it is possible to describe the splitting, connecting and change of distance of regions, as well as the creation, deletion and change of size of a (part of a) region.

\subsection{Logic Frameworks}
Viewing vectors in a vector space as abstract arrows, \citeN{aiello2007} introduce a so-called arrow logic as a hybrid modal logic that captures  mereotopological relations between sets of vectors. 
Based on the concepts of inversion and composition of arrows, morphological operators such as \emph{dilation}, \emph{erosion} and \emph{difference} can be defined.
A resolution calculus 
allows for automated reasoning about topological relations as well as relative size.

%
%

%
\subsection{Quantitative Methods}

%
%
Linear programming (LP) techniques have been used to decide  constraint problems posed as linear inequalities,
 allowing polyhedral regions, lines, and points to be represented.
LP can mix free-ranging variables with concrete values (\eg, points at known positions) and, beyond consistency checking, determine a model in polynomial time.
By posing QCSP instances as LPs, constraints originating in distinct calculi can easily be mixed.
While some QSTR problems can almost directly be posed as LPs \cite{Jonsson:1998,Ligozat11,lee-renz-wolter-IJCAI:13}, disjunctive LP formulae allow several QSTR calculi to be handled simultaneously \cite{KreutzmannW:2014:AndOrLP}.
In a similar fashion, \citeN{schockaert2011} combine qualitative and quantitative reasoning of relations about different spatial aspects by using genetic optimization. 
Techniques for deciding satisfiability of equations yield advancements on the inherent problem of consistency checking for directional constraints such as those present in the \calc{\LR}\ calculus, as (disjunctions of) linear equations can capture relevant geometric invariances \cite{LuckeMossakowski_2010,vanDeldenMossakowski_2013}.


\section{Conclusion and Future Research Directions}\label{sec:conclusion}    

Qualitative spatial and temporal reasoning explores potentially interesting domain conceptualizations and their computational effects.
As a consequence, QSTR is connected to various research areas in and around artificial intelligence, such as knowledge representation, linguistics and spatial cognition.
Thus QSTR plays the role of a hub for connecting symbolic techniques to real-world applications.
The notion of a {\em qualitative calculus} attests to this role by representing  knowledge about spatial and temporal domains as an abstract algebra that provides the semantics to knowledge representation languages.
Reasoning with qualitative representations occurs in several forms, with deductive forms of inference, such as deciding consistency, being in a central position. 
This is captured in the {\em qualitative constraint satisfaction problem}, which is decidable for all qualitative calculi \new{(in the strict sense of Definition~\ref{def:qualitative_calculus})}, ranging from low-order polynomial time complexity to within PSPACE (cf.~Table~\ref{tab:calculi_reasoning}).
With this survey we present the first comprehensive overview of the known computational properties of all qualitative calculi proposed so far.

\subsection{Beneficiaries of This Survey}

This survey addresses a broad range of researchers and engineers 
from different research communities and application areas.
We expect three groups of beneficiaries.

The first group comprises researchers and engineers who apply QSTR
and build systems for their applications.
Our survey provides them with a comprehensive and concise overview of the formalisms available, 
allowing objective design choices.

The second group consists of researchers contributing to QSTR to whom we
provide revised definitions  that are general enough to address
all formalisms proposed so far. 
The overview of domain conceptualizations studied so far fosters identification of
interesting new conceptualizations to be studied.
Moreover, the summary of algebraic and computational properties of existing formalisms reveals open research questions:
for calculi not listed in Table \ref{tab:calculi_reasoning_legend} reasoning properties have still to be analyzed.

Last, but not least, the third group benefiting from this presentation consists of developers of reasoning tools.
In order to accrete the position of QSTR as hub, sophisticated tools are necessary that disseminate formalisms and algorithms,
linking basic research to application development.
On the one hand, we provide pointers to all formalisms proposed and the decision methods necessary to perform reasoning.
This also reveals commonalities between formalisms, hopefully gearing tools towards becoming universal in the sense that they allow many variants of representations to be handled.
On the other hand -- and related to the discrepancy  between the amount of formalisms proposed and those fully analyzed discussed before --
the most efficient algorithms to decide QCSP instances have often not yet been identified and solid algorithm engineering can likely yield a great leap ahead for QSTR.

\subsection{Open Problem Areas in QSTR}
\label{sec:open_problem_areas}

\paragraph{Combining qualitative abstractions}
Despite the work reported in Section \ref{sec:combination_of_calculi},
\emph{generally applicable} methods for combining existing abstractions for different spatial and temporal aspects are missing -- a potential threat to the applicability of qualitative methods.
It is clearly not feasible to identify all potentially useful combinations individually: there are infinitely many abstractions that give rise to a qualitative calculus.

%
%
%

\paragraph{Integration with other symbolic methods} 
In addition to the above observation that an application may need to handle more than one calculus at the same time, 
 expressivity provided by domain-independent knowledge representation techniques may be important too.
%
There are first contributions (e.g., combining description logic with QSTR),
but these are limited to specific combinations using specific methods.
\begin{New}
A promising approach is the integration of a variety of QSTR
formalisms into a first-order framework
\cite{DBLP:conf/cosit/BhattLS11}---the challenge being the development
of efficient reasoning methods. We expect that this will result in a
combination of first-order methods, constraint-solving methods,
relation-algebraic methods and specialised methods for the existential
theory over the reals, see \cite{vanDeldenMossakowski_2013} for some
first steps.
\end{New}

\paragraph{Integration with quantitative approaches} 
Qualitative approaches link metric data and symbolic reasoning, but consistent interpretation of sensor data considering its inevitable uncertainty is a recurring and challenging task. 
An algorithmic understanding of this problem has to the best of our knowledge not been developed yet.
Conversely, it can also be helpful to link qualitative inference with quantitative or other kinds of constraints.
As \citeN{liu-li-ECAI:12} recently discovered, constraint-based qualitative reasoning with  information partially grounded in data can differ significantly from classic qualitative reasoning and thus calls for further exploration.

\paragraph{Algebras for higher-arity qualitative calculi} 
Abstract algebras provide the foundations for symbolic knowledge manipulation and enable  optimizations 
to reasoning procedures. 
Our study gives an extensive account of algebraic properties of existing binary calculi,
but we have also seen that it is highly non-trivial to extend this study to ternary calculi.
The main problem is a missing notion of relation algebra already for ternary relations
that is general enough to encompass the variety of existing calculi.

\paragraph{Practical reasoning algorithms} 
Few of the various methods required in qualitative reasoning (see Table \ref{tab:calculi_reasoning_legend}) have been studied rigorously in a practical context. 
In the light of continuously growing data bases, identifying best-practice algorithms, evaluating the scaling behavior, and potentially developing heuristic approximations  will be crucial to foster the relevance of QSTR methods.


By completing the picture of computational complexity and identifying practical solutions to reasoning with all individual calculi, either individually or in combination with one another or even other KR techniques, it will be possible to realize truly universal QSTR tools. 
These tools will foster the position of QSTR as a hub, not only conceptually, but implemented in almost all knowledge-based systems.


%
%
\appendixhead{DYLLA}

\begin{acks}
We thank Immo Colonius and Arne Kreutzmann for inspiring discussions during the ``Spatial Reasoning Teatime''. 
Furthermore, we thank Jan-Oliver Wallgr\"un for discussions regarding the taxonomy of QSTR.
\new{We thank the anonymous reviewers for their profound and constructive comments.}
Special thanks go to Erwin R.\ Catesbeiana for the provision of his sitting area.
\end{acks}


\bibliographystyle{ACM-abbrv} 
\bibliography{NEW-QSR-survey-macros,NEW-QSR-survey}
%
\received{Month Year}{Month Year}{Month Year}

\elecappendix

\medskip

\appendix

\section{Additional proofs: Section ``Requirements to Qualitative Representations''} 
\label{app:proofs4Req}

\subsection{Proof of Fact \ref{fact:abstract_vs_weak_vs_strong}}
\label{app:abstract_vs_weak_vs_strong}

\emph{Fact} \ref{fact:abstract_vs_weak_vs_strong}.
Every strong permutation (composition) is weak,
and every weak permutation (composition) is abstract.%

\begin{proof}
  \emph{``Every strong permutation is weak.''}
  We assume that the permutation $\breve{~}$ associated with $\pi$ is strong, i.e., 
  for all $r \in \Rel$,
  \begin{equation}
    \varphi(r\breve{~}) = \varphi(r)^{\pi},
    \label{eq:strong_converse_appx}
  \end{equation}
  and show that $\breve{~}$ is weak, i.e.,
  for all $r \in \Rel$:
  \begin{equation}
    r\breve{~} = \bigcap\{S \subseteq \Rel \mid \varphi(S) \supseteq \varphi(r)^{\pi}\}
    \label{eq:weak_converse_appx}
  \end{equation}
  For ``$\subseteq$'', it suffices to show that,
  for every $S \subseteq \Rel$ with $\varphi(S) \supseteq \varphi(r)^{\pi}$,
  we have $r\breve{~} \subseteq S$.
  This follows from the inclusion ``$\subseteq$'' of \eqref{eq:strong_converse_appx}
  and the injectivity of $\varphi$.

  For ``$\supseteq$'', let $s \in \bigcap\{S \subseteq \Rel \mid \varphi(S) \supseteq \varphi(r)^{\pi}\}$,
  that is, $s \in S$ for every $S \subseteq \Rel$ with $\varphi(S) \supseteq \varphi(r)^{\pi}\}$.
  Since $r\breve{~}$ is such an $S$ due to the inclusion ``$\supseteq$'' of \eqref{eq:strong_converse_appx},
  we have $s \in r\breve{~}$.

  \par\medskip\noindent
  \emph{``Every weak permutation is abstract.''}
  Strictly speaking, the phrasing in Definition \ref{def:stronger_versions_of_comp+conv}
  implies this statement. However, it is easy to show the stronger statement
  that \eqref{eq:weak_converse_appx} implies
  \begin{equation*}
    \varphi(r\breve{~}) \supseteq \varphi(r)^{\pi}\,.
  \end{equation*}
  Indeed, this is justified by the following chain of equalities and inclusions.
  \begin{align*}
    \varphi(r\breve{~}) & = \varphi\left(\bigcap\{S \subseteq \Rel \mid \varphi(S) \supseteq \varphi(r)^{\pi}\}\right) \\
                        & = \bigcap\{\varphi(S) \subseteq \Rel \mid \varphi(S) \supseteq \varphi(r)^{\pi}\}            \\
                        & \supseteq \varphi(r)^{\pi},
  \end{align*}
  where the first equality follows from \eqref{eq:weak_converse_appx},
  the second follows from the extension of $\varphi$ to composite relations as per Definition \ref{def:qualitative_calculus},
  and the final inclusion is an obvious property of sets.

  \par\medskip\noindent
  The respective statements about composition are proven analogously.
\end{proof}

\subsection{Proof of Fact \ref{fact:weak+strong_conv+comp_general}}
\label{app:weak+strong_conv+comp_general}

\emph{Fact} \ref{fact:weak+strong_conv+comp_general}.
Given a qualitative calculus $(\Rel,\Int,\breve{~}^1,\dots,\breve{~}^k,\diamond)$
based on the interpretation $\Int = (\Univ, \varphi, \cdot^{\pi_1},\dots,\cdot^{\pi_k}, \circ)$,
the following hold.
\par\smallskip\noindent
For all relations $R \subseteq \Rel$ and $i=1,\dots,k$:
\begin{equation}
  \varphi(R\breve{~}^i) \supseteq \varphi(R)^{\pi_i}
  \label{eq:abstract_converse_general_app}
\end{equation}
For all relations $R_1,\dots,R_m \subseteq \Rel$:
\begin{equation}
  \varphi(\diamond(R_1,\dots,R_m))  \supseteq \circ(\varphi(R_1),\dots,\varphi(R_m))
  \label{eq:abstract_composition_general_app}
\end{equation}
If $\breve{~}^i$ is a
weak permutation, then, for all $R \subseteq \Rel$:
\begin{equation}
  R\breve{~}^i = \bigcap\{S \subseteq \Rel \mid \varphi(S) \supseteq \varphi(R)^{\pi_i}\} 
  \label{eq:weak_converse_general_app}
\end{equation}
If $\breve{~}^i$ is a
strong permutation, then, for all $R \subseteq \Rel$:
\begin{equation}
  \varphi(R\breve{~}^i) = \varphi(R)^{\pi_i}
  \label{eq:strong_converse_general_app}
\end{equation}
If $\diamond$ is a
weak composition, then, for all $R_1,\dots,R_m \subseteq \Rel$:
\begin{equation}
  \diamond(R_1,\dots,R_m) = \bigcap\{S \subseteq \Rel \mid \varphi(S) \supseteq \circ(\varphi(R_1),\dots,\varphi(R_m)\}
  \label{eq:weak_composition_general_app}
\end{equation}
If $\diamond$ is a
strong composition, then, for all $R_1,\dots,R_m \subseteq \Rel$:
\begin{equation}
  \varphi(\diamond(R_1,\dots,R_m)) = \circ(\varphi(R_1),\dots,\varphi(R_m))
  \label{eq:strong_composition_general_app}
\end{equation}

\begin{proof}
  For \eqref{eq:abstract_converse_general_app}, consider
  \begin{xalignat*}{2}
    \varphi(R\breve{~}^i) & = \bigcup_{r \in R} \varphi(r\breve{~}^i)           & & \text{~definition of~}\varphi(R\breve{~}^i)   \\
                          & \supseteq \bigcup_{r \in R} \varphi(r)^{\pi_i}      & & \text{~property \eqref{eq:abstract_converse}} \\
                          & = \left(\bigcup_{r \in R} \varphi(r)\right)^{\pi_i} & & \text{~distributivity in set theory}          \\
                          & = \varphi(R)^{\pi_i}                                & & \text{~definition of~} \varphi(R).
  \intertext{For \eqref{eq:abstract_composition_general_app}, consider}
    \varphi(\diamond(R_1,\dots,R_m))
      & = \bigcup_{r_1\in R_1}\dots\bigcup_{r_m\in R_m} \varphi(\diamond(r_1,\dots,r_m))               & & \text{~definition of~}\varphi(\diamond(R_1,\dots,R_m)) \\
      & \supseteq \bigcup_{r_1\in R_1}\dots\bigcup_{r_m\in R_m} \circ(\varphi(r_1),\dots,\varphi(r_m)) & & \text{~property \eqref{eq:abstract_composition}}       \\
      & = \circ\left(\bigcup_{r_1\in R_1}\varphi(r_1),\dots,\bigcup_{r_m\in R_m}\varphi(r_m)\right)    & &  \text{~distributivity in set theory}                  \\
      & = \circ(\varphi(R_1),\dots,\varphi(R_m))                                                       & & \text{~definition of~} \varphi(R_i)
  \end{xalignat*}
  
  \par\medskip\noindent
  Properties \eqref{eq:strong_converse_general_app} and \eqref{eq:strong_composition_general_app}
  are proven using \eqref{eq:strong_converse} and \eqref{eq:strong_composition}
  in the same way as we have just proven \eqref{eq:abstract_converse_general_app} and \eqref{eq:abstract_composition_general_app}
  using \eqref{eq:abstract_converse} and \eqref{eq:abstract_composition}.
  
  \par\medskip\noindent
  For \eqref{eq:weak_converse_general_app},
  let $R = \{r_1, \dots, r_n\}$ for some $n \geqslant 1$ and $r_1, \dots, r_n \in \Rel$.
  Due to Definition \ref{def:stronger_versions_of_comp+conv} \eqref{eq:weak_converse},
  we have that
  \[
    r_j\breve{~}^i = \bigcap\{S \subseteq \Rel \mid \varphi(S) \supseteq \varphi(r_i)^{\pi_i}\} 
  \]
  for every $j=1,\dots,n$.
  Let $S_{j1}, \dots, S_{jm_j}$ be the $S$ over which the above intersection ranges, i.e.,
  \[
    r_j\breve{~}^i = \bigcap_{h=1}^{m_j}  S_{jh}\,.
  \]
  Due to Definition \ref{def:qualitative_calculus},
  we have that
  \[
    R\breve{~}^i ~~=~~ \bigcup_{j=1}^{n} r_j\breve{~}^i
                 ~~=~~ \bigcup_{j=1}^{n} \bigcap_{h=1}^{m_j}  S_{jh}
                 ~~=~~ \bigcap_{h_1=1}^{m_1} \dots \bigcap_{h_n=1}^{m_n} \bigcup_{j=1}^{n} S_{jh_j}\,,
  \]
  where the last equality is due to the distributivity of intersection over union.
  Now \eqref{eq:weak_converse_general_app} follows if we show
  that, for every $S \in \Rel$, the following are equivalent.
  \begin{enumerate}
    \item
      $\varphi(S) \supseteq \varphi(R)^{\pi_i}$
    \item
      there are $S_1,\dots,S_n$ with $S = S_1 \cup\dots\cup S_n$ and $\varphi(S_j) \supseteq \varphi(r_j)^{\pi_i}$ for every $j=1,\dots,n$.
  \end{enumerate}
  For ``$1 \Rightarrow 2$'', assume $\varphi(S) \supseteq \varphi(R)^{\pi_i}$,
  i.e., $\varphi(S) \supseteq \bigcup_{j=1}^{n}\varphi(r_j)^{\pi_i}$ (Definition \ref{def:qualitative_calculus}).
  If we further assume that $S = \{s_1, \dots, s_\ell\}$, which implies that 
  $\varphi(S) \supseteq \bigcup_{h=1}^{\ell}\varphi(s_h)$ (Definition \ref{def:qualitative_calculus}),
  then we can choose $S_j = \{s_h \mid \varphi(s_h) \cap \varphi(r_j)^{\pi_i} \neq \emptyset\}$
  for every $j=1,\dots,n$. Because $C$ is based on JEPD relations,
  we have that $\varphi(S_j) \supseteq \varphi(r_j)^{\pi_i}$.
  
  For ``$2\Rightarrow1$'', let $S = S_1 \cup\dots\cup S_n$ and $\varphi(S_j) \supseteq \varphi(r_j)^{\pi_i}$ for every $j=1,\dots,n$.
  Due to Definition \ref{def:qualitative_calculus} and because $C$ is based on JEPD relations,
  we have that $\varphi(S) = \bigcup_{j=1}^{n}\varphi(S_j)$.
  Hence, $\varphi(S) \supseteq \bigcup_{j=1}^{n} \varphi(r_j)^{\pi_i}$ via the assumption,
  and $\varphi(S) \supseteq \varphi(R)^{\pi_i}$ due to Definition \ref{def:qualitative_calculus}.

  \eqref{eq:weak_composition_general_app} is proven analogously.
\end{proof}

\section{Additional proofs: Section ``Relation algebras''} 
\label{app:proofs4RA}

%

\subsection[RA6 and RA6l from Table \ref{tab:relation_algebra_axioms} are equivalent given RA7 and RA9]{\RA6 and \RA{6l} from Table \ref{tab:relation_algebra_axioms} are equivalent given \RA7 and \RA9}
\label{app:R6_R6l_equiv}

We only show that \RA6 implies \RA{6l}; the converse direction is analogous.
We first establish that $\id\breve{~} = \id$.
\begin{xalignat*}{2}
  \id\breve{~} & = \id\breve{~} \diamond \id                     & & (\RA6) \\
               & = \id\breve{~} \diamond (\id\breve{~})\breve{~} & & (\RA7) \\
               & = (\id\breve{~} \diamond \id)\breve{~}          & & (\RA9) \\
               & = (\id\breve{~})\breve{~}                       & & (\RA6) \\
               & = \id                                           & & (\RA7)
\intertext{Now we use this lemma to establish \RA{6l}.}
  \id \diamond r  & = (\id\breve{~})\breve{~} \diamond (r\breve{~})\breve{~} & & (\RA7)  \\
                  & = (r\breve{~} \diamond \id\breve{~})\breve{~}            & & (\RA9)  \\
                  & = (r\breve{~} \diamond \id)\breve{~}                     & & (\text{Lemma}) \\
                  & = (r\breve{~})\breve{~}                                  & & (\RA6)  \\
                  & = r                                                      & & (\RA7)
\end{xalignat*}
\qed

\subsection{Proof of Fact \ref{fact:minimal_axiom_set_implied_by_calculus_def}}
\label{app:minimal_axiom_set_implied_by_calculus_def}

\emph{Fact} \ref{fact:minimal_axiom_set_implied_by_calculus_def}.
Every qualitative calculus (Def.\ \ref{def:qualitative_calculus}) satisfies
\RA1--\RA3, \RA5, \RA[sup]7, \RA8, \WA[sup], \SA[sup]
for all (atomic and composite) relations.
This axiom set is maximal: each of the remaining axioms
in Table \ref{tab:relation_algebra_axioms} is not satisfied by some
qualitative calculus.

\begin{proof}
  Axioms \RA1--\RA3 are always satisfied because they are a characterization of a Boolean algebra;
  and the set operations on the relations form a Boolean algebra because
  $\varphi$ maps base relations to a set of JEPD relations and complex relations
  to sets of interpretations of base relations.

  The definition of the converse and composition operations for non-base relations
  in Definition \ref{def:qualitative_calculus} ensures that Axioms \RA5 and \RA8 hold.

  Axiom \RA[sup]7 always holds due to JEPD and the converse being weak:
  For every $r \in \Rel$, we have that
  \[
    \varphi(r\breve{~}\breve{~}) \supseteq \varphi(r\breve{~})\breve{~} \supseteq \varphi(r)\breve{~}\breve{~} = \varphi(r),
  \]
  where the first inclusion is due to Fact \ref{fact:weak+strong_conv+comp_general} \eqref{eq:abstract_converse_general}
  with $R = r\breve{~}$,
  the second inclusion is due to Definition \ref{def:qualitative_calculus} \eqref{eq:abstract_converse} for $r$,
  and the equation is due to the properties of binary relations over the universe $\Univ$.
  Since the $\varphi(r)$ are a set of JEPD relations, $r\breve{~}\breve{~} \supseteq r$ follows.
  This reasoning carries over to arbitrary relations.

  Axioms \WA[sup] and \SA[sup] always hold due to JEPD and the composition being weak:
  For every $r \in \Rel$, we have that
  \[
    \varphi((r \diamond 1) \diamond 1) \supseteq \varphi(r \diamond 1) \circ \varphi(1)
                                      =         \varphi(r \diamond 1) \circ (\Univ \times \Univ)
                                      \supseteq \varphi(r \diamond 1),
  \]
  where the first inclusion is due to to Fact \ref{fact:weak+strong_conv+comp_general} \eqref{eq:abstract_composition_general}
  with $R = r \diamond 1$ and $S = 1$,
  and the last inclusion is due to the fact that $R \circ (\Univ \times \Univ) \supseteq R$ for any binary relation $R \subseteq \Univ \times \Univ$.
  Since the $\varphi(r)$ are a set of JEPD relations, $(r \diamond 1) \diamond 1 \supseteq r \diamond 1$  follows.
  Again, this reasoning carries over to arbitrary relations.

  Axioms \RA[sub]6, \RA[sub]{6l}, \RA[sub]7 
  are violated by the following calculus.
  Let $\Rel = \{r_1,r_2\}$, $\Univ = \{0,1\}$, $\id = r_1$, $1 = \{r_1,r_2\}$ with:
  \begin{xalignat*}{3}
    \varphi(r_1)  & = \{(0,0),(0,1)\} & r_1\breve{~}  & = 1 & r_1 \diamond r_1 & = 1   \\
    \varphi(r_2)  & = \{(1,0),(1,1)\} & r_2\breve{~}  & = 1 & r_1 \diamond r_2 & = r_1 \\
                  &                   &               &     & r_2 \diamond r_1 & = 1   \\
                  &                   &               &     & r_2 \diamond r_2 & = r_2
  \end{xalignat*}
  This calculus satisfies the conditions in Definition \ref{def:qualitative_calculus}
  but violates Axioms \RA[sub]6, \RA[sub]{6l}, \RA[sub]7: 
  \begin{xalignat*}{2}
    & \RA[sub]6    & r_1 \diamond \id = 1         & ~\nsubseteq~ r_1 \\
    & \RA[sub]{6l} & \id \diamond r_1 = 1         & ~\nsubseteq~ r_1 \\
    & \RA[sub]7    & r_1\breve{~}\breve{~} = 1    & ~\nsubseteq~ r_1
  \end{xalignat*}

  Axioms \WA[sub], \SA[sub], \RA[sub]4, \RA[sup]4, \RA[sup]6, \RA[sup]{6l}, \RA[sub]9, \RA[sup]9, \RA[sub]{10}, \RA[sup]{10}, \PL[right], \PL[left] are violated by the following calculus.
  Let $\Rel = \{r_1,r_2,r_3,r_4\}$, $\Univ = \{0,1\}$, $\id = r_1$, $1 = \{r_1,r_2\}$ with:
  \par\noindent
  \parbox{.45\textwidth}{%
    \begin{xalignat*}{3}
      \varphi(r_1)  & = \{(0,0)\} & r_1\breve{~}  & = r_1 \\
      \varphi(r_2)  & = \{(1,1)\} & r_2\breve{~}  & = r_2 \\
      \varphi(r_3)  & = \{(0,1)\} & r_3\breve{~}  & = r_4 \\
      \varphi(r_4)  & = \{(1,0)\} & r_4\breve{~}  & = r_3
    \end{xalignat*}
  }%
  \hfill
  \parbox{.55\textwidth}{%
    \begin{align*}
      \begin{array}{l|llll}
        \quad \text{right operand} & r_1       & r_2       & r_3       & r_4       \\
        \text{left operand}\quad \diamond &    &           &           &           \\
        \hline
        \qquad r_1                 & r_1       & \emptyset & r_3       & \emptyset \\
        \qquad r_2                 & \emptyset & r_3       & \emptyset & r_4       \\
        \qquad r_3                 & \emptyset & r_3       & \emptyset & r_1,r_4   \\
        \qquad r_4                 & r_1,r_4   & \emptyset & r_2       & \emptyset
      \end{array}
    \end{align*}
  }
  \par\noindent
  This calculus satisfies the conditions from Definition \ref{def:qualitative_calculus}
  but violates Axioms \WA[sub], \SA[sub], \RA[sub]4, \RA[sup]4, \RA[sup]6, \RA[sup]{6l}, \RA[sub]9, \RA[sup]9, \RA[sub]{10}, \RA[sup]{10}, \PL[right], \PL[left]:
  \begin{small}
  \begin{align*}
    & \WA[sub], \SA[sub]:~~         (r_1 \diamond 1) \diamond 1 = 1                                                                     ~\nsubseteq~    \{r_1,r_3,r_4\} = r_1 \diamond 1                                                 \\
    & \RA[sub]4:~~                  (r_1 \diamond r_3) \diamond r_4 = r_3 \diamond r_4 = \{r_1,r_4\}                                    ~\nsubseteq~    r_1 = r_1 \diamond \{r_1,r_4\} = r_1 \diamond (r_3 \diamond r_4)                 \\
    & \RA[sup]4:~~                  (r_4 \diamond r_3) \diamond r_4 = r_2 \diamond r_4 = r_4                                            ~\nsupseteq~    \{r_1,r_4\} = r_4 \diamond \{r_1,r_4\} = r_4 \diamond (r_3 \diamond r_4)         \\
    & \RA[sup]6:~~                  r_2 \diamond \id = r_2 \diamond r_1 = \emptyset                                                     ~\nsupseteq~    r_2                                                                              \\
    & \RA[sup]{6l}:~~               \id \diamond r_2 = r_1 \diamond r_2 = \emptyset                                                     ~\nsupseteq~    r_2                                                                              \\
    & \RA[sub]9, \RA[sup]9:~~       (r_3 \diamond r_4)\breve{~} = \{r_1,r_4\}\breve{~} = \{r_1,r_3\}                                    ~\nsubeqsubset~ \{r_1,r_4\} = r_3 \diamond r_4 = r_4\breve{~} \diamond r_3\breve{~}              \\
    & \RA[sub]{10}, \RA[sup]{10}:~~ r_3\breve{~} \diamond \overline{r_3\diamond r_1} = r_4 \diamond \emptyset = r_4 \diamond 1 = \{r_1,r_2,r_4\} ~\nsubeqsubset~ \{r_2,r_3,r_4\} = \overline{r_1}                                        \\
    & \PL[right]:~~                 (r_1 \diamond r_4) \cap r_1\breve{~} = \emptyset \cap r_1 = \emptyset                               ~~\text{but}~~  (r_4 \diamond r_1) \cap r_1\breve{~} = \{r_4,r_1\} \cap r_1 = r_4 \neq \emptyset \\
    & \PL[left]:~~                  (r_4 \diamond r_1) \cap r_1\breve{~} = \{r_4,r_1\} \cap r_1 = r_1 \neq \emptyset                    ~~\text{but}~~  (r_1 \diamond r_1) \cap r_4\breve{~} = r_1 \cap r_3 = \emptyset
  \end{align*}
  \end{small}
  \qed
\end{proof}
  
\begin{remark}
  Of course, there are calculi that satisfy only the weak conditions from Definition \ref{def:qualitative_calculus}
  but are a relation algebra, for example the following.
  Let $\Rel = \{r_0,r_1\}$, $\Univ = \{0,1\}$, $\id = r_1$, $1 = \{r_1,r_2\}$ with:
  \begin{xalignat*}{3}
    \varphi(r_1)  & = \{(0,0),(0,1)\} & r_1\breve{~}  & = r_2 & r_1 \diamond r_1 & = r_1 \\
    \varphi(r_2)  & = \{(1,0),(1,1)\} & r_2\breve{~}  & = r_1 & r_1 \diamond r_2 & = 1   \\
                  &                   &               &       & r_2 \diamond r_1 & = 1   \\
                  &                   &               &       & r_2 \diamond r_2 & = r_2
  \end{xalignat*}
\end{remark}

\section{Additional complexity proofs}
\label{app:complexity_proofs}

\emph{Fact} \ref{app:complexity_proofs}.\ref{dra-conn}.
Consistency of QCSPs for \calc{DRA-conn} can be decided in time $\Landau{n^3}$.

\begin{proof}
  The \calc{DRA-conn} calculus is an abstraction of the more fine-grained dipole calculi, only retaining connectivity relations of line segments.
  Connectivity is represented by equality relations between positions of a dipole's start or end point.
  For checking consistency of a set of \calc{DRA-conn} constraints, the clusters of equally positioned points need to be constructed.
  This can easily be done with the algebraic closure algorithm.
  Since the effect of a disjunctive relation in \calc{DRA-conn} with respect to single point equality is identical to absence of the constraint, 
  reasoning with partial atomic QCSPs is equivalent in complexity to reasoning with general QCSPs with \calc{DRA-conn}.
\end{proof}

\emph{Fact} \ref{app:complexity_proofs}.\ref{fact:eia}.
Consistency of atomic QCSPs for \calc{EIA} can be decided in polynomial time.

\begin{proof}
  As described by \citeN{ZhangRenz_2014_AngryBirds}, extended interval algebra constraints can be translated to \calc{\INDU} constraint networks,
  and those can be decided in polynomial time \cite{BCL06}.
  \calc{EIA} describes relative ordering with respect to interval start, end, and center point.
  Consequently, for every single variable in a given \calc{EIA} network,
  the translation introduces three variables representing an interval and its two halves,
  together with the obvious constraints between them.
\end{proof}

\emph{Fact} \ref{app:complexity_proofs}.\ref{GenInt}.
The tractable subset of \calc{GenInt} consisting of all strongly pre-convex general relations
covers less than  1\textperthousand\ of all relations for the case of 3-intervals.

\begin{proof}
  Generalized intervals \cite{condotta-ECAI:00} generalize \calc{IA} relations to tuples of intervals.
  Relations between a $p$- and and a $q$-tuple, \emph{general relations}, are represented in a $p\times q$ matrix of \calc{IA} relations.
  A strongly pre-convex general relation is a matrix where all entries are strongly preconvex.
  Since the strongly pre-convex relations are a subset of pre-convex relations
  and only some 10\% of all \calc{IA} relations are pre-convex,
  at most a fraction of $0.1^{p\cdot q}$ of all general relations is strongly pre-convex,
  which is far less than 1\textperthousand\ if $p=q=3$.
  Even if we could take the matrix entries from a tractable subset of, say, 20\% of \calc{IA},
  we would still get $0.2^{p\cdot q} \ll 1\text{\textperthousand}$ tractable relations.
\end{proof}

\emph{Fact} \ref{app:complexity_proofs}.\ref{fact:om3d}.
Deciding consistency of atomic QCSPs for \calc{OM-3D} is NP-hard
and can be reduced to solving multivariate polynomial equalities.

\begin{proof}
  \calc{OM-3D} generalizes the \calc{double-cross} calculus from 2D arrangement to 3D arrangement, containing the 2D case as a sub-algebra.
  Since base relations of the 2D case are already NP-hard~\cite{WolterL:2010:Realization4Direction}, so is \calc{OM-3D}. 
  All base relations for the 3D case can be modeled by multivariate polynomial equalities similar to the 2D case.
\end{proof}

\emph{Fact} \ref{app:complexity_proofs}.\ref{star-convex}.
Consistency of QCSPs with convex relations for $\calc{STAR}_m$ and $\calc{STAR}^r_m$
can be decided in polynomial time.

\begin{proof}
  $\calc{STAR}_m$ defines $4m$ relations (line segments and sectors);
  $\calc{STAR}^r_m$ defines $2m$ relations which are all sectors. 
  Tractability of convex relations follows from the observation
  that these can be represented by half-plane intersections using linear inequalities,
  systems of which can be decided in polynomial time using linear programming techniques.
\end{proof}
While the number of all relations in $\calc{STAR}^{(r)}_m$ grows exponentially with $m$, there are only $m$ convex relations that include $1,\ldots, m$ relations, i.e., $\Landau{m^2}$ convex relations.
The percentage of convex relations thus decreases with increasing values of $m$.

\section{Expressivity relations between calculi}\label{app:expressivity}
\new{%
We give additional proof sketches for expressivity relations presented in Figure~\ref{fig:calculi-expressivity-overview}. 
Recall that we say a calculus is of equivalent expressivity as another calculus if every QCSP instance of the first can be simulated by a propositional formulae of constraints in the second.
}

\new{%
\begin{theorem}
Temporal calculi \calc{PC},\calc{IA},\calc{SIC},\calc{DIA},\calc{GenInt} and spatial calculi \calc{BA}, \calc{CDC}, and \calc{CI} form a cluster of expressivity.
\end{theorem}}
\begin{proof}[sketch]
\new{%
Temporal point- and interval-based calculi (semi-intervals in case of \calc{SIC}) represent ordering relations which can all be translated into Boolean formulae of \calc{PC} relations among interval start and end point. 
Solutions for QCSPs over these temporal calculi in the cluster can easily be obtained from their corresponding \calc{PC} formulae by instantiating intervals from their start and and points.}

\new{%
The spatial calculus \calc{BA} is an independent product \calc{IA}$\times$\calc{IA} easily expressible using propositional \calc{BA} formulae, analogously is \calc{CDC} expressible as product \calc{PC}$\times$\calc{PC}. 
\calc{CI} represents a cyclic order (e.g., intervals of longitude). 
These relations can be simulated with \calc{PC} by instantiating an lower and upper limit points $p_-$ and $p_+$ and splitting all intervals containing either $p_-$ or $p_+$ to continue from the opposite border.}
\end{proof}

\new{%
\begin{theorem}
\calc{VR} relations can be expressed using \calc{LR} constraints.
\end{theorem}}
\new{%
\begin{proof}[sketch]
\calc{VR} expresses visibility of convex objects in the plane using ternary relations. 
Visibility relations can be represented based on the relative position of tangent points of the base entities, e.g., visibility between two objects is not affected if and only if a third object discrete from the first two does not intersect with the four-sided polygon obtained by connecting the upper and lower tangent points of the two objects. 
Overlap between polygonal contours can easily be written using \calc{LR} constraints, e.g., a point is outside a convex polygon if it is located to the right hand side of at least one edge of the polygon, assuming the polygon edges to be ordered in counter-clockwise manner.
The construction is then performed for every visibility relation, instantiating lower and upper tangent points individually for every pair of \calc{VR} entities.
The \calc{VR} entities which are regions are then represented only by their set of tangent points which can be enforced to be arranged along a convex-shaped contour.
\end{proof}}

\new{%
\begin{theorem}
Calculi \calc{TPCC},\calc{OPRA},\calc{EOPRA}, \calc{1-}, and \calc{2-cross} constitute a cluster of equal expressive power for Boolean combinations of constraints.
\end{theorem}}
\begin{proof}[sketch]
\new{%
This group of calculi considers locations of points in the Euclidean plane.
We first consider equivalence of \calc{OPRA}, \calc{1-}, and \calc{2-cross} and later address \calc{TPCC} and \calc{EOPRA} which augment the first group by additional distance concepts.
All  calculi from the first group employ a partition scheme that is based on relations that specifies directions to points relative to some entity-specific orientation (either by reference to another entity in case of  \calc{1-}, and \calc{2-cross} or as intrinsic part of the base entity in case of \calc{OPRA}).
Directions measured in radians are represented by membership in a finite and JEPD set of intervals partitioning $(0, 2\pi]$, using solely rational ratios of $\pi$ as boundaries.
By geometric construction one can obtain any of these direction intervals (i.e., sectors)  of these calculi from a any partition scheme for point location that is able to express superposition of points, a statement that two line segments connecting three points $A,B,C$ meet in a right angle, i.e., $\angle(A,B,C) = \frac{\pi}{2}$ as well as a statement saying that a point is located directly in front of some point $P$ with respect to ``front'' orientation of $P$
All the named calculi meet these conditions and allow for the following construction:
Let $P$ be the entity for which we seek to construct direction intervals in form of a sector. 
First, enforce four points $A,B,C,D$ to form a rectangle with $A$ in superposition with $P$ and $C$ in front of $P$.
Next we construct $E$ to be positioned on the intersection of $\overline{AC}$ and $\overline{BD}$ which meet in a right angle. 
Doing so we have constructed a square.
Repeating the construction we can construct a grid from which we can derive the desired angular sectors.}

\new{%
Now we show that \calc{OPRA}, \calc{TPCC}, and \calc{EOPRA} have the same expressivity.
\calc{EOPRA} augments \calc{OPRA} by a relative distance concept in the same way \calc{TPCC} augments \calc{1-cross}.
Constructions translating  \calc{EOPRA} to \calc{OPRA} are very similar to translating \calc{TPCC} to \calc{1-cross}, so we only consider the first case.
Distance classes in the calculi \calc{OPRA} and \calc{TPCC} are named ``close'', ``same'', and ``far'' and are defined by comparison of the Euclidean distance between two entities with an object-specific threshold distance.
This means that the statement ``A is close to B'' is independent from ``B is close to A''.
These distance constraints can be simulated in \calc{OPRA} by introducing {\em border} points for each entity along the ``same'' distance, one for every pair of entities. 
To this end we have to enforce that all border points are in the same distance to their corresponding entity. 
This can be accomplished by \calc{OPRA} constraints by first constructing a bisector for a pair of border points (as done in the construction above) and, second, enforcing a right angle between the line connecting two border points with the bisector.}
\end{proof}

\section{Detailed description of the test results by calculus: Section ``Algebraic Properties of Existing Calculi''} 
\label{app:descrOfProperties}

The results of the analysis are summarized in Table \ref{tab:calculi_tests}.
A part of the calculi have already been tested
by \citeN{Mos07}, 
using a different CASL specification based on an equivalent axiomatization from \cite{LigozatR04}.
He comprehensively reports on the outcome of these tests, and on errors discovered in published composition tables.
We now list counterexamples for the cases where axioms are violated.
%
\begin{table}[t]
  \centering
  \begin{small}
    \rowcolors{1}{lightblue}{}%
    \setcounter{myfn}{0}%
    \begin{tabular}{ll*{9}c}
      \hline\rowcolor{medblue}\rule{0pt}{9pt}%
      \eestrut Calculus                       & Tests\myfnm{fnTests} 
                                              & ~\RA4~ & ~\SA~ & ~\WA~ & ~\RA6~ & ~\RA{6l}~ & ~\RA7~ & ~\RA9~ & ~\PL~ & ~\RA{10}~ \\[1pt]
      \hline\rule{0pt}{9pt}%
      \eestrut $\calc{BA}_n$, $n\leqslant 2$  & \HS  & \YES   & \YES  & \YES  & \YES   & \YES      & \YES   & \YES   & \YES  & \YES      \\
      \eestrut \calc{CDC}                     & \MHS & \YES   & \YES  & \YES  & \YES   & \YES      & \YES   & \YES   & \YES  & \YES      \\
      \eestrut \calc{\CYCb}                   & \HS  & \YES   & \YES  & \YES  & \YES   & \YES      & \YES   & \YES   & \YES  & \YES      \\
      \eestrut \calc{\DRAfp}, \calc{DRA-conn} & \HS  & \YES   & \YES  & \YES  & \YES   & \YES      & \YES   & \YES   & \YES  & \YES      \\
      \eestrut \calc{IA}                      & \MHS & \YES   & \YES  & \YES  & \YES   & \YES      & \YES   & \YES   & \YES  & \YES      \\
      \eestrut $\calc{PC}_1$                  & \HS  & \YES   & \YES  & \YES  & \YES   & \YES      & \YES   & \YES   & \YES  & \YES      \\
      \eestrut \calc{RCC-5}, \calc{DepCalc}   & \MHS & \YES   & \YES  & \YES  & \YES   & \YES      & \YES   & \YES   & \YES  & \YES      \\
      \eestrut \calc{RCC-8}, \calc{9-int}     & \MHS & \YES   & \YES  & \YES  & \YES   & \YES      & \YES   & \YES   & \YES  & \YES      \\
      \eestrut $\calc{STAR}^r_4$              & \HS  & \YES   & \YES  & \YES  & \YES   & \YES      & \YES   & \YES   & \YES  & \YES      \\
      \hline\rule{0pt}{9pt}%
      \eestrut \calc{\DRAf}                   & \MHS & ~~\,19 & \YES  & \YES  & \YES   & \YES      & \YES   & \YES   & \YES  & \YES      \\
      \eestrut \calc{\INDU}                   & \MHS & ~~\,12 & \YES  & \YES  & \YES   & \YES      & \YES   & \YES   & \YES  & \YES      \\
      \eestrut $\calc{OPRA}_n$, $n\leqslant8$ & \MHS & 21--91\myfnm{fnOPRA}
                                                              & \YES  & \YES  & \YES   & \YES      & \YES   & \YES   & \YES  & \YES      \\
      \hline\rule{0pt}{9pt}%
      \eestrut \calc{QTC-B$xx$}               & \MHS & \YES   & \YES  & \YES  & \multicolumn{2}{c}{89--100} 
                                                                                                   & \YES   & \YES   & \YES  & \YES      \\
      \hline\rule{0pt}{9pt}%
      \eestrut \calc{QTC-C21}                 &  \HS & ~~\,55 & \YES  & \YES  & 99     & 99        & \YES   & ~\,2   & $<$1  & ~\,1      \\
      \eestrut \calc{QTC-C22}                 &  \HS & ~~\,79 & \YES  & \YES  & 99     & 99        & \YES   & ~\,3   & $<$1  & ~\,1      \\
      \hline\rule{0pt}{9pt}%
      \eestrut \calc{RCD}                     & \HS  & \YES   & \YES  & \YES  & 97     & 92        & 89     & 66     & ~\,7  & 52        \\
      \hline\rule{0pt}{9pt}%
      \eestrut \calc{cCDR}                    & \HS  & ~~\,28 & 17    & \YES  & 99     & 99        & 98     & 12     & $<$1  & 88        \\
      \hline
    \end{tabular}
    \par\vspace*{2pt}
    \rowcolors{1}{}{}%
    \begin{tabular}{p{.88\textwidth}}
      \myfn{fnTests}{calculus was tested by: M $=$ \cite{Mos07}, H $=$ \system{Hets}, S $=$ \system{SparQ}}         \\[1pt]
      \myfn{fnOPRA}{21\%, 69\%, 78\%, 83\%, 86\%, 88\%, 90\%, 91\% for $\calc{OPRA}_n$, $n = 1,\dots,8$}
    \end{tabular}
    \par
  \end{small}
  \par\smallskip
  \caption{%
    Overview of calculi tested and their properties. The symbol ``\Yes'' means that the axiom is satisfied;
    otherwise the percentage of counterexamples (relations, pairs or triples violating the axiom) is given.%
  }
  \label{tab:calculi_tests}
\end{table}

\par\medskip\noindent
\calc{cCDR}
\begin{itemize}
  \item
    \RA6 is violated for all base relations but one.
  \item
    \RA{6l} is violated for only 209 base relations.
  \item
    \RA7 is violated for 214 base relations.
  \item
    \RA9 is violated for 5,607 pairs of base relations.
    Counterexample:
    \[
      \noninvdist{\Reln{S}}{\Reln{S}}
    \]
  \item
    \RA{10} is violated for 41,834 pairs of base relations.
    Counterexample:
    \[
      \nontarskidm{\Reln{S}}{\Reln{S}}
    \]
  \item
    \PL is violated for 22,976 triples of base relations.
    Counterexample:
    \[
      \nonpeirceext{\Reln{W-NW-N-NE-E}}{\Reln{NW-N-NE}}{\Reln{B-S}}{}{\Reln{B}}
    \]
  \item
    \RA4 is violated for 2,936,946 triples of base relations.
    Counterexample:
    \begin{align*}
      \Nonassoc{\Reln{W-NW-N-NE-E-SE}}{\Reln{W-NW-N-NE-E-SE}}{\Reln{W-NW-N-NE-E}}
    \end{align*}
  \item
    \SA is violated for 38 base relations.
    Counterexample:
    \[
      \nonsemiassoc{\Reln{B-S-W-NW}}
    \]
\end{itemize}

\par\medskip\noindent
\calc{\DRA}
\begin{itemize}
  \item
    \calc{\DRAc} violates \RA4 for 704 triples of base relations. Counterexample:
    \[
        \nonassoc[]{\Reln{rrrl}}{\Reln{rrrl}}{\Reln{llrl}}
    \]
  \item
    \calc{\DRAf} violates \RA4 for 71,424 triples of base relations, with the same counterexample,
    or with the one reported by \citeN{MoratzEtAl2011}, who attribute the violation of associativity
    to the composition operation being weak and illustrate this by the example
    $\Reln{bf{}ii} \diamond \Reln{lllb} = \Reln{llll}$.
  \item
    \calc{\DRAfp} and \calc{DRA-conn} satisfy all axioms.
\end{itemize}

\par\medskip\noindent
\calc{INDU}
\par\smallskip\noindent
    \RA4 is violated by 1,880 triples of base relations.
    The violation of associativity has already been reported and attributed to the absence of strong composition
    in \cite{BCL06}: e.g.,
    \[
        \nonassoc[]{\Reln{bi}^>}{\Reln{mi}^>}{\Reln{m}^>}.
    \]

\par\medskip\noindent
\calc{MC-4}
\par\smallskip\noindent
\calc{MC-4} is not based on a partition scheme
because the relation \texttt{cg} (``congruent''),
which behaves in the context of the other three relations
as if it were an identity relation,
is coarser than $\id^2$.
Furthermore, \calc{MC-4} is still an abstract partition scheme
and thus fits into our general notion of a calculus.

For testing purposes, we have implemented an artificial
variant of \calc{MC-4} where we divided the \texttt{cg} relation
into $\id^2$ and the difference of \texttt{cg} and $\id^2$.
That calculus too is a relation algebra.

\par\medskip\noindent
\calc{$\calc{OPRA}_n$, $n \leqslant 8$}
\par\smallskip\noindent
\RA4 is violated by 
\par\smallskip\noindent
\begin{tabular}[t]{@{}r@{~}l@{~~~}l@{}}
          1,664 & triples for $\calc{OPRA}_1$, & e.g., $\nonassoc{3_{3}}{3_{2}}{0_3}$           \\
        257,024 & triples for $\calc{OPRA}_2$, & e.g., $\nonassoc{7_{7}}{7_{7}}{6_7}$           \\
      2,963,952 & triples for $\calc{OPRA}_3$, & e.g., $\nonassoc{11_{11}}{11_{11}}{11_{10}}$   \\
     16,711,680 & triples for $\calc{OPRA}_4$, & e.g., $\nonassoc[]{15_{15}}{15_{15}}{15_{15}}$ \\
     63,840,000 & triples for $\calc{OPRA}_5$, & e.g., $\nonassoc[]{19_{19}}{19_{19}}{19_{19}}$ \\
    190,771,200 & triples for $\calc{OPRA}_6$, & e.g., $\nonassoc[]{23_{23}}{23_{23}}{23_{23}}$ \\
    481,275,648 & triples for $\calc{OPRA}_7$, & e.g., $\nonassoc[]{27_{27}}{27_{27}}{27_{27}}$ \\
  1,072,693,248 & triples for $\calc{OPRA}_8$, & e.g., $\nonassoc[]{31_{31}}{31_{31}}{31_{31}}$
\end{tabular}

\par\medskip\noindent
\calc{QTC}
\begin{itemize}
  \item
    \calc{QTC-B11}, \calc{-B12}, \calc{-C21}, \calc{-C22} violate \RA6 and \RA{6l} for all base relations but one;
    \calc{QTC-B21}, \calc{-B22} do so for all base relations.
    After introducing a new \id relation and making the relations JEPD,
    \calc{QTC-B11} and \calc{-B12} satisfy all axioms \cite{Mos07}.
  \item
    \calc{QTC-C21} (81 base relations) violates
    \RA4 for 292,424 triples, \RA9 for 160 pairs, \RA{10} for 80 pairs, and \PL for 1056 triples.\footnote{%
      \label{ftn:special_treatment_R10}%
      Note that, for calculi that violate \RA9,
      the equivalence between \PL and \RA{10} is no longer ensured,
      hence the mentioning of both of them.
      Furthermore, \RA{10} is the only axiom that should be tested for all relations,
      but we have only tested it for all base relations.
      Therefore, there could be more violations than the four listed.
      The same cautions apply to \calc{QTC-C22.}%
    }
  \item
    \calc{QTC-C22} (209 base relations) violates
    \RA9 for 1248 pairs, \RA{10} for 624 pairs, \PL for 12,768 triples, and \RA4 for 7,201,800 triples,
    see also footnote \ref{ftn:special_treatment_R10}.
\end{itemize}

\par\medskip\noindent
\calc{RCD}
\begin{itemize}
  \item
    \RA6 is violated for all base relations but one.
  \item
    \RA{6l} is violated for only 33 base relations.
  \item
    \RA7 is violated for 32 base relations.
  \item
    \RA9 is violated for 855 pairs.
    Counterexample:
    \[
      \noninvdist{\Reln{B}}{\Reln{S:SW}}
    \]
  \item
    \RA{10} is violated for 671 pairs.
    Counterexample:
    \[
      \nontarskidm{\Reln{B}}{\Reln{S:SW}}
    \]
  \item
    \PL is violated for 3424 triples.
    Counterexample:
    \[
      \nonpeirce{\Reln{B}}{\Reln{N}}{\Reln{B:W}}
    \]
\end{itemize}

\end{document}